\documentclass[oneside,english]{article}

\usepackage[utf8]{inputenc}
\usepackage{hyperref}
\usepackage{url}
\usepackage{algorithm,algpseudocode}
\usepackage{mathtools}
\usepackage{amsmath,amssymb,amsthm}
\usepackage{graphicx}
\usepackage{mathptmx}
\usepackage{todonotes}
\usepackage{enumitem}

\usepackage{titlesec}
\titleformat{\section}
	{\normalfont\Large\bfseries\filcenter}{\thesection.}{1 ex}{}
\titleformat{\subsection}
	{\normalfont\normalsize\bfseries}{\thesubsection.}{1 ex}{}
\titleformat{\subsubsection}[runin]
	{\normalfont\normalsize\bfseries\filcenter}{\thesubsubsection.}{1 ex}{}

\usepackage[charter]{mathdesign}
\usepackage[mathcal]{eucal}

\usepackage[margin=1.2 in]{geometry}

\usepackage[square,numbers,sort&compress]{natbib}

\usepackage{thmtools}

\declaretheorem[name=Theorem]{Thm}
\declaretheorem[within=section,name=Lemma]{Lem}
\declaretheorem[sibling=Lem,name=Definition]{Def}
\declaretheorem[sibling=Lem,name=Assumption]{Ass}
\declaretheorem[sibling=Lem,name=Notation]{Not}
\declaretheorem[sibling=Lem,name=Proposition]{Prop}
\declaretheorem[sibling=Lem,name=Remark]{Rem}
\declaretheorem[sibling=Lem,name=Example]{Ex}
\declaretheorem[sibling=Lem,name=Corollary]{Cor}

\newcommand{\mesh}{\operatorname{mesh}}

\newcommand{\calA}{\mathcal{A}}

\newcommand{\calF}{\mathcal{F}}

\newcommand{\calH}{\mathcal{H}}

\newcommand{\calP}{\mathcal{P}}

\newcommand{\calX}{\mathcal{X}}
\newcommand{\calY}{\mathcal{Y}}
\newcommand{\calZ}{\mathcal{Z}}

\newcommand{\RR}{\ensuremath{\mathbb{R}}}

\newcommand{\NN}{\ensuremath{\mathbb{N}}}

\newcommand{\EE}{\ensuremath{\mathbb{E}}}

\makeatletter
\providecommand*{\diff}%
        {\@ifnextchar^{\DIfF}{\DIfF^{}}}
\def\DIfF^#1{%
        \mathop{\mathrm{\mathstrut d}}%
                \nolimits^{#1}\gobblespace
}
\def\gobblespace{%
        \futurelet\diffarg\opspace}
\def\opspace{%
        \let\DiffSpace\!%
        \ifx\diffarg(%
                \let\DiffSpace\relax
        \else
                \ifx\diffarg\[%
                        \let\DiffSpace\relax
                \else
                        \ifx\diffarg\{%
                                \let\DiffSpace\relax
                        \fi\fi\fi\DiffSpace}
\makeatother

\newcommand{\BV}{\operatorname{BV}}
\newcommand{\Var}{\operatorname{V}}

\newcommand{\tensalg}{\operatorname{T}(\calH)}

\newcommand{\simplex}{\Delta}

\newcommand{\ba}{\boldsymbol{a}}
\newcommand{\bb}{\boldsymbol{b}}

\newcommand{\bi}{\boldsymbol{i}}
\newcommand{\bj}{\boldsymbol{j}}
\newcommand{\br}{\boldsymbol{r}}
\newcommand{\bs}{\boldsymbol{s}}
\newcommand{\bt}{\boldsymbol{t}}
\newcommand{\bu}{\boldsymbol{u}}
\newcommand{\bx}{\boldsymbol{x}}
\newcommand{\by}{\boldsymbol{y}}
\newcommand{\bsigma}{\boldsymbol{\sigma}}
\newcommand{\btau}{\boldsymbol{\tau}}
\newcommand{\XSeq}{\calX^+}

\newcommand{\Paths}{\operatorname{Paths}}
\newcommand{\PathsX}{\calP\left(\calX\right)}

\newcommand{\seqdf}{\nabla}

\newcommand{\intvl}{U}

\newcommand{\IndLvl}{M}
\newcommand{\IndDegr}{D}

\newcommand{\Sig}{\operatorname{S}}
\newcommand{\Sigk}{\operatorname{S}_{\le \IndLvl}}

\newcommand{\SSig}{\mathfrak{S}}
\newcommand{\SSigk}{\mathfrak{S}_{\IndLvl}}
\newcommand{\SSigDk}{\mathfrak{S}_{(\IndDegr),\IndLvl}}
\newcommand{\SSigD}{\mathfrak{S}_{(\IndDegr)}}

\newcommand{\kernel}{\operatorname{k}} 
\newcommand{\KSigA}{\operatorname{K}^{\oplus}} 
\newcommand{\KSigB}{\kernel^{\oplus}} 
\newcommand{\KSigAk}{\operatorname{K}^{\oplus}_{\le\IndLvl}} 
\newcommand{\KSigBk}{\kernel^{\oplus}_{\leq\IndLvl}}
\newcommand{\Kmeasure}{\kappa}

\newcommand{\KSeq}{\operatorname{k}^+}
\newcommand{\KSeqA}{\operatorname{K}^+}
\newcommand{\KSeqAk}{\operatorname{K}^+_{\IndLvl}}
\newcommand{\KSeqB}{\kernel^+}
\newcommand{\KSeqBk}{\kernel^+_{\IndLvl}}
\newcommand{\KSeqBDk}{\kernel^+_{(\IndDegr),\IndLvl}}
\newcommand{\KSeqBonek}{\kernel^+_{(\IndDegr),\IndLvl}}

\newcommand{\KSeqDk}{\kernel^+_{(\IndDegr),\IndLvl}}

\newcommand{\Exp}{\exp}
\newcommand{\ExpM}{\exp_{\IndLvl}}

\usepackage{esint}

\makeatletter

\numberwithin{equation}{section}
\numberwithin{figure}{section}
\theoremstyle{plain}
\newtheorem{thm}{\protect\theoremname}
  \theoremstyle{definition}
  
  \theoremstyle{definition}
  \newtheorem{defn}[thm]{\protect\definitionname}
  \theoremstyle{plain}
  
  \theoremstyle{plain}
  
  \theoremstyle{remark}
  \newtheorem{rem}[thm]{\protect\remarkname}
  \theoremstyle{plain}

\makeatother

\usepackage{babel}
  \providecommand{\corollaryname}{Corollary}
  \providecommand{\definitionname}{Definition}
  \providecommand{\examplename}{Example}
  \providecommand{\lemmaname}{Lemma}
  \providecommand{\remarkname}{Remark}
  \providecommand{\theoremname}{Theorem}
  \providecommand{\propositionname}{Proposition}

\usepackage{authblk}

\title{Kernels for sequentially ordered data}

\author[1]{
Franz J.~Kir\'{a}ly
\thanks{\url{f.kiraly@ucl.ac.uk}}
}

\author[2]{Harald Oberhauser
\thanks{\url{oberhauser@maths.ox.ac.uk}}
}

\affil[1]{
Department of Statistical Science,
University College London,\newline
Gower Street,
London WC1E 6BT, United Kingdom
}

\affil[2]{Mathematical Institute,
University of Oxford,\newline
Andrew Wiles Building,
Oxford OX2 6GG, United Kingdom
}

\date{}
\begin{document}

\maketitle

\begin{abstract}
We present a novel framework for kernel learning with sequential data of any kind, such as time series, sequences of graphs, or strings.
Our approach is based on signature features which can be seen as an
ordered variant of sample (cross-)moments; it allows to obtain a
``sequentialized'' version of any static kernel. The sequential
kernels are efficiently computable for discrete sequences and are
shown to approximate a continuous moment form in a sampling sense.

A number of known kernels for sequences arise as ``sequentializations'' of suitable static kernels: string kernels may be obtained as a special case, and alignment kernels are closely related up to a modification that resolves their open non-definiteness issue.
Our experiments indicate that our signature-based sequential kernel framework may be a promising approach to learning with sequential data, such as time series, that allows to avoid extensive manual pre-processing.
\end{abstract}

\newpage
\tableofcontents
\newpage

\newpage
\section{Introduction}

{\bf Sequentially ordered data are ubiquitous} in modern science,
occurring as time series, location series, or, more generally,
sequentially observed samples of numbers, vectors, and structured
objects. They occur frequently in structured machine learning tasks,
in supervised classification and regression as well as in forecasting,
as well as in unsupervised learning.

{\bf Three stylized facts} make learning with sequential data an
ongoing challenge:
\begin{enumerate}[label=(\Alph*)]
\item\label{A} Sequential data is usually very diverse, with wildly different features being useful. In the state-of-the-art, this is usually addressed by {\bf manual extraction of hand-crafted features}, the combination of which is often very specific to the application at hand and does not transfer easily.
\item\label{B} Sequential data often occurs as {\bf sequences of structured objects}, such as letters in text, images in video, graphs in network evolution, or heterogenous combination of all mentioned say in database or internet applications. This is usually solved by ad-hoc approaches adapted to the specific structure.
\item\label{C} Sequential data is often large, with sequences easily obtaining the length of hundreds, thousands, millions. Especially when there is one or more sequences per data point, the data sets quickly become {\bf very huge}, and with them computational time.
\end{enumerate}
In this paper, we present a novel approach to learning with sequential data based on a {\bf joining of the theory of signatures/rough paths}, and {\bf{kernels/Gaussian processes}}, addressing the points above:

\begin{enumerate}[label=(\Alph*)]
\item The {\bf signature} of a path is a (large) collection of {\bf canonical features} that can be intuitively described an ordered version of sample moments. They completely describe a
  sequence of vectors (provably), and make sequences of different size and length comparable.
The use of signature features is therefore a straightforward way of
avoiding manual feature extraction
(Section \ref{sec:signature_features}).
\item Combining signatures with the {\bf kernel trick},
  by considering the signature map as a feature map yields a kernel
  for sequences.
It also allows {\bf learning with sequences of
    structured objects} for which non-sequential kernels exist ---
  consequently we call the process of obtaining a sequence kernel from
  a kernel for structured objects ``kernel sequentialization'' (Section \ref{sec:kernelized_signatures}).
\item The {\bf sequentialized kernel} can be {\bf computed
    efficiently} via dynamic programming ideas similar to those known
  for string kernels (Section \ref{sec:approx}). The kernel formalism makes the
  computations further amenable to low-rank type speed-ups in kernel
  and Gaussian process learning such as Nyström-type and Cholesky
  approximations or inducing point methods (Section \ref{sec:comp}).
\end{enumerate}

To sum up, we provide a canonical construction to transform any kernel $\kernel : \calX\times \calX:\rightarrow\RR$
into a version for sequences $\KSeq: \XSeq\times \XSeq\rightarrow
\RR$, where we have denoted by $\calX^+$ the set of arbitrary length
sequences in $\calX$.
We call $\KSeq$ the \textbf{sequentialization of} $\kernel$.
This sequentialization is canonical in the sense that it converges to an inner product of ordered moments, the signature, when sequences in $\calX^+$ converge to functions $[[0,1]\rightarrow \calX]$ in a meaningful way.
We will see that existing kernels for sequences such as string or alignment kernels are closely related to this construction.

We explore the practical use of the sequential kernel in experiments
which show that sequentialization of non-linear kernels may be
beneficial, and that the sequential kernel we propose can beat the
state-of-the-art in sequence classification while avoiding extensive
pre-processing.
Below we give an informal overview of the main ideas, and a summary of related prior art.

\subsection{Signature features, and their universality for sequences}
\label{sec:intro.sign}
Signatures are universal features for sequences, characterizing sequential structure by quantifying dependencies in their change, similar to sample moments. We showcase how to obtain such signature features for the simple example of a two-dimensional, smooth series
\begin{align*}
x: [0,1]\rightarrow\mathbb{R}^2,t\mapsto (a(t),b(t))^{\top},
\end{align*}
whose argument we interpret as ``time''.
As with sample moments or sample cumulants of different degree, there are signature features of different degree, first degree, second degree, and so on. The first degree part of the signature is the average change in the series, that is,
\begin{align*}
\Sig_1(x) := \mathbb{E}_{t}[\dot{x}(t)] = \mathbb{E}_{t}
\left[
\begin{pmatrix}
    \dot{a}(t)\\
    \dot{b}(t)
\end{pmatrix}\right]
=\begin{pmatrix}
    \int_0^1 \diff a(t)\\
    \int_0^1 \diff b(t)
\end{pmatrix}
 =
\begin{pmatrix}
    a(1) - a (0)\\
    b(1) - b (0)
\end{pmatrix},
\end{align*}
where we have written $\diff a(t):=\dot{a}(t) \diff t$,$\diff b(t):=\dot{b}(t) \diff t$.
The second degree part of the signature is the (non-centered) covariance of changes at two subsequent time points, that is, the expectation
\begin{align*}
\Sig_2(x) &:= \frac{1}{2}\EE_{t_1<t_2} [\dot{x}(t_1)\cdot \dot{x}(t_2)^\top]
= \frac{1}{2}\EE_{t_1<t_2}
\left[
\begin{pmatrix}
    \dot{a}(t_1)\dot{a}(t_2)&\dot{a}(t_1)\dot{b}(t_2)\\
    \dot{b}(t_1)\dot{a}(t_2)&\dot{b}(t_1)\dot{b}(t_2)
\end{pmatrix}\right]\\
&= \begin{pmatrix}
    \int_0^1\int_0^{t_1} \diff a(t_2)\diff a(t_1)&\int_0^1\int_0^{t_1} \diff a(t_2)\diff b(t_1)\\
    \int_0^1\int_0^{t_1} \diff b(t_2)\diff a(t_1)&\int_0^1\int_0^{t_1} \diff b(t_2)\diff b(t_1)
\end{pmatrix},
\end{align*}
where the expectations in the first line are uniformly over time
points in chronological order $t_1\le t_2$ (that is, $t_1, t_2$ is
the order statistic of two points sampled from the uniform distribution on $[0,1]$). This is equivalent to integration over the so-called $2$-order-simplex $\{t_1,t_2:0\leq t_1\leq t_2\leq 1\}$ in the second line, up to a factor of $1/2$ corresponding to the uniform density (we put it in front of the expectation and not its inverse in front of the integral to obtain an exponential generating function later on).

Note that the second order signature is different from the second order moment matrix of the infinitesimal changes by the chronological order imposed in the expectation.
Similarly, one defines the degree $\IndLvl$ part of the signature as the $\IndLvl$-th order moment tensor of the infinitesimal changes, where expectation is taken over chronologically ordered time points (which is a tensor of degree $\IndLvl$). A basis-free definition over an arbitrary RKHS is given in Section~\ref{sec:sig.int}. Note that the signature tensors are not symmetric, similarly to the second order matrix in which the number arising from the $\dot{b}(t_1)\dot{a}(t_2)$ term is in general different from the number obtained from the $\dot{a}(t_1)\dot{b}(t_2)$ term.

The signature features are in close mathematical analogy to moments and thus polynomials on the domain of multi-dimensional series. Namely, one can show:
 \begin{itemize}
\item A (sufficiently regular) series is (almost) uniquely determined by their signature - this is not true for higher order moments or cumulants without the order structure (recall that these almost uniquely determine the distribution of values, without order structure).

\item Any (sufficiently regular real-valued) function $f$ on series
  can be arbitrarily approximated by a function linear in signature
  features, that is for a non-linear functionals $f$ of our two-dimensional path $x=\left( a,b \right)^{\top}$,
\begin{align*}
  f(x)\approx\alpha+\sum \beta_{i_1,\ldots,i_{\IndLvl}}\int \diff
  x_{i_1}\cdots \diff x_{i_{\IndLvl}},
\end{align*}
where the sum runs over $\IndLvl$ and $(i_1,\ldots,i_{\IndLvl})\in\{1,2\}^\IndLvl$
and we denote $x_1:=a$,$x_2:=b$. Note that $x$ is the indeterminate here, that $\int
\diff x_{i_1}\cdots \diff x_{i_{\IndLvl}}$ is the degree $\IndLvl$ part of the signature and
approximation is over a compact set of different paths $x$. The exact statements are given in Section~\ref{sec:signature_features}.
\end{itemize}
From a methodological viewpoint, these assertions mean that not only are signature features rich enough to capture all relevant features of the sequential data, but also that any practically relevant feature can be expressed \emph{linearly} in the signature, addressing point \ref{A} in the sense of a universal methodological approach.

Unfortunately, native signature features, in the form above are only
practical in low dimension and low degree $\IndLvl$:
already in the example above of a two-dimensional path $x$, there are
$2^{\IndLvl}$ (scalar) signature features of degree $\IndLvl$, in
general computation of a larger number of signature features is
infeasible, point \ref{C}
Further, all data are discrete sequences not continuous,
and possibly of objects which are not necessarily real vectors; point \ref{B}.

\subsection{The sequential kernel and sequentialization}
\label{sec:intro.B}
The two issues mentioned can be addressed by the kernel trick --- more precisely, by the kernel trick applied twice: once, to cope with the combinatorial explosion of signature features, akin to the polynomial kernel which prevents computation of an exponential number of polynomial features; a second time, to allow treatment of sequences of arbitrary objects. This double kernelization addresses point \ref{B}, and also the combinatorial explosion of the feature space. An additional discretization-approximation, which we discuss in the next paragraph below, makes the so-defined kernel amenable to efficient computation.

We describe the two kernelization steps. The first kernelization step addresses the combinatorial explosion. It simply consists of taking the scalar product of signature features as kernel, and then observing that this scalar product of integrals is an integral of scalar products.
More precisely, this kernel, called \emph{signature kernel} can be
defined as follows, continuing the example with two two-dimensional
sequences $t\mapsto x(t)=(a(t),b(t))^{\top}$,
$t\mapsto\bar{x}(t)=(\bar{a}(t),\bar{b}(t))^{\top}$ as above:
\begin{align*}
\KSigA(x,\bar{x}) := \langle \Sig(x),\Sig(\bar{x}) \rangle = 1 +
  \langle \Sig_1(x),\Sig_1(\bar{x}) \rangle + \langle
  \Sig_2(x),\Sig_2(\bar{x}) \rangle + \cdots.
\end{align*}
The scalar product of $\Sig_1$-s (vectors in $\RR^2$) is the Euclidean scalar product in $\RR^2$, the scalar product of $\Sig_2$-s (matrices in $\RR^{2\times 2}$) is the trace product in $\RR^{2\time 2}$, and so on (with higher degree tensor trace products). The ``1'' is an ``$\Sig_0$''-contribution (for mathematical reasons becoming apparent in paragraph~\ref{sec:intro.C} below).

For the first degree contribution to the signature kernel, one now notes that
\begin{align*}
\langle \Sig_1(x),\Sig_1(\overline{x}) \rangle & = \EE_s \left[\dot a(s)\right] \cdot \EE_t \left[\dot{\overline{a}}(t)\right] + \EE_s \left[\dot b(s)\right] \cdot \EE_t \left[\dot{\overline{b}}(t)\right]\\
& = \EE_{s,t}\left[ \dot a(s) \cdot \dot{\overline{a}}(t) + \dot b(s) \cdot \dot{\overline{a}}(t)\right]\\
& = \EE_{s,t}\left[ \langle\dot x(s), \dot{\overline{x}}(t)\rangle \right].
\end{align*}
In analogy, one computes that the second degree contribution to the signature kernel evaluates to
$$\langle \Sig_2(x),\Sig_2(\overline{x}) \rangle = \frac{1}{2!^2}\EE_{s_1<s_2,t_1<t_2}\left[ \langle\dot x(s_1), \dot{\overline{x}}(t_1)\rangle\cdot \langle\dot x(s_2), \dot{\overline{x}}(t_2)\rangle \right].$$
Similarly, for a higher degree $\IndLvl$, one obtains a product of $\IndLvl$ scalar products in the expectation.

The presentation is not only reminiscent of the polynomial kernel in how it copes with the combinatorial explosion, it also directly suggests the second kernelization to cope with sequences of arbitrary objects: since the sequential kernel is now entirely expressed in scalar products in $\RR^2$, the scalar products in the expectation may be replaced by any kernel $\kernel: \calX\times \calX \rightarrow \RR$, of arbitrary objects, yielding a sequential kernel, now for sequences in $\calX$, given as
\begin{align}\label{eq:sequnentialization}
\KSigB(x,\bar{x}) := 1 + \frac{1}{1!^2}\EE_{s,t}\left[ \kernel(\dot x(s), \dot{\overline{x}}(t)) \right] +
\frac{1}{2!^2}\EE_{s_1<s_2,t_1<t_2}\left[ \kernel(\dot x(s_1),
  \dot{\overline{x}}(t_1))\cdot \kernel(\dot x(s_2),
  \dot{\overline{x}}(t_2)) \right] + \frac{1}{3!^2}\EE \dots
\end{align}
(for expository convenience we assume here that differentials in $\calX$
are defined which in general is untrue, see Section
\ref{Sig_kernel_trick_2} for the general statement).
Note that (\ref{eq:sequnentialization}) can be seen as a process that takes any kernel $\kernel$ on $\calX$, and makes it into a kernel $\KSigB$ on $\calX$-sequences, therefore we term it ``sequentialization'' of the kernel $\kernel$. This addresses point \ref{B}, and can be found in more detail in Section~\ref{sec:kernelized_signatures}.

\subsection{Efficient computation and discretization}
\label{sec:intro.C}
An efficient way of evaluating the sequential kernel is suggested by a second observation,
closely related to (and generalizing) Horner's method of evaluating
polynomials.
Note that the sequential kernel can be written as an iterated conditional expectation
$$\KSigB(x,\bar{x}) = \left(1 + \EE_{s_1,t_1}\left[ \kernel(\dot x(s_1), \dot{\overline{x}}(t_1))\cdot \left(1+\frac{1}{2^2}\EE_{s_1<s_2,t_1<t_2}\left[ \kernel(\dot x(s_2), \dot{\overline{x}}(t_2))\cdot \left(1+\frac{1}{3^2}\EE_{s_2<s_3,t_2<t_3}[\dots ] \right)\right]\right)\right]\right).$$
The iterated expectation directly suggests a discretization by
replacing expectations by sums, such as
$$\left(1 + \frac{1}{n^2}\sum_{s_1,t_1}\left[ \kernel(\dot x(s_1), \dot{\overline{x}}(t_1))\cdot \left(1+\frac{1}{(2n)^2} \sum_{s_1<s_2,t_1<t_2}\left[ \kernel(\dot x(s_2), \dot{\overline{x}}(t_2))\cdot \left(1+ \frac{1}{(3n)^2}\sum_{s_2<s_3,t_2<t_3}[\dots ] \right)\right]\right)\right]\right),$$
where the sums range over a discrete set of points
$s_i,t_i\in[0,1]$, $i=1,\ldots L$. A reader familiar with string kernels will immediately notice the similarity: the sequential kernel can in fact be seen as infinitesimal limit of a string kernel, and the (vanilla) string kernel can be obtained as a special case (see Section~\ref{sec:discr.string}).
As a final subtlety, we note that the derivatives of $x,\overline{x}$ will not be known in observations, therefore one needs to replace
$\kernel(\dot x(s_i), \dot{\overline{x}}(t_j))$ by a
discrete difference approximation
\begin{align*}
\kernel( x(s_{i+1}), {\overline{x}}(t_{j+1}))+ \kernel( x(s_i), {\overline{x}}(t_j)) - \kernel( x(s_i), {\overline{x}}(t_{j+1}))
- \kernel( x(s_{i+1}), {\overline{x}}(t_j))
\end{align*}
where $s_i,s_{i+1}$ resp.~$t_j,t_{j+1}$ denote adjacent support values of the discretization.

Our theoretical results in Section~\ref{sec:approx} show that the discretization, as described above, converges to the continuous kernel, with a convergence order linear in the sampling density.
Moreover, similarly to the Horner scheme for polynomials (or fast
string kernel techniques), the iterated sum-product can be efficiently
evaluated by dynamical programming techniques on arrays of dimension
three, as outlined in Section~\ref{sec:comp}. The computational
complexity is quadratic in the length of the sequences and linear in
the degree of approximation, and can be further reduced to linear
complexity in both with low-rank techniques.

This addresses the remaining point \ref{C} and therefore yields an efficently computable, canonical and universal kernel for sequences of arbitrary objects.

\subsection{Prior art}

Prior art relevant to learning with sequential data may be found in
three areas:
\begin{enumerate}
\item dynamic programming algorithms for sequence comparison in the
  engineering community,
\item kernel learning and Gaussian processes in the machine learning
  community
\item rough paths in the stochastic analysis community.
\end{enumerate}
The dynamic programming literature (1) from the 70's and 80's has
inspired some of the progress in kernels (2) for sequential data over
the last decade, but to our knowledge so far no connections have been
made between these two, and (3), even though (3) pre-dates kernel literature for sequences by more than a decade.
Beyond the above, we are not aware of literature in statistics of time
series that deals with sequence-valued data points in a way other than
first identifying one-dimensional sequences with real vectors of same
size, or even forgetting the sequence structure entirely and replacing
the sequences with (order-agnostic) aggregates such as cumulants,
quantiles or principal component scores (this is equally true for
forecasting methods).
Though, simple as such reduction to more classic situations may be, it constitutes an important baseline, since only in comparison one can infer that the ordering was informative or not.

\subsubsection*{Dynamic programming for sequence comparison.}
The earliest occurrence in which the genuine order information in sequences is used for learning can probably be found in the work of Sakoe et al~\cite{sakoe1970similarity,sakoe1979two} which introduces the idea of using editing or distortion distances to compare sequences of different length, and to efficiently determine such distances via dynamic programming strategies. These distances are then employed for classification by maximum similarity/minimum distance principles. Through theoretical appeal and efficient computability, sequence comparison methods, later synonymously called dynamic time warping methods, have become one of the standard methods in comparing sequential data~\cite{kruskal1983overview, giorgino2009computing}.

Though it may need to be said that sequence comparison methods in their pure form --- namely an efficiently computable distance between sequences --- have remained somewhat restricted in that they can only be directly adapted only to relatively heuristic distance-based learning algorithms, by definition. This may be one of the reasons why sequence comparison/dynamic time warping methods have not given rise to a closed learning theory, and why in their original practical application, speech recognition and speech classification, they have later been superseded by Hidden Markov Models~\cite{rabiner1989tutorial} and more recently by neural network/deep learning methodologies~\cite{hinton2012deep} as gold standard.

A possible solution of the above-mentioned shortcomings has been demonstrated in kernel learning literature.

\subsubsection*{Kernels for sequences.}
Kernel learning is a relatively new field, providing a general
framework to make non-linear data of arbitrary kind amenable to
classical and scalable linear algorithms such as regression or the
support vector machine in a unified way, by using a non-linear scalar
product: this strategy is called ``kernelization'';
see~\cite{scholkopf2002learning} or~\cite{shawe2004kernel}.
Mathematically, there are close relations to Gaussian process theory~\cite{rasmussen2006gaussian} which is often considered as a complimentary viewpoint to kernels, and aspects of spatial geostatistics~\cite{cressie2015statistics}, particularly Kriging, an interpolation/prediction method from the 60's~\cite{matheron1963principles} which has been re-discovered 30 years later in the form of Gaussian process regression~\cite{williams1996gaussian}.
In all three incarnations, coping with a specific kind of data practically reduces to finding a suitable kernel (= covariance function), or a family of kernels for the type of data at hand --- after which one can apply a ready arsenal of learning theory and non-linear methodology to such data. In this sense, providing suitable kernels for sequences has proved to be one of the main strategies in removing the shortcomings of the sequence comparison approach.

Kernels for strings, that is, sequences of symbols, were among the first to be considered~\cite{haussler1999convolution}. Fast dynamic programming algorithms to compute string kernels were obtained a few years later~\cite{lohdi02textclassification,leslie04faststringkernels}. Almost in parallel and somewhat separately, kernels based on the above-mentioned dynamic time warping approach were developed, for sequences of arbitrary objects~\cite{bahlmann2002online, noma2002dynamic}. A re-formulation/modification led to the so-called global alignment kernels~\cite{cuturi2007kernel}, for which later fast dynamic programming algorithms were found~\cite{cuturi2011fast} as well. An interesting subtle issue common to both strains was that initially, the dynamic programming algorithms found were quadratic in the length of the sequences, and only later linear complexity algorithms were devised: for string kernels, the transition was made in~\cite{leslie04faststringkernels}, while for the sequence matching strain, this became only possible after passing to the global alignment kernels~\cite{cuturi2011fast}.

Looking from a general perspective: while in hindsight all of the mentioned kernels can be viewed from Haussler's original, visionary relation-convolution kernel framework, and all above-mentioned kernels for sequences, in some form, admit fast dynamic programming algorithms, existing literature provides no unifying view on kernels for sequences: the exact relation between string kernels and dynamic time warping/global alignment kernels, or to the classical theory of time series has remained unclear; further, the only known kernels for sequences of arbitrary objects, the dynamic time warping/global alignment kernels, suffer from the fact that they are not proper kernels, failing to be positive definite.

In this paper, we attempt to resolve these issues. More precisely, the
string kernel will arise as a special case, and the global alignment
kernel as a deficient version of our new signature kernel, built on
the theory of signatures and rough paths from stochastic analysis.

\subsubsection*{Iterated integrals, signatures and rough paths.}
Series of iterated integrals are a classic mathematical object
that plays a fundamental role in many areas like control theory,
combinatorics, homotopy theory,
Feynman--Dyson--Schwinger theory in physics and more recently
probability theory.
We refer to \cite[Section ``Historic papers'', p97]{lyons2004stflour}
for a bibliography of influential articles.
This series, or certain aspects of it, are treated under various names in the
literature like ``Magnus expansion'', ``time ordered exponential'', or
the one we chose which comes from Lyons' rough path theory: ``the signature of a path''.
The reason we work in the rough path setting is that it provides a
concise mathematical framework that clearly separates
analytic and algebraic aspects, applies in
infinite dimensional spaces like our RKHS and is robust under noise;
we refer to
\cite{lyons2004stflour,lyons1998differential,friz2010multidimensional}
as introductions.
The role of the signature as a ``non-commutative exponential'' plays a guiding principle for many recent
developments in stochastic analysis, though it might be less known outside this community.

The major application of rough path theory was and still is to provide
a robust understanding of differential
equations that are perturbed by noise, going beyond classic
Ito-calculus; Hairer's work on regularity
structures~\cite{hairer2014theory} which was recently awared a Fields
medal can be seen as vast generalization of such ideas.
The interpretation of the signature in a statistical context is more
recent: work of Papavasiliou and Ladroue \cite{papavasiliou2011parameter} applies it to SDE parameter
estimation;
work of Gyurko, Hao, Lyons, Oberhauser
\cite{gyurko2013extracting,levin2013learning, lyons2014feature}
applies it to forecasting and classifcation of financial time series
using linear and logistic regression;
work of Diehl~\cite{Diehlinvariants} and Graham
\cite{DBLP:journals/corr/Graham13} uses signature features for
handwritten digit recognition, see also
\cite{DBLP:journals/corr/YangJL15} for more recent state of the art
results.

The interpretation of the signature as an expectation already occurs as a technical Lemma~3.9 in~\cite{hambly2010uniqueness}. A scalar product formula for the norm, somewhat reminiscent of that of the sequential kernel, can be found in the same Lemma.
Similarly we would like to mention the Computational Rough Paths
package~\cite{CoRoPa}, that contains C\texttt{++} code to compute
signature features directly for $\RR^d$-valued paths.
However, it does not contain specialized code to calculate inner
products of signature features directly. The Horner-type algorithm
we describe in Section \ref{sec:intro.C} gives already significant speed improvements when it is applied to
paths in finite dimensional, linear spaces (that is, the sequentialization of the
Euclidean inner product $\kernel(\cdot,\cdot)=\langle \cdot,\cdot
\rangle_{\RR^d}$; see Remark \ref{rem:inner_product_pcw_linear}).

\begin{Rem}
The above overview contains a substantial number of examples in which
independent or parallel development of ideas related to sequential
data has occurred, possibly due to researchers being unaware of
similar ideas in communities socially far from their own.
We are aware that this could therefore also be the case for this very paper, unintended.
We are thus especially thankful for any pointers from readers of this pre-print
version that help us give credit where credit is due.
\end{Rem}

\section*{Acknowledgement}
HO is grateful for support by the Oxford-Man Institute of Quantitative finance.

\section{Notation for ordered data}

\label{sec:sig.ord}
We introduce recurring notation.

\begin{Def}
The set of natural numbers, including $0$, will be denoted by $\NN$.
\end{Def}

\begin{Def}\label{def:ordered_tuples}
Let $A$ be a set and $L\in\NN$. We denote
\begin{enumerate}
\item the set of integers between $1$ and $L$ by $[L]:=\{1,2,\dots, L\}$,
\item the set of ordered $L$-tuples in $A$ as usual by $A^L:=\{\ba =
  (a_1,\dots,a_L):\; a_1, \dots, a_L\in A\}$,
\item the set of such tuples, of arbitrary but finite length, by
  $A^+:=\bigcup_{L\in\NN}A^L$ where by convention $A^0=\emptyset$.
\end{enumerate}
Moreover, we use the following index notation for
$\ba=(a_1,\ldots,a_L)\in A^+$,
\begin{enumerate}
\item $\ell(\ba):=L$,
\item $\#\ba$ the count of the most frequent item in $\ba$,
\item $\ba[i]:=a_i$ for $i=1,\ldots,L$,
\item $\ba[1:N]:=(\ba[1],\ldots,\ba[N])$ for $N\in\left[ L \right]$,
\item $\ba[\bi]:=(\ba[i_1],\ldots,\ba[i_N])\in A^N$ for $\bi=\left(
    i_1,\ldots,i_N\right)\in \left[ L \right]^N$,
\item $\ba!:=n_1!\cdots n_k!$ if $\ba$
  consists of $k=|\{a_1,\ldots,a_L\}|$ different elements in $A$ and
  $n_1,\ldots,n_k$ denote the number of times they occur in $\ba$,
\item $f(\ba)=\left( f(a_1),\ldots,f(a_L) \right)\in B^L$ for $f\in[A\rightarrow B]$.
\end{enumerate}
\end{Def}
In the case that $A\subset\RR$, we can define subsets of
$A^+$ that consist of increasing tuples.
These tuples play an important role for calculations with the signature features.
\begin{Def}
Let $A\subset \RR$ and $\IndLvl\in\NN$.
We denote the set of monotonously ordered $\IndLvl$-tuples in $S$ by
\begin{align*}
\simplex^{\IndLvl}\left(  A \right)&:=\{\bu\in A^{\IndLvl}:\bu[1]\le \bu[2] \le \cdots  \le \bu[\IndLvl]\}.
\end{align*}
We denote the union of such tuples by $\simplex\left(
  A\right):=\bigcup_{\IndLvl\in\NN}\simplex^{\IndLvl}\left(
  A \right)$ where again by convention $\simplex^0\left( A \right)=\emptyset$.
We call $\simplex^{\IndLvl}(A)$ the order $\IndLvl$ simplex on $A$, and $\simplex(A)$ the order simplex on $A$.
A monotonously ordered tuple $\bu\in \simplex^{\IndLvl} (A)$ is called strict if
$\bu[i]\lneq \bu[i+1]$ for all $i\in [\IndLvl-1]$.
The index notation of Definition \ref{def:ordered_tuples} applies also
to $\bu\in\simplex (A)$, understanding that $\simplex (A)
\subseteq A^+$ if $A\subseteq \RR$.
\end{Def}
\begin{Rem}
Above definition of order simplex is slightly different from the usual one, in which one takes
$A=\left[ 0,1 \right]$, the counting starts at $t_0$, and one has $t_0=0$ and $t_{\IndLvl}=1$.
\end{Rem}
The following notation is less standard, but becomes very useful for the
algorithms and recursions that we discuss.
\begin{Def}
Let $\ba,\bb\in A^+$, $D\in\NN$. We use the following notation:
\begin{itemize}
\item $\ba\sqsubseteq \bb$ if there is a $\bi\in\simplex (\NN)$ such
  that $\ba=\bb[\bi]$,
\item $\ba\sqsubset\bb$ if there is strictly ordered tuple
  $\bi\in\simplex (\NN)$ such that $\ba=\bb[\bi]$,
\item $\ba\sqsubseteq_D\bb$ if there is a $\bi\in\simplex (\NN)$ such
  that $\ba=\bb[\bi]$ and $\#\ba\leq D$.
\end{itemize}
For $L\in\NN$ we also the notation
\begin{enumerate}
\item $\ba\sqsubseteq [L]$ if $\ba\in\simplex([L])$,
\item $\ba\sqsubset [L]$ if $\ba\in\simplex([L])$ and $\ba$ is
  strict (that is $\ba[i]\lneq \ba[i+1]$ for all $i\in[\ell(\ba)-1]$).
\end{enumerate}
\end{Def}

\section{Signature features: ordered moments for sequential data}\label{sec:signature_features}
In this section, the signature features are introduced in a
mathematically precise way, and the properties which make them
canonical features for sequences are derived.
We refer the interested reader to \cite{lyons2004stflour} for further
properties of signatures and its use in stochastic analysis.

As outlined in introductory Section~\ref{sec:intro.sign}, the
signature can be understood as an ordered and hence non-commutative
variant of sample moments.
If we model ordered data by a function $x\in [[0,1]\rightarrow \RR^n]$, the
signature features are obtained as iterated integrals of $x$ and the $M$-times
iterated integral can be interpreted as a $M$-th ordered moment;
examples of such features for $M=1,2$ are the (non-commutative) integral moments
$$\Sig_1: \left[[0,1]\rightarrow \RR^n\right]\rightarrow \RR^n\;,\;x\mapsto \int_0^1 \diff x(t),\quad\mbox{or}\quad \Sig_2: \left[[0,1]\rightarrow \RR^n\right]\rightarrow \RR^{n\times n}\;,\;x\mapsto \int_0^1\int_{0}^{t_2} \diff x(t_1)\diff x(t_2)^\top$$
(where the domain of $x$ will be restricted so the integrals are well-defined).
The more general idea of signature features, made mathematically precise, is to consider the integral moments
\begin{align}
\label{signature_again}
\Sig_M (x)= \left( \int_{\simplex^{\IndLvl}} \diff x_{i_1}(t_1)\cdots \diff
  x_{i_\IndLvl}(t_{\IndLvl})\right)_{i_1,\ldots,i_\IndLvl\in\{1,\ldots,n\}}\in
  \RR^{n\times \cdots \times n}
\end{align}
where the integration is over the $\IndLvl$-simplex $\simplex^{\IndLvl}$, i.e., all ordered sequences $0\le t_1\le t_2\le \dots \le t_{\IndLvl} \le 1,$
and the choice of index $(i_1,\dots, i_{\IndLvl})$ parametrises the features.
The features $\Sig_1(x),\Sig_2(x),\ldots$ are all element of a (graded) linear space and a kernel for sequential data may then be obtained from taking the scalar product over the signature features.

The section is devoted to a description and charactarization of these signature features and the kernel obtained from it. This is done in a basis-free form which allows treatment of ordered data $x:[0,1]\rightarrow \calH$, where $\calH$ is a Hilbert space (over $\RR$) not necessarily identical to $\RR^n$. This will allow considering ordered versions of data, the non-ordered variant of which may be already encoded via its reproducing kernel Hilbert space $\calH$.

For illustration, the reader is invited to keep the prototypical case $\calH =\RR^n$ in mind.

\subsection{Paths of bounded variation}
\label{sec:sig.BV}
The main object modelling a (generative) datum will be a continuous sequence, a path $x:[0,1]\rightarrow \calH$ where $\calH$ is a Hilbert space which is fixed in the following.
For reasons of regularity (integrals need to converge), we will
restrict ourselves to paths $[[0,1]\rightarrow \calH]$ which are
continuous and of bounded length\footnote{Considering bounded variation paths is slightly restrictive as it excludes samples from stochastic process models such as the Wiener process/Brownian motion. The theory may be extended to such sequences at the cost of an increase in technicality. For reading convenience, this will be done only at a later stage.} (called ``variation'', as defined below).

The variation of a path is defined as the supremum of variations of discrete sub-sequences taken from the path:

\begin{Def}
For an ordered tuple $\bx\in \calH^{L+1},$ we define:
\begin{enumerate}[label=(\roman*)]
\item $\seqdf \bx\in \calH^L$ as $\seqdf \bx[i] := \bx[i+1] -
  \bx[i]$,  $i\in [L]$,
\item $\mesh\left(\bx\right) := \max_i \|(\seqdf \bx)[i]\|$,
\item $\Var(\bx) := \sum_{i=1}^{L} \|\seqdf \bx \|_\calH$.
\end{enumerate}
We call $\seqdf \bx$ the first difference sequence of $\bx$, we call $\mesh (\bx)$ the mesh of $\bx$, and $\Var(\bx)$ the variation of $\bx$.
\end{Def}

\begin{Def}
We define the variation of $x\in[[0,1]\rightarrow\calH]$ as
$$ \Var(x):=\sup_{\bt\in\simplex ([0,1])} \Var(x(\bt)).$$
The mapping $x$ is said to be of bounded variation if $\Var(x) < \infty$.
\end{Def}

A bounded variation path, as the name says, is a path with bounded variation.

\begin{Def}
Let $[a,b]\subseteq[0,1]$. We denote the set of $\calH$-valued
paths of bounded variation on $[a,b]$ by
\begin{align*}
\BV([a,b],\calH):=\{x\in C\left(  [a,b],\calH\right)\;:\; \Var(x) < \infty\}.
\end{align*}
When $[a,b]=[0,1]$ we often omit the qualifier $[0,1]$ and write $\BV (\calH)$.
\end{Def}

\begin{Def}
\label{RS_integral}
Let $E$ be a linear space and denote with $L\left(\calH,E\right)$ the set of continuous linear maps from $\calH$ to $E$.
Given $x\in \BV([0,1],\calH)$ and $y\in \BV\left(\left[ 0,1 \right],L\left(\calH\right)\right)$,
the Riemann--Stieltjes integral of $y$ over $\left[ a,b \right]\subseteq\left[ 0,1 \right]$ is defined as the element in $E$ given as
$$\int_a^b y \diff x :=\lim_{\substack{\bt\in\simplex\left( \left[
        a,b \right] \right)\\\mesh\left(\bt\right)\rightarrow
    0}}\sum_{i=1}^{\ell(\bt)-1} y(\bt[i]) \left( \seqdf x(\bt[i])\right).$$
We also use the shorter notation $\int y\diff x$ when the integration domain is clear from the context.
\end{Def}
As in the finite-dimensional case, above is indeed well-defined, that
is the Riemann sums converge and the limit is independent of the
sequence $\bt$; see \cite[Theorem 1.16]{lyons2004stflour} for a proof
of a more general result.
Note that the integral itself is in general not a scalar, but an element of the Hilbert space $\calH$.

\subsection{Signature integrals in Hilbert spaces}
\label{sec:sig.int}
With the Riemann--Stieltjes integral, we can define the signature integrals in a basis-free way. For this, note that the Riemann--Stieltjes integral of a bounded variation path is again a bounded variation path.
\begin{Def}
Let $[a,b]\subseteq[0,1]$ and  $x\in\BV ([a,b], \calH)$. We define:
\begin{enumerate}[label=(\roman*)]
\item $\int_{\simplex^1 ([a,b])}\diff x^{\otimes 1} := \int_a^b\diff x=x(b)-x(a)$,
\item$\int_{\simplex^{\IndLvl} ([a,b])} \diff x^{\otimes \IndLvl} : = \int_a^b\left( \int_{\simplex^{\IndLvl-1}([a,s])}\diff x^{\otimes (\IndLvl-1)} \right)\otimes \diff x (s)$ for integers $\IndLvl\ge 2$.
\end{enumerate}
We call $\int_{\simplex^\IndLvl ([a,b])} \diff x ^{\otimes \IndLvl}$
the $\IndLvl$-th iterated integral of $x$ on $[a,b]$.
When $[a,b]=[0,1]$ we often omit the qualifier $[0,1]$ and write
$\simplex^{\IndLvl}$ for $\simplex^{\IndLvl}([0,1])$.
\end{Def}

In the prototypical case $\calH=\RR^n$, the first iterated integral is a vector in $\RR^n$, the second iterated integral is a matrix in $\RR^{n\times n}$, the third iterated integral is a tensor in $\RR^{n\times n\times n}$ and so on.
For the case of arbitrary Hilbert spaces, we need to introduce some additional notation to write down the space where the iterated integrals live:
\begin{Def}
We denote by $\calH \otimes \calH$ the tensor (=outer product) product
of $\calH$ with itself.
Instead of $\calH\otimes \calH\otimes \dots \otimes \calH$ ($\IndLvl$
times), we also write $\calH^{\otimes \IndLvl}$. By convention,
$\calH^{\otimes 0} = \RR$.
\end{Def}

The $\IndLvl$-th iterated integral of $x$ is an element of $\calH^{\otimes \IndLvl}$, a tensor of degree $\IndLvl$.
The natural space of iterated integrals of all degrees together is the tensor power algebra the Hilbert space $\calH$:

\begin{Def}
The tensor power algebra over $\calH$ is defined as $\bigoplus_{\IndLvl=0}^\infty \calH^{\otimes \IndLvl}$,
addition $+$, multiplication $\ast$ and an inner product $\langle
\cdot,\cdot\rangle_{\tensalg}$ are defined by canonical extension from $\calH$, that is:
\begin{align*}
&(g_0,g_1,\dots,g_{\IndLvl},\dots) + (h_0,h_1,\dots,h_{\IndLvl},\dots):=
  \left(g_0+h_0\;g_1+ h_1, \ldots\,g_{\IndLvl}+h_{\IndLvl}\;,\ldots \right).\\
&(g_0,g_1,\dots,g_{\IndLvl},\dots)\ast (h_0,h_1,\dots,h_{\IndLvl},\dots):=
  \left(g_0h_0\;,\; g_0\otimes h_1+g_1\otimes h_0\;,
  \;\dots\;,\;\sum_{i=0}^{\IndLvl} g_i\otimes h_{\IndLvl-i}\;,\;\dots \right),\\
&\left\langle (g_0,g_1,\dots,g_{\IndLvl},\dots), (h_0,h_1,\dots,h_{\IndLvl},\dots)\right\rangle_{\tensalg}:=
\langle g_0,h_0\rangle_\RR + \langle g_1h_1\rangle_\calH+\dots + \langle g_{\IndLvl},h_{\IndLvl}\rangle_{\calH^{\otimes \IndLvl}} +\dots.
\end{align*}
The elements in the tensor power algebra with
with finite norm $\left\|g\right\|_{\tensalg}=\sqrt{\langle g,g
\rangle_{\tensalg}}<\infty$ are denoted by
$\tensalg$\footnote{This is slightly non-standard notation: usually
  $\tensalg$ equals $\bigoplus_{\IndLvl=0}^\infty \calH^{\otimes
    \IndLvl}$.}.
Further, we canonically identify $\calH^{\otimes \IndLvl}$ with the sub-Hilbert space of $\tensalg$ that contains exactly elements of the type $g=(0,\dots,0,g_{\IndLvl},0,\dots)$, which we call homogenous (of degree $k$).
Under this identification, for $g\in \tensalg$, we will also write, for reading convenience,
$$g_0+g_1+g_2+\dots\quad\mbox{instead of}\quad (g_0,g_1,g_2,\dots),$$
where we may opt to omit zeros in the sum. We adopt an analogue use of sum and product signs $\sum,\prod$.
\end{Def}

\begin{Ex}
We work out tensor algebra multiplication and scalar product in the case of the prototypical example $\calH = \RR^n$. Consider homogenous elements of degree 2: it holds that $\calH^{\otimes 2}=\RR^n\otimes \RR^n\cong \RR^{n\times n}$, and the tensor product of two vectors is $v\otimes w = vw^\top$. The trace product on $\RR^{n\times n}$ is induced by the Euclidean scalar product, since $\langle v_1w_1^\top, v_2w_2^\top\rangle = \langle v_1,v_2\rangle\langle w_1,w_2\rangle$. Homogenous elements of degree 3 are similar: it holds that $\calH^{\otimes 3} \cong \RR^{n\times n\times n}$, and the tensor product of three vectors $u\otimes v\otimes w$ is a tensor $T$ of degree $3$, where $T_{ijk}=u_iv_jw_k$ for all $i,j,k\in [n]$. The scalar product on $\RR^{n\times n\times n}$ is the tensor trace product, $\langle T, T'\rangle_{\calH^{\otimes 3}} = \sum_{i,j,k=1}^n T_{ijk}T'_{ijk}$.\\
An element $g\in \tensalg$ is of the form
$$g = c + v + M + T + \dots,\quad\mbox{where}\;c\in \RR, v\in \RR^n, M\in \RR^{n\times n}, T\in \RR^{n\times n \times n}, \mbox{etc}.$$
As an example of tensor algebra multiplication, it holds that
$$g\ast g = c^2 + 2cv + cM + vv^\top + cT + v\otimes M + M\otimes v+\dots.$$
Note the difference between $v\otimes M$ and $M\otimes v$: it holds that $(v\otimes M)_{ijk} = v_i M_{jk},$ while $(M\otimes v)_{ijk} = v_k M_{ij}.$
\end{Ex}
One checks that $\tensalg$ is indeed an algebra:

\begin{Prop}$\left( \tensalg, +, \ast \right)$ is a an associative $\RR$-algebra with
\[
\left( 1,0,\ldots,0,\ldots \right) \text{ resp. }\left( 0,\ldots,0,\ldots \right)
\]
as multiplicative neutral element resp.~additive neutral element.
In general, the tensor algebra multiplication $\ast$ is not commutative.
\end{Prop}
\begin{proof}
Verifying the axioms of an associative $\RR$-algebra is a series of non-trivial but elementary calculations.
To see that $\ast$ is not commutative, consider the counter-example where $g,h\in \calH$ are linearly independent. Then, $g\otimes h \neq h\otimes g$ (in the case of $\calH =\RR^n$, this is $gh^\top \neq hg^\top$).
\end{proof}

\begin{Rem}
We further emphasize the following points, also to a reader who may already be familiar with the tensor product/outer product:
\begin{itemize}
\item[(i)] The Cartesian product $\calH^{\IndLvl}$ is different from the tensor product $\calH^{\otimes \IndLvl}$ in the same way as $\left(\RR^{n}\right)^2$ is different from $\RR^{n\times n} = \left(\RR^{n}\right)^{\otimes 2}$ and $\left(\RR^{n}\right)^3$ from $\RR^{n\times n \times n} = \left(\RR^{n}\right)^{\otimes 3}.$
\item[(ii)] In general, $\ast$ is different from the formal tensor
    product/outer product of elements, since for $g,h\in
    \tensalg$, the formal tensor product $g\otimes h$ is
    an element of $\tensalg\otimes
    \tensalg$, while the tensor power algebra product
    $g\ast h$ is an element of $\tensalg$.
\item[(iii)] Under the identification introduced above, there is one case where $\ast$ coincides with the
  tensor product - namely, when $g$ and $h$ are homogenous.
Identifying $g\in \calH^{\otimes m}$ and $h\in \calH^{\otimes n}$, it
holds that $g\otimes h \in \calH^{\otimes (m + n)}$ may be identified
with $g\ast h$ which is also homogenous.
No equivalence of this kind holds when $g$ and $h$ are not homogenous.

\end{itemize}
\end{Rem}

\subsection{The signature as a canoncial feature set}

We are now ready to define the signature features.
\begin{Def}
We call the mapping
\begin{align*}
\Sig: \BV\left( \left[ 0,1 \right],\calH \right)\rightarrow \tensalg \;,\;x\mapsto \sum_{\IndLvl\geq
  0}\int_{\simplex^{\IndLvl} \left( \left[ 0,1 \right] \right)}(\diff x)^{\otimes \IndLvl}
\end{align*}
the signature map of $\calH$-valued paths and we refer to $\Sig(x)$, as the signature features of $x\in\BV$.
Similarly, we define (level-$\IndLvl$-)truncated signature mapping as
\begin{align*}
\Sigk: \BV([0,1],\calH)\rightarrow \tensalg \;,\;x\mapsto \sum_{m=0}^{\IndLvl} \int_{\simplex^{m} ([0,1])} \diff x^{\otimes \IndLvl}.
\end{align*}
\end{Def}
Above is well-defined since the signature can be shown to have finite norm:
\begin{Lem}
\label{Lem:finsig}
Let $x\in \BV([0,1],\calH)$. Then:
\begin{enumerate}[label=(\roman*)]
\item $\left\|\int_{\simplex^{\IndLvl} [0,1]}\diff x^{\otimes \IndLvl}\right\|_{\calH^{\otimes \IndLvl}}\le \frac{1}{\IndLvl!}\Var(x)^{\IndLvl} < \infty$,
\item $\|\Sig(x)\|_{\tensalg}\le \exp(\Var(x)) < \infty$.
\end{enumerate}
\end{Lem}
\begin{proof}

(i) is classical in the literature on bounded variation paths, it is
also proven in Lemma~\ref{Lem:intsimpl} of the appendix. (ii) follows
from (i) by observing that $\|\Sig(x)\|_{\tensalg} \le \sum_{\IndLvl=0}^\infty \left\|\int_{\simplex^\IndLvl [0,1]}\diff x^{\otimes \IndLvl}\right\|_{\calH^{\otimes \IndLvl}}$
due to the triangle equality, then substituting (i) and the Taylor expansion of $\exp$.
\end{proof}

There are several reasons why the signature features are (practically and theoretically) attractive, which we summarize before we state the results exactly:

\begin{enumerate}
\item the signature features are a {\bf mathematically faithful representation of the underlying sequential data} $x$: the map $x\mapsto \Sig(x)$ is essentially one-on-one.
\item the signature features are {\bf sequentially ordered analogues to polynomial features and moments}. The tensor algebra
has the natural grading with $\IndLvl$ designating the ``polynomial degree''. It is further canonically compatible with natural operations on $\BV([0,1],\calH)$.
\item linear combinations of signature features
{\bf approximate continuous functions of sequential data arbitrarily well}. This is in analogy with classic polynomial features and implies that signature features are as rich a class for learning purposes as one can hope.
\end{enumerate}

\begin{Thm}[Signature uniqueness]
\label{Thm:injectivetree}
Let $x,y\in \BV ([0,1],\calH)$. Then $\Sig (x) = \Sig (y)$ if and only $x$
and $y$ are equal up to tree-like equivalence\footnote{We call $x,y$
  tree-like equivalent if  $x\sqcup y^{-1}$ is tree-like.
A path $z\in \BV([a,b])$ is called tree-like if there exists a
continuous map $h:\left[a,b \right]\rightarrow \left[0,\infty\right)$
with $h\left(0\right)=h\left(T\right)=0$ and
$\left|z\left(t\right)-\left(s\right)\right|\leq
h\left(s\right)+h\left(t\right)-2\inf_{u\in\left[s,t\right]}h\left(u\right)$
for all $s\le t\in [a,b]$.
}.
\end{Thm}
\begin{proof}
This is \cite[Theorem 2.29~(ii)]{lyons2004stflour}.
\end{proof}
\begin{rem}
  Being not tree-like equivalent is a very weak requirement, e.g.~if
  $x$ and $y$ have a strictly increasing coordinate they are not
  tree-like equivalent.
  All the data we will encounter in the experiments is not tree-like.
  Even if presented with a tree-like path, simply adding time as extra
  coordinate (that is, working with  $t\mapsto(t,x(t))$ instead of $t\mapsto x(t)$) guarantees the assumptions of above Theorem are met.
\end{rem}
\begin{rem}
Above Theorem extends to unbounded variation paths, cf.~\cite{2014arXiv1406.7871B}
\end{rem}
Secondly, the signature features are analogous to polynomial features:
the tensor algebra has a natural grading with $\IndLvl$ designating the ``polynomial
degree''.
\begin{Thm}[Chen's Theorem]\label{thm:chens}
Let $x\in\BV([0,1],\calH)$, then

\[
\int_{\simplex^M} \diff x^{\otimes M}\otimes \int_{\simplex^N} \diff
x^{\otimes N}=\sum_\sigma \sigma\left(\int_{\simplex^{M+N}} \diff x^{\otimes(M+N)}\right).
\]
Here the sum is taken over all ordered shuffles
\[
\sigma\in\operatorname{OS}_{M,N}=\left\{ \sigma:\sigma\text{ permutation of }\left\{ 1,\ldots,M+N\right\} ,\sigma\left(1\right)<\cdots<\sigma\left(M\right),\sigma\left(M+1\right)<\cdots<\sigma\left(M+N\right)\right\} .
\]
and $\sigma\in\operatorname{OS}_{M,N}$ acts on $\calH^{\otimes\left(M+N\right)}$
as $\sigma\left(e_{i_{1}}\otimes\cdots\otimes e_{i_{M+N}}\right)=e_{\sigma\left(i_{1}\right)}\otimes\cdots\otimes e_{\sigma\left(i_{M+N}\right)}$.
\end{Thm}
Finally, a direct consequence of the above and again in analogy with
classic polynomial features, linear combinations of signature features
approximate continuous functions of sequential(!) data arbitrary well.

\begin{Thm}[Linear approximations]\label{thm:stone-weierstrass}
Let $\mathcal{P}$ be a compact subset of $\BV\left([0,1],\calH \right)$ of paths that are not tree-like equivalent.
Let $f:\mathcal{P}\rightarrow\mathbb{R}$ be continuous in variation norm.
Then for any $\epsilon>0$, there exists a $w\in \tensalg$ such that
\[
\sup_{x\in\mathcal{P}}\left|f\left(x\right)-\left\langle w,\Sig\left(x\right)\right\rangle_{\tensalg} \right|<\epsilon.
\]
\end{Thm}

\begin{proof}
The statement follows from the Stone--Weierstraß theorem if the set $\mathcal{F}\subset C(\mathcal{P},\RR)$
\begin{equation}
\mathcal{F}:=\operatorname{span}\left\{ \mathcal{P}\ni
  x\mapsto\left\langle e_{i_1}\otimes\cdots\otimes e_{i_{\IndLvl}},\operatorname{S}\left(x\right)\right\rangle_{\tensalg},~\IndLvl\in\NN\right\} \label{eq:span}
\end{equation}
forms a point-separating algebra.
However, this is a direct consequence of the above: by Chen's theorem, Theorem \ref{thm:chens}, $\mathcal{F}$ is an
algebra, and by the signature uniqueness, Theorem \ref{Thm:injectivetree}, $\calF$ separates points.
\end{proof}
Above shows more than stated: for a fixed ONB $\left(e_{i}\right)$
of $\tensalg$, there exists a finite subset of this ONB and
$w$ can be found in the linear span of this finite set.

\section{Kernelized signatures and sequentialized kernels}\label{sec:kernelized_signatures}
Our goal is to construct a kernel for sequential data of arbitrary
type, to enable learning with such data.
We proceed in two steps and first discuss the case when the sequential data are sequences in the Hilbert space $\calH$ (for
example $\calH = \RR^n$).
In this scenario, the properties of the signature, presented in
Section~\ref{sec:signature_features}, suggest as kernel the scalar
product of the signature features.
This yields the following kernels,
\begin{Def}
Fix $\IndLvl\in\NN$. We define the kernels
\begin{align*}
\KSigA :& \BV(\calH)\times \BV(\calH)\rightarrow
          \RR\;,\quad (x,y)\mapsto\langle \Sig(x),
          \Sig(y)\rangle_{\tensalg},\\
\KSigAk:& \BV(\calH)\times \BV(\calH)\rightarrow \RR\;,\quad (x,y)\mapsto\langle \Sigk(x), \Sigk (y)\rangle_{\tensalg}.
\end{align*}
We refer to $\KSigA$ as the \emph{signature kernel} and to $\KSigAk$ as
the \emph{signature kernel truncated at level $\IndLvl$}.
\end{Def}
To make these kernels practically meaningful, we need to verify a number of points:
\begin{enumerate}[label=(\alph*)]
\item\label{welldf} That they are well-defined, positive
  (semi-)definite kernels. Note that checking finiteness of the scalar
  product is not immediate (but follows from well-known estimates
  about the signature features).
\item\label{eff} That they are efficiently computable. A naive evaluation is
  infeasible, due to combinatorial explosion of the number of
  signature features. However, we show that $\KSigA$ and $\KSigAk$ can be
  expressed \emph{entirely in integrals of inner products} $\langle
  \diff x(s), \diff y(t) \rangle_{\calH}$. The formula can be written
  as an {\bf efficient recursion}, similar to the Horner scheme for
  efficient evaluation of polynomials.
\item\label{robust} That they are robust under discretization: the
  issue is that paths are never directly observed since all real
  observations are discrete sequences. The subsequent
  Section~\ref{sec:approx} introduces {\bf discretizations for
    signatures}, and the two kernelization steps above, which are {\bf
    canonical and consistent} to the above continuous steps, in a
  sampling sense of discrete sequences $\calH^+$ converging to bounded
  variation path $\BV ([0,1],\calH)$.
\item\label{noise} That they are robust under noise: in most
  situations our measurements of the underlying path are perturbed by random
  perturbations and noise. We discuss the common
  situation of additive white noise/Brownian motion in Section \ref{sec:noise}.

\end{enumerate}
We refer to the above procedure as ``\emph{kernel trick one}'' and discuss it
below in Section \ref{sec:welldf} and Section
\ref{Sig_kernel_trick_1}.

In a second step, to which we refer as ``\emph{kernel trick two}'', we show that the above is also
meaningful for sequential data in an arbitrary set $\calX$.
This second step yields, for any primary kernel
$\kernel$ on (static objects) $\calX$, a sequential kernel $\KSigB$ on
(a sufficiently regular subset of) paths
$\PathsX\subset[[0,1]\rightarrow \calX]$.
We thus call this procedure that transforms a static kernel $\kernel$ on
$\calX$ into a kernel $\KSigB$ on sequences in $\calX$, the so-called {\bf sequentialization} (of the kernel $\kernel$).
\begin{Def}
Fix $\IndLvl\in\NN$ and $\kernel:\calX\times\calX\rightarrow\RR$,
$\kernel(\cdot,\cdot)=\langle \phi,\phi \rangle_{\calH}$. We
define\footnote{For $\phi:\calX\rightarrow\calH$ and
  $\sigma\in[[0,1]\rightarrow \calX]$ we denote with $\phi(\sigma)\in[[0,1]\rightarrow\calH]$
  the path $t\mapsto \phi(\sigma(t))$.}
\begin{align*}
\KSigB:& \PathsX \times \PathsX\rightarrow \RR\;,\quad \left(
         \sigma,\tau \right)\mapsto\langle \Sig\left( \phi
         (\sigma)\right), \Sig \left( \phi(\tau) \right)\rangle_{\tensalg}\\
\KSigBk :& \PathsX\times \PathsX\rightarrow \RR\;,\quad
          \left( \sigma,\tau \right)\mapsto\langle \Sigk\left( \phi(\sigma) \right),
          \Sigk\left(  \phi (\tau) \right)\rangle_{\tensalg}.
\end{align*}
We refer to $\KSigB$  as the \emph{sequentialization} of
the kernel $\kernel$ and to $\KSigBk$ as the \emph{sequentialization} of
the kernel $\kernel$ truncated at level $\IndLvl$.
\end{Def}
As we have done for $\KSigA$ and $\KSigAk$, we again need to verify points~\ref{welldf},~\ref{eff},~\ref{robust},~\ref{noise} for $\KSigB$ and $\KSigBk$.
Point \ref{welldf} follows under appropriate regularity assumptions on
$\PathsX$ immediately from the corresponding
statement for $\KSigA$ and $\KSigAk$. For point \ref{eff} note, that although the data enters $\KSigA$ and
$\KSigAk$ in the recursion formula only in the form of scalar
products of differentials in $\calH$, it is mathematically somewhat
subtle to replace these by evaluations of $\kernel$
over an arbitrary set $\calX$, due to the differential
operators involved. However, we will do so by identifying the kernel $\kernel$ with a signed measure on $\left[ 0,1 \right]^2$.

Points \ref{welldf} and \ref{eff} will be discussed in this section, point
\ref{robust} will be the main topic of Section \ref{sec:approx}, and
point \ref{noise} is discussed in Section \ref{sec:noise}.

\subsection{Well-definedness of the signature kernel}\label{sec:welldf}
For \ref{welldf} well-definedness, note: $\KSigA$ and $\KSigAk$ are positive definite kernels, since explicitly defined as a scalar product of features. Also, these scalar products are always finite (thus well-defined) for paths of bounded variation, as the following Lemma~\ref{Lem:finker} shows.

\begin{Lem}
\label{Lem:finker}
Let $x,y\in \BV(\calH)$. Then it holds that
\begin{enumerate}[label=(\roman*)]
\item $\|\KSigAk(x,y)\|\le \frac{1}{\IndLvl!^2}V(x)^{\IndLvl}V(y)^{\IndLvl}< \infty$
\item $\|\KSigA(x,y)\| \le \exp(V(x)+V(y))< \infty$
\end{enumerate}
\end{Lem}
\begin{proof}
This follows from Lemma~\ref{Lem:finsig} and the Cauchy--Schwarz-inequality.
\end{proof}
Hence, $\KSigA,\KSigAk$ are well-defined (positive definite) kernels.
\subsection{Kernel trick number one: kernelizing the signature}\label{Sig_kernel_trick_1}
The kernel trick consists of defining a kernel which is \ref{eff} efficiently computable. For this, we show that $\KSigA,\KSigAk$ can be entirely expressed in terms of $\calH$-scalar products:

\begin{Prop}
\label{Prop:Ksigscpr}
Let $x,y\in \BV(\calH)$.
Then:
\begin{enumerate}[label=(\roman*)]
\item $\KSigA(x,y) = \sum_{m=0}^\infty \int_{\left( \bs,\bt\right)\in
    \simplex^m\times \simplex^m} \prod_{i=1}^m
             \left\langle \diff x (\bs[i]), \diff y (\bt[i])
             \right\rangle_{\calH}$,
\item $\KSigAk(x,y) = \sum_{m=0}^\IndLvl\int_{\left( \bs,\bt\right)\in
    \simplex^m\times \simplex^m} \prod_{i=1}^m \left\langle \diff x(\bs [i]), \diff y (\bt[i])\right\rangle_{\calH}$.
\end{enumerate}
\end{Prop}
\begin{proof}
The first formula follows from substituting definitions for $\KSigAk,
\KSigA$ and using the linearity of the integrals (recall that we use the
convention $\prod_{i=1}^0\dots=1$).
\end{proof}

Furthermore, there are the following Horner-scheme-type recursions:

\begin{Prop}
\label{Prop:Recsig}
Let $x,y\in \BV(\calH)$.
Then
\begin{enumerate}[label=(\roman*)]
\item $\KSigA(x,y) = 1+\int_{\left( s_1,t_1
              \right)\in\left( 0,1 \right)\times \left( 0,1 \right)} \left( 1+ \int_{\left(
              s_2,t_2 \right)\in\left( 0,s_1 \right)\times \left(
              0,t_1 \right)} \left(  1+ \cdots \right)\cdots \langle \diff x\left( s_2
              \right),\diff y\left( t_2 \right)  \rangle_{\calH}\right)\langle \diff x\left( s_1
              \right),\diff y\left( t_1 \right)  \rangle_{\calH}$,
\item $\KSigAk(x,y) = 1+\int_{\left( s_1,t_1
              \right)\in\left( 0,1 \right)\times \left( 0,1 \right)} \left( 1+ \dots \int_{\left(
              s_{\IndLvl},t_{\IndLvl} \right)\in\left( 0,s_{\IndLvl-1} \right)\times \left(
              0,t_{\IndLvl}-1 \right)} \langle \diff x\left( s_{\IndLvl}
              \right),\diff y\left( t_{\IndLvl} \right)  \rangle_{\calH} \dots \right) \langle \diff x\left( s_1
              \right),\diff y\left( t_1 \right)  \rangle_{\calH}$.
\end{enumerate}
\end{Prop}
\begin{proof}
This follows from Proposition~\ref{Prop:Ksigscpr} and an elementary computation.
\end{proof}

The recursion is the mathematically precise form of the iterated
expectation from Section~\ref{sec:intro.C}.
In Section \ref{sec:approx} we show that this recursion is preserved
under discretization.
\subsection{Kernel trick number two: sequentialization of kernels}\label{Sig_kernel_trick_2}
As shown in the previous Section~\ref{Sig_kernel_trick_1}, the signature kernel on bounded variation paths over $\calH$ can be entirely expressed in scalar products over $\calH$.
However, in general, the data will not be directly observed as sequences in a (potentially infinite dimensional) Hilbert space $\calH$, but in some arbitrary $\calX$. Sequences in $\calX$ can be treated with a \emph{second} kernelization step:
We choose a primary feature map $\phi: \calX\rightarrow \calH$ for
(non-sequential) data in $\calX$.
By Proposition~\ref{Prop:Ksigscpr}, the sequential kernel can be expressed in scalar products in $\calH$, therefore, after the second kernelization step, in the primary kernel.

We make this point precise.
\begin{Ass}
\label{Ass:primaryk}
For the general situation of sequences in an arbitrary set $\calX$, we may assume:
\begin{enumerate}[label=(\roman*)]
\item The potential observations, denoted $\PathsX$, are sequential data of type $[[0,1] \rightarrow \calX]$.
\item The potential observations $\PathsX$ are absolutely continuous
  paths when corresponding
  sequences of primary features are considered.
That is, for every potential observation $\sigma\in \PathsX$, the concatenation
$\phi\circ \sigma$ is an element of $\BV (\calH)$ that is absolutely
continuous\footnote{Absolutely continuous paths are dense in variation
  norm in $\BV$.
 We exclude non-absolutely continuous paths (like Cantor functions) to
 avoid technicalities.
In Section \ref{sec:noise} we show that above can generalized to much ``rougher
paths'' than bounded variation.}.
\item Scalar products of primary features can be efficiently computed via a kernel function
$$\kernel:\calX\times \calX \rightarrow \RR,\quad \kernel(x,y) = \langle \phi(x),\phi(y)\rangle_\calH.$$
More precisely, $\calH$ is the RKHS associated to the kernel $\kernel$.
\end{enumerate}
\end{Ass}

Under Assumption \ref{Ass:primaryk}, the formulas given in Proposition \ref{Prop:Ksigscpr} and Proposition \ref{Prop:Recsig}
suggest to substitute $x=\phi\left( \sigma \right), y = \phi
\left(\tau \right)$ and to replace $\left\langle\diff x(\boldsymbol{s}[i]), \diff y (\boldsymbol{t}[i])
 \right\rangle_{\calH}$ with an evaluation of $\kernel$.
However, the differential is between the scalar product and the $\phi$
and hence a naive replacement of the scalar product by the primary
kernel $\kernel$ is not possible. Below we show that a slightly more technical form of the iterated expectation formula in the introductory Section~\ref{sec:intro.B} holds and agrees with this formula for differentiable paths:

\begin{Prop}\label{prop:kerneltricktwo}
Let Assumptions~\ref{Ass:primaryk} be satisfied.
Then $\KSigB$ and
$\KSigBk$ are positive definite kernels on $\PathsX$ and for $\sigma,\tau\in\PathsX$:
\begin{enumerate}[label=(\roman*)]
\item $\KSigB\left(
    \sigma,\tau\right)=\sum_{m=0}^{\infty}\int_{\left(\boldsymbol{s},\boldsymbol{t}\right)\in\Delta_{m}\times\Delta_{m}}\diff\Kmeasure_{\sigma,\tau}\left(\boldsymbol{s}\left[1\right],\boldsymbol{t}\left[1\right]\right)\cdot\ldots
  \cdot
  \diff\Kmeasure_{\sigma,\tau}\left(\boldsymbol{s}\left[m\right],\boldsymbol{t}\left[m\right]\right)$,
\item
$\KSigBk\left( \sigma,\tau\right)=\sum_{m=0}^{\IndLvl}\int_{\left(\boldsymbol{s},\boldsymbol{t}\right)\in\Delta_{m}\times\Delta_{m}}\diff\Kmeasure_{\sigma,\tau}\left(\boldsymbol{s}\left[1\right],\boldsymbol{t}\left[1\right]\right)\cdot\ldots \cdot \diff\Kmeasure_{\sigma,\tau}\left(\boldsymbol{s}\left[m\right],\boldsymbol{t}\left[m\right]\right)$.
\end{enumerate}
for a suitably chosen signed measure $\Kmeasure_{\sigma,\tau}$ on
$\left[ 0,1 \right]^2$. If $\calX$ is an $\RR$-vector space and
$\sigma,\tau$ are differentiable, then
\begin{align*}
  \diff\Kmeasure_{\sigma,\tau}\left(s,t\right) =
\kernel \left(\dot{\sigma}(s), \dot{\tau}\left( t\right)\right)
\diff s \diff t.
\end{align*}
\end{Prop}
\begin{proof}
The proof is carried out in Appendix~\ref{Apx:Kern2}.
\end{proof}
As before, above can be written as Horner-type recursions
\begin{Prop}
  Let Assumptions \ref{Ass:primaryk} be satisfied and $\sigma,\tau\in
  \PathsX$. Then:
  \begin{enumerate}[label=(\roman*)]
  \item $\KSigB\left( \sigma,\tau \right) = 1+\int_{\left( s_1,t_1
              \right)\in\left( 0,1 \right)\times \left( 0,1 \right)} \left( 1+ \int_{\left(
              s_2,t_2 \right)\in\left( 0,s_1 \right)\times \left(
              0,t_1 \right)} \left(  1+ \cdots \right)\cdots \diff \Kmeasure_{\sigma,\tau}(s_2,t_2)\right)
            \diff\Kmeasure_{\sigma,\tau}(s_1,t_1)$
\item $\KSigBk\left( \sigma,\tau \right) = 1+\int_{\left( s_1,t_1
              \right)\in\left( 0,1 \right)\times \left( 0,1 \right)} \left( 1+ \dots \int_{\left(
              s_{\IndLvl},t_{\IndLvl} \right)\in\left( 0,s_{\IndLvl-1} \right)\times \left(
              0,t_{\IndLvl}-1 \right)} \diff\Kmeasure_{\sigma,\tau}\left( s_1,t_2 \right) \dots
        \right) \diff\Kmeasure_{\sigma,\tau}\left( s_1,t_2 \right)$
  \end{enumerate}
\end{Prop}
\begin{proof}
  This follows from Proposition \ref{prop:kerneltricktwo} and an elementary calculation.
\end{proof}

\section{Discrete signatures and kernels for sequences}
\label{sec:approx}
As it was mentioned under point \ref{robust} in
Section~\ref{sec:kernelized_signatures}, there is one major practical
issue with the kernels introduced in Section
\ref{sec:kernelized_signatures}: continuous sequences are in reality never fully observed, since observations in practice are always finite, for example given as time points at which a time series is sampled.
To be more precise, instead of having full knowledge of
$\sigma\in[[0,1]\rightarrow\calX]$ we can use at most a finite number of samples, $\sigma(\bs)=\left(\sigma(s_1),\dots,\sigma(s_n) \right)\in \calH^+$ for $\bs\in\simplex$.
We address this by providing a discretization of all prior
concepts: a discrete version $\SSig:\calH^+\rightarrow\tensalg$ of the signature $\Sig$, a
discrete version $\KSeqA:\calH^+\times\calH^+\rightarrow\RR$ of the signature kernel $\KSigA$, and a discrete
version $\KSeqB:\calX^+\times\calX^+$ of the sequnentialization $\KSigB$ of $\kernel$.

The discretization is based on an elementary integral-vs-sum approximation argument, in which
$$\mbox{the iterated integral } \int_{\simplex^m ([0,1])}\diff
x^{\otimes m} \mbox{ is approximated by a sum }
\sum_{\substack{\bi\sqsubset[\ell(\seqdf \bx)] \\ \ell(\bi)=m}} \seqdf
\bx[i_1]\otimes \dots \otimes \seqdf \bx[i_m],$$
where $\bx= x(\bt)$ is a suitable discretization of the bounded
variation path $x\in\BV([0,1],\calH)$ with $\bt\in\simplex([0,1])$. Both the integral and the sum live in the tensor power algebra, an we will show that the sum approximates the integral in a sampling sense. Similarly, one can define a signature for discrete sequences that approximates the continuous signature. The beginning of this section showcases qualitative statements of this approximative kind.

In analogy to Section \ref{sec:kernelized_signatures} we proceed in two steps:
in the first step, ``\emph{kernel trick one}'', we use the
discretizted signature $\SSig$ to define a kernel $\KSeqA$ for
sequences in $\calH^+$ (of possibly different length). It can be
expressed as a polynomial in scalar products in $\calH$, as
follows\footnote{Recall our convention $\prod_{r=1}^0f(r)=1$ and that
  $\bi\sqsubset[L]$ includes the empty tuple.}:
\begin{align}
  \label{eq:kseqa_first}
\KSeqA(\bx,\by) = \sum_{\substack{\bi\sqsubset [\ell (\seqdf \bx)] \\
  \bj \sqsubset [\ell(\seqdf \by)]\\ \ell(\bi) = \ell(\bi)}} \prod_{r=1}^{\ell(\bi)}\langle \seqdf \bx [i_r], \seqdf \by [j_r]\rangle_\calH
\end{align}
and also a Horner-type formula is derived (see Proposition~\ref{Prop:sigrecursion} below).
In a second step, ``\emph{kernel trick two}'', we replace the inner
product of finite differences in (\ref{eq:kseqa_first}) by a linear combination of evaluations
of the kernel $\kernel$. This gives a kernel $\KSeqB$ defined on
sequence in $\calX$.
Primary kernels $\kernel$ on arbitrary sets $\calX$ can thus be ``sequentialized'' to kernels on sequences in such sets.

Besides the Horner type formula, the main theoretical result of this section is
that both kernels are robust in the
sense that $\KSeqA(x(\bs),y(\bt))$ and $\KSeqB(\sigma(\bs),\tau(\bt))$ converge to $\KSigA(x,y)$ and $\KSigB(\sigma,\tau)$
as the mesh of $\bs$ and $\bt$ vanishes.

In Section \ref{sec:discr.string} we will also see that the well-known
string kernels can be derived as such a special case of
sequentialization, for sequences of symbols;
in section \ref{sec:noise} we show that above kernels are also robust
under noise, that is when the bounded variation assumption does not
hold;
in Section~\ref{sec:comp}, we will further show that despite an exponential number of terms in the sum-product expansion above, the sequential kernels can be computed efficiently.


\subsection{Discretizing the signature}
\label{sec:approx.disc}
To obtain a kernel $\KSeqA$ for discrete sequences in $\calH^+$ that mimicks
the signature kernel $\KSigA$, we pass from the Riemann--Stieltjes
integral to a discretized, finite-sum approximation.
The central mathematical observation is that
\begin{align*}
\bigotimes_{j=0}^{\ell(\seqdf\bx)} (1+ \seqdf \bx[j])  =
\sum_{\bi\sqsubset [\ell(\seqdf \bx)]} \seqdf \bx[i_1]\otimes \dots \otimes \seqdf \bx[i_{\ell(\bi)}] \approx
\sum_{m=0}^\infty \int_{\simplex^m ([0,1])}\diff x^{\otimes m},
\end{align*}
that is, the sum approximating the signature can be written as an outer product reminiscent of an exponential approximation.

The mathematical statement underlying our approximation generalizes Euler's famous theorem on the exponential function.
Euler's original theorem and its proof are below, in a form that is slightly more quantitative than how it is usually presented, while being closer to our later approximation result:

\begin{Thm}
\label{Thm:Euler}
Let $x\in\RR, n\in \NN$. Then,
$$\left(1+\frac{x}{n}\right)^n - \exp(x) = g(x,n),\quad\mbox{where}\;\|g(x,n)\| \le \frac{\exp(x)}{n}\left(1+\frac{x^n}{(n-2)!}\right).$$
In particular, it holds that
\begin{align*}
\lim_{n\rightarrow \infty}\left(1+\frac{x}{n}\right)^n = \exp(x),
\end{align*}
where convergence is uniform of order $O(n^{-1})$ on any compact subset of $\RR$.
\end{Thm}
\begin{proof}
All statements follow from the first, which we proceed to prove. By the binomial theorem, it holds that
$$\left(1+\frac{x}{n}\right)^n = \sum_{k=0}^n {n\choose k} \cdot \frac{x^k}{n^k}.$$
From the definition of the binomial coefficient and an elementary computation, one obtains
$${n\choose k}\cdot \frac{x^k}{n^k} = \frac{x^k}{k!} + g(x,n,k),\;\mbox{where}\; \|g(x,n,k)\| \le \frac{x^k}{k!n},$$
for $k\le n$. For $k\ge n$, one has
$$\frac{x^k}{k!}\le \frac{x^n}{n!}\cdot \frac{x^{k-n}}{(k-n)!}.$$
Putting together all inequalities and using the Taylor expansion of $\exp$ yields the claim.
\end{proof}

Approximation bounds for a discretized version of the signature approximating the continuous one will be given by a generalization of Euler's Theorem~\ref{Thm:Euler}. We introduce a discretized variant of the signature for sequences, as follows:

\begin{Def}
We call the map $\SSig:\calH^+\rightarrow\tensalg$ defined as
$\SSig(\bx)=\prod_{i=1}^{\ell(\seqdf \bx)}\left(1+ \seqdf\bx[i] \right)$ the (approximate) signature map, for discrete sequences.
\end{Def}

We are ready to prove the main theorem for discrete approximation of
the signature. The gist of Theorem~\ref{Thm:discretmesh} is that the
signature of sub-sequences of $x$ approximates the actual signature of
$x$, in a limit similar to Euler's
Theorem~\ref{Thm:Euler}.

\begin{Thm}
\label{Thm:discretmesh}
Let $x\in\BV(\calH)$ and $\bt\in \simplex^{M+1}([0,1])$.
Then,
\begin{align*}
\left\|\SSig(x(\bt)) - \Sig(x)\right\|_{\tensalg}&\le \Exp(\Var(x)) - \Exp_1 (\Var_{\bt} (x)),
\end{align*}
where $\Exp (\br) := \prod_{i=1}^{\ell(\br)}(1+\br[j] )$ for $\br\in
\RR^+$, and $\Var_{\bt}(x):= \left(\Var
  (x[t_1,t_2]),\dots,\Var (x[t_M,t_{M+1}])\right)\in \RR^M$ (as usual,
$x[a,b]$ denotes the restriction of $x$ to $\BV ([a,b],\calH)$.).
Further, it holds that:

\begin{enumerate}[label=(\roman*)]
\item The right hand sides are always positive and can be bounded as follows:
\begin{align*}
0\le & \exp(\Var(x)) - \Exp_1 (\Var_{\bt} (x)) \le \Var(x)\exp(\Var(x))\cdot \|\Var_{\bt}(x)\|_0,
\end{align*}
where as usual we denote $\|\Var_{\bt}(x)\|_0 = \max_i \Var \left( x[t_i,t_{i+1}]\right)$.

\item One has convergence on the right side
\begin{align*}
\lim_{\substack{\Var(x(\bt))\rightarrow \Var(x)}} \exp(\Var_{\bt}(x)) = \exp (\Var(x))
\end{align*}
(and similar for $\ExpM$), uniform of order $O\left(\|\Var_{\bt}(x)\|_0 \right)$ on any compact subset of $\BV(\calH)$.

\item In particular, one has convergence of the discretized signature to the continuous signature
\begin{align*}
\lim_{\substack{\Var(x(\bt))\rightarrow \Var(x)}} \SSig(x(\bt)) = \Sig(x),
\end{align*}
(hence also for $\Sigk$), uniform of order $O\left(\|\Var_{\bt}(x)\|_0 \right)$ on any compact subset of $\BV(\intvl)$.

\item If $\bt$ is chosen such that $\Var (x[t_i, t_{i+1}]) = \frac{\Var(x)}{M}$ for all $i$, then one has the asymptotically tight bound
$$\left\|\SSig(x(\bt)) - \Sig(x)\right\|_{\tensalg} \le \frac{\exp\left( \Var (x) \right)}{M}\left(1+\frac{(\Var x)^M}{(M-2)!}\right).$$
\end{enumerate}

\end{Thm}
\begin{proof}
The proof is carried out in Appendix~\ref{Apx:Euler}:
The main statement follows from Theorem~\ref{Thm:Eulersimplex}~(ii);
points (i) and (ii) follow from applying
Proposition~\ref{Prop:Eulerv2} to $\Var (x) = \sum_{i=1}^M \Var
(x[t_i, t_{i+1}])$; point (iii) follows from (ii);
(iv) follows from Theorem~\ref{Thm:Eulersimplex}~(iii).
\end{proof}

Euler's original Theorem~\ref{Thm:Euler} is recovered from for substituting $x:I\rightarrow \RR, t\mapsto x'\cdot t$ in Theorem~\ref{Thm:discretmesh}~(iii), where $x'$ is the $x$ of Theorem~\ref{Thm:Euler}. Similar statements for the truncated signature may be derived, where the bounds are truncated versions of the ones given; for an exact mathematical formulation, see Theorem~\ref{Thm:Eulersimplex}~(i).

\begin{Rem}[About geometric approximations]\label{rem:not_group_like}
In general, there is no bounded
variation path $x^{\prime}\in \BV(\calH)$ such that $\Sig(x^{\prime}) =
\SSig(x(\bt))$, even though both $\SSig(x(\bt))$ and $\Sig (x)$ are
elements of the tensor algebra $\tensalg$. Thus the approximation
characterized in Theorem~\ref{Thm:discretmesh} is an algebraic, non-geometric approximation on
the level of signatures (= algebraic objects); since, in general, it does not arise from an
approximation/discretization on the level of paths (= geometric objects).

The latter is the common type of approximation for (rough) paths, see \cite{friz2010multidimensional}.
The reader familiar with such approximations on the level of paths is
invited to follow the exposition in parallel with any path-level
approximation in mind. The theoretical considerations can be carried
out in analogy for a while, but to our knowledge do not lead to a
kernel which is as efficiently computable (as opposed to the efficient
Algorithm~\ref{alg:sigpw} discussed in Section \ref{sec:comp}); see also Remark \ref{rem:inner_product_pcw_linear}
about approximations of inner products of signature features of piecewise linear paths.
\end{Rem}

\subsection{Kernel trick number one discretized}

In the discrete setting, we have access only to the discrete signature $\SSig(x(\bt))$ of a path $x$, at a finite sequence of points $\bt\in\simplex$. Theorem~\ref{Thm:discretmesh} guarantees that $\SSig (x(\bt))$ is a good approximation of the continuous signature $\Sig (x)$ when $\bt$ is densely sampled (as $\Var (x(\bt))$ approaches $\Var(x)$).

Both signatures live in the tensor algebra, it is therefore natural to obtain discretized variant of signature kernels as the scalar product of discretized signatures; again, we also define a truncated version.

\begin{Def}
\label{Def:seqk}
The discretized signature kernel $\KSeqA$, for sequences in $\calH$, and its truncation $\KSeqAk$, are defined as
\begin{align*}
\KSeqA:& \calH^+\times \calH^+\rightarrow \RR\;,\quad (\bx,\by)\mapsto\langle \SSig(\bx), \SSig(\by)\rangle_{\tensalg}\\
\KSeqAk:& \calH^+\times \calH^+\rightarrow \RR\;,\quad (\bx,\by)\mapsto\langle \SSigk(\bx), \SSigk(\by)\rangle_{\tensalg}.
\end{align*}
\end{Def}

Both $\KSeqA$ and $\KSeqAk$ are a positive (semi-)definite kernels, since explicitly defined as a scalar product of features.

Through the approximation Theorem~\ref{Thm:discretmesh}, the sequential kernel naturally applies to a ground truth of bounded variation paths $x,y\in\BV([0,1],\calH)$ which is sampled in a finite number of consecutive points $\bs,\bt\in\simplex([0,1])$, by considering $\KSeqA (x(\bs),y(\bt))$ (note that $\bs$ and $\bt$ being of different lengths will create no issue in-principle). We obtain a corollary of Theorem~\ref{Thm:discretmesh} for approximating the kernel function in this way:

\begin{Cor}
Let $x,y\in\BV([0,1],\calH)$, let $\bs,\bt\in \simplex ([0,1])$. Then,
$$\left\|\KSeqA(x(\bs),y(\bt)) - \KSigA(x,y)\right\| \le 4 \exp(\Var(x)+\Var(y)) - 2\exp(\Var(y))\exp (\Var_{\bs} (x)) -2\exp(\Var(x))\exp (\Var_{\bt} (y)).$$
In particular, it holds that
\begin{align*}
\lim_{\substack{\Var(x(\bs))\rightarrow \Var(x)\\\Var(y(\bt))\rightarrow \Var(y)}} \KSeqA(x(\bs),y(\bt)) = \KSigA(x,y),
\end{align*}
where convergence is uniform of order $O(\|\Var_{\bs}(x)\|_0+\|\Var_{\bt}(y)\|_0 )$ on any compact subset of $\BV([0,1],\calH)\times \BV([0,1],\calH)$.
\end{Cor}
\begin{proof}
It holds that
\begin{align*}
&\KSeqA(x(\bs),y(\bt)) - \KSigA(x,y)\\
= &\langle \SSig (x(\bs)), \SSig(y(\bt)) \rangle_{\tensalg} - \langle \Sig(x), \Sig(y) \rangle_{\tensalg}\\
= & \langle \SSig (x(\bs)), \SSig(y(\bt)) - \Sig(y) \rangle_{\tensalg} + \langle \SSig(x(\bs)) - \Sig(x) , \Sig(y) \rangle_{\tensalg}.
\end{align*}
The Cauchy-Schwarz-inequality implies that
$$\|\langle \SSig(x(\bs)), \SSig(y(\bt)) - \Sig(y) \rangle_{\tensalg}\|_{\tensalg} \le \|\SSig(x(\bs))\|\cdot \|\SSig(y(\bt)) - \Sig(y)\|$$
Theorem~\ref{Thm:discretmesh}, together with Lemma~\ref{Lem:finsig}~(i), implies that
$$\|\SSig(x(\bs))\|\cdot \|\SSig(y(\bt)) - \Sig(y)\|\le 2 \exp(\Var(x))\cdot \left(\exp(\Var(y)) - \exp (\Var_{\bt} (y))\right)$$
Similarly, one obtains
$$\|\langle \SSig(y(\bt)), \SSig(x(\bs)) - \Sig(x) \rangle_{\tensalg}\|_{\tensalg} \le 2 \exp(\Var(y))\cdot \left(\exp(\Var(x)) - \exp (\Var_{\bs} (x))\right).$$
Putting all (in-)equalities together yields the main claim, the convergence statement follows from Theorem~\ref{Thm:discretmesh}~(ii).
\end{proof}

Note that the sampling points $\bs,\bt$ are crucial for the convergence statement, even though the definition of $\KSigA$ itself depends only on $x,y$. A similar statement may be obtained for $\KSeqAk$ in analogy.

We proceed by giving explicit formulae for $\KSeqA$ that will become crucial for its efficient computation.

\begin{Prop}
\label{Prop:sigrecursion}
Let $\bx,\by\in \calH^+$. The following identities hold for the
discretized signature kernels:
\begin{enumerate}[label=(\roman*)]
\item\label{item:kseqA} $\KSeqA(\bx,\by) = \sum_{\substack{\bx^{\prime}\sqsubset \seqdf \bx, \by^{\prime}
      \sqsubset \seqdf \by\\ \ell(\bx^\prime) = \ell(\by^\prime)}}
  \prod_{i=1}^{\ell(\bx^\prime)}\langle \bx^\prime[i], \by^\prime[i]\rangle_{\calH}$,
\item\label{item:kseqAk} $\KSeqAk(\bx,\by) = \sum_{\substack{\bx^\prime\sqsubset \seqdf \bx, \by^\prime
      \sqsubset \seqdf \by\\ \ell(\bx^\prime) = \ell(\by^\prime)\le \IndLvl}}
  \prod_{i=1}^{\ell(\bx^\prime)}\langle \bx^\prime[i], \by^\prime[i]\rangle_{\calH}$
\item\label{item:kseqA_horner} $\KSeqA(\bx,\by) = 1 + \sum_{\substack{i_1\ge 1 \\ j_1 \ge 1}}
  \langle \seqdf \bx[i_1], \seqdf \by[j_1]\rangle\left(1+
    \sum_{\substack{i_2 \gneq i_1 \\ j_2 \gneq j_1}} \langle \seqdf
    \bx[i_2], \seqdf \by[j_2]\rangle\left( 1 + \sum_{\substack{i_3 \gneq
          i_{2} \\ j_3 \gneq j_{2}}} \langle \seqdf \bx[i_3], \seqdf
      \by[j_3]\rangle\left(1 +  \sum_{\dots} \dots \right)
    \right)\right)$
\item\label{item:kseqAk_horner} $\KSeqAk(\bx,\by) = 1+ \sum_{\substack{i_1\ge 1 \\ j_1 \ge 1}}\langle \seqdf \bx[i_1], \seqdf \by[i_1]\rangle \left(1+\sum_{\substack{i_2 \gneq i_1 \\ j_2 \gneq j_1}}\langle \seqdf \bx[i_2], \seqdf \by[i_2]\rangle \left(1+ \dots \sum_{\substack{i_{\IndLvl} \gneq i_{\IndLvl-1} \\ j_{\IndLvl} \gneq j_{\IndLvl-1}}} \langle \seqdf \bx[i_{\IndLvl}],\seqdf \by[i_{\IndLvl}]\rangle \right)\right)$
\end{enumerate}
where the usual convention that a sum running over an empty index set evaluates to zero applies.
\end{Prop}
\begin{proof}
This follows directly from an explicit computation where both sides are expanded and compared.
\end{proof}
\begin{Rem}
All summation in Proposition~\ref{Prop:sigrecursion} is finite, even for $\KSeqA$: there are only a finite number of sub-sequences $\bx^\prime,\by^\prime$ in \ref{item:kseqA}; and in \ref{item:kseqAk}, summation ends at $\IndLvl = \min (\ell(\bx),\ell(\by))-1$.
An important feature of the sum-formula given by Proposition~\ref{Prop:sigrecursion}~\ref{item:kseqAk_horner} is that it is more efficient to evaluate than the more naive presentation in Proposition~\ref{Prop:sigrecursion}~\ref{item:kseqAk}, due to a much smaller amount of summation.
\end{Rem}

\subsection{Kernel trick number two discretized}

Both equalities in Proposition~\ref{Prop:sigrecursion} show $\KSeqA$ to be expressible entirely in terms of scalar products in $\calH$.
Thus, if we are again in the situation of Assumptions~\ref{Ass:primaryk}, that is there is a primary kernel function $\kernel: \calX\times \calX\rightarrow \RR$ and a feature map $\phi:\calX\rightarrow \calH$ such that $\kernel (\sigma,\tau) = \left\langle \phi(\sigma),\phi(\tau)\right\rangle_\calH$ for $\sigma,\tau\in\calX$, the sequential kernel can be again be second-kernelized
\begin{Def}
\label{Def:SeqB}
Let $\kernel:\calX\times\calX\rightarrow\RR$.
The (discrete)
sequentialization $\KSeqB$ of $\kernel$ for sequences in $\calX$, and its truncation $\KSeqBk$ are defined as
\begin{align*}
\KSeqB:& \calX^+\times \calX^+\rightarrow \RR\;,\quad (\bsigma,\btau)\mapsto\langle \SSig(\phi(\bsigma)), \SSig(\phi(\btau))\rangle_{\tensalg}\\
\KSeqBk:& \calX^+\times \calX^+\rightarrow \RR\;,\quad (\bsigma,\btau)\mapsto\langle \SSigk(\phi(\bsigma)), \SSigk(\phi(\btau))\rangle_{\tensalg}.
\end{align*}
\end{Def}

Following the presentation in the previous section
\ref{sec:approx.disc} and making the substitution $x=\phi(\sigma)$,
$y=\phi(\tau)$ yields again an explict sum formula and a recursive formula for $\KSeqB$
and $\KSeqBk$, in direct analogy to
Proposition~\ref{Prop:sigrecursion}.
This discrete recursion is at the foundation of efficient computation
in a practical setting, as demonstrated in
Section~\ref{sec:discr.string} and Section~\ref{sec:comp}.
\begin{Prop}
  Let $\bsigma,\btau\in \calX^+$ and $\IndLvl\in \NN$. The following
  identities hold for the sequentialization  $\KSeqB$ of $\kernel$:
\begin{enumerate}[label=(\roman*)]
\item $\KSeqB(\bsigma,\btau) = \sum_{\substack{\bi\sqsubset [\ell(\seqdf\bsigma)] \\ \bj\sqsubset [\seqdf\ell(\btau) ]\\ \ell(\bi) = \ell(\bj)}}
  \prod_{r=1}^{\ell(\bi)} \seqdf\kernel(\bsigma,\btau)[\bi[r],\bj[r]]$
\item $\KSeqB(\bsigma,\btau) = 1 + \sum_{\substack{i_1\ge 1 \\ j_1 \ge
      1}} \seqdf \kernel(\bsigma,\btau)[i_1,j_1]\cdot\left(1+
    \sum_{\substack{i_2 \gneq i_1 \\ j_2 \gneq j_1}}  \seqdf
    \kernel(\bsigma,\btau)[i_2,j_2]\cdot\left(1 +  \sum_{\dots} \dots
    \right) \right)$,
\item $\KSeqBk(\bsigma,\btau) = \sum_{\substack{\bi\sqsubset [\ell(\seqdf\bsigma)] \\ \bj\sqsubset [\ell(\seqdf\btau)]\\ \ell(\bi) = \ell(\bj)\le \IndLvl}}
  \prod_{r=1}^{\ell(\bi)} \seqdf \kernel(\bsigma,\btau)[\bi[r],\bj[r]]$,
\item $\KSeqBk(\bsigma,\btau) = 1 + \sum_{\substack{i_1\ge 1 \\ j_1
      \ge 1}} \seqdf \kernel(\bsigma,\btau)[i_1,j_1]\cdot\left(1+
    \sum_{\substack{i_2 \gneq i_1 \\ j_2 \gneq j_1}}  \seqdf
    \kernel(\bsigma,\btau)[i_2,j_2]\cdot\left(1 + \cdots
      \sum_{\substack{i_{\IndLvl} \gneq i_{\IndLvl-1} \\ j_{\IndLvl}
          \gneq j_{\IndLvl-1}}}  \seqdf \kernel(\bsigma,\btau)[i_{\IndLvl},j_{\IndLvl}]\right) \right).$
\end{enumerate}
where we use the notation
\begin{align*}
\seqdf \kernel(\bsigma,\btau)[i,j] := \kernel(\bsigma[i+1], \btau[j+1]) + \kernel(\bsigma[i],\btau[j]) - \kernel(\bsigma[i],\btau[j+1]) - \kernel(\bsigma[i+1],\btau[j]).
\end{align*}
\end{Prop}
\begin{proof}
Substituting $\bx=\phi(\bsigma),\by = \phi(\btau)$ we have for $\bi=(i_1,\ldots,i_r),\bj=(j_1,\ldots,j_r)$
$$\langle (\seqdf \bx)[i_r], (\seqdf \by )[j_r]\rangle_{\calH} = \langle (\seqdf \phi(\bsigma))[i_r], (\seqdf \phi(b\tau))[j_r]\rangle_{\calH}$$
for the scalar products.
One obtains, via an elementary computation, that
$$\langle (\seqdf \bx)[i_r], (\seqdf \by )[j_r]\rangle_{\calH} = \seqdf_{i_r,j_r} \kernel(\bsigma,\btau).$$
The statement then follows by Proposition \ref{Prop:sigrecursion}
\end{proof}
\begin{Rem}
  Again a direct consequence of the approximation
  Theorem~\ref{Thm:discretmesh} is the convergence of $\KSeqB$ to
  $\KSigB$,
  \begin{align*}
    \KSeqB(\sigma(\bs),\tau(\bt))\rightarrow \KSigB(\sigma,\tau)\text{
    as }\mesh(\bs),\mesh(\bt)\text{ vanishes}
  \end{align*}
(under Assumptions \ref{Ass:primaryk}). However, although the rate of
convergence is explicitly given by
Theorem~\ref{Thm:discretmesh}, it depends on $\Var(\phi(\sigma))$ and $\Var(\phi(\tau))$.
\end{Rem}
\begin{Rem}[Variations on a theme]
\label{Rem:seqztn-complicated}
We further point out several, more complex variants of the second kernelization:
\begin{enumerate}[label=(\roman*)]
\item\label{item:variation1} In the formula for the sequential kernel, one could also simply omit the first differences. This corresponds to cumulative summation of the sequences in feature space, or replacing signatures by exponentials:
\begin{align*}
\KSeqB: \calX^+ \times  \calX^+\rightarrow \RR; \quad
    (\bsigma,\btau)\mapsto \sum_{\substack{\bsigma^\prime\sqsubset \bsigma,\; \btau^\prime \sqsubset \btau\\ \ell(\bsigma^\prime)=\ell(\btau^\prime)}}
    \prod_{r=1}^{\ell(\bsigma^\prime)} \kernel(\bsigma^\prime[r],\btau^\prime[r]).
\end{align*}
In practice, we have anecdotally found that this to lead to large sums and numerical instabilities --- though one could argue that the variant without finite differences is the simpler one.

\item More generally, one may consider a family of primary kernels, of form $\kernel: \calX^m\times \calX^m \rightarrow \RR$, not necessarily induced as the product of kernels of type $\kernel: \calX\times \calX \rightarrow \RR$, sequentialization is possible via Proposition~\ref{Prop:sigrecursion}~(i), namely as
\begin{align*}
\KSeqB: \calX^+ \times  \calX^+\rightarrow \RR; \quad (\bsigma,\btau)\mapsto \sum_{\substack{\bi\sqsubset [\ell(\bsigma)],\; \bj \sqsubset [\ell(\btau)]\\ \ell(\bi)=\ell(\bj)=m}} \kernel^m(\bsigma[\bi], \btau[\bj]).
\end{align*}
This corresponds to choosing different kernels on different levels of the tensor algebra, e.g., re-normalization.

\item One may additionally opt to have the primary kernel remember the
  position of elements in the sequence. This may be done by mapping
  sequences in $\calX^+$ to a position-remembering sequence in
  $(\calX\times\NN)^+$ first, and then sequentializing a primary
  kernel which is a product kernel of type
  $\kernel^m((\bsigma,\bi),(\btau,\bj)) = \kappa(\bi,\bj)\cdot
  \kernel^m(\bsigma,\btau)$, where $\kappa:\NN^+\times\NN^+\rightarrow
  \RR$ is positive definite. This yields a sequential kernel of form
  \begin{align*}
    \KSeqB: \calX^+ \times  \calX^+\rightarrow \RR; \quad (\bsigma,\btau)\mapsto \sum_{\substack{\bi\sqsubset [\ell(\bsigma)],\; \bj \sqsubset [\ell(\btau)]\\ \ell(\bi)=\ell(\bj)}} \kappa(\bi,\bj)\cdot \kernel^m(\bsigma [\bi],\btau[\bj]).
  \end{align*}
\end{enumerate}

From the viewpoint of the sequential kernel, the above choices, while minor, seem somewhat arbitrary. We will see in the coming Section~\ref{sec:discr.string} that they have their justification in explaining prior art. Nevertheless we would with Occam's razor that they should not be made unless they empirically improve the goodness of the method at hand, above the mathematically more simple version of the sequential kernel, or sequentialization in Remark~\ref{Rem:seqztn}.

\end{Rem}


\section{The string, alignment and ANOVA kernel seen via sequentializations}
\label{sec:discr.string}
In this section we show how the sequential kernel is closely related to the existing kernels for sequential data:
\begin{enumerate}[label=(\alph*)]
\item String kernels~\cite{lohdi02textclassification,leslie04faststringkernels} may be seen as a special case of the sequential kernel.
\item The global alignment
  kernel~\cite{cuturi2007kernel,cuturi2011fast} can be obtained from a
  special case of the sequential kernel by deleting terms that destroy
  positive definiteness.
\item The ANOVA kernel arises by considering only the symmetric part of
  tensors in $\tensalg$.
\item The sequential kernel may be understood in the general framework of the relation-convolution kernel~\cite{haussler1999convolution}.
\end{enumerate}

The link will be established to one of the more general variants of sequentialization presented in Remark~\ref{Rem:seqztn-complicated}, of form
$$\KSeq: \calX^+ \times  \calX^+\rightarrow \RR; \quad (\bsigma,\btau)\mapsto \sum_{\substack{\bi\sqsubset [\ell(\bsigma)], \; \bj \sqsubset [\ell(\btau)]\\ \ell(\bi)=\ell(\bj)}} \kappa(\bi,\bj)\cdot \prod_{r=1}^{\ell(\bi)} \kernel(\bsigma [i_r],\btau[j_r]),$$
where $\kernel:\calX\times\calX\rightarrow \RR$ and $\kappa:\NN^+\times \NN^+\rightarrow \RR$ are suitably chosen.

\subsection{The string kernel}

The string kernel is a kernel for strings in a fixed alphabet $\Sigma$, say $\Sigma = \{\mbox{A},\mbox{B}\}$. A number of variants exist; we will consider the original definition of the string kernel (Definition 1, page 423 of \cite{lohdi02textclassification}). Modifications may be treated similarly to the below.

\begin{Def}
Fix a finite alphabet $\Sigma$. The string kernel on $\Sigma$ is
$$K_\Sigma:\Sigma^+\times \Sigma^+ \rightarrow \RR\;\quad (\bsigma,\btau) \mapsto  \sum_{\substack{\bi\sqsubset [\ell(\bsigma)] \\ \bj\sqsubset [\ell(\btau)]}} \lambda^{d(\bi)+d(\bj)}\cdot \mathbb{1}(\bsigma[\bi]=\btau[\bj]) = \sum_{u\in \Sigma^+}\sum_{\bi:u=\bsigma[\bi]}\sum_{\bj:u=\btau[\bj]} \lambda^{d(\bi)+d(\bj)},$$
where $\lambda\in\RR^+$ is a parameter, $d(\bi):= \bi[\ell(\bi)] - \bi[1] +1$ is the distance of the last and first symbol in the sub-string given by $\bi$, and $\mathbb{1}(\bsigma[\bi]=\btau[\bj])$ is the indicator function of the event whether $\bsigma[\bi]$ and $\btau[\bj]$ agree, i.e., one if $\bsigma[\bi] = \btau[\bj]$ and zero otherwise.
\end{Def}

For $\lambda = 1$, one has
$$K_\Sigma(\bsigma,\btau) = \sum_{\substack{\bi\sqsubset [\ell(\bsigma)] \\ \bj\sqsubset [\ell(\btau)]}} \mathbb{1}(\bsigma[\bi]=\btau[\bj]) = \sum_{\substack{\bsigma^\prime\sqsubset \bsigma \\ \btau^\prime \sqsubset \btau\\ \ell(\bsigma^\prime)=\ell(\btau^\prime)}} \prod_{r=1}^{\ell(\bsigma^\prime)} \kernel (\bsigma^\prime[r], \btau^\prime[r])$$
for the choice of primary kernel $\kernel(a,b) = \mathbb{1}(a=b)$ for symbols $a,b\in \Sigma$.
Comparing with Proposition~\ref{Prop:sigrecursion}, this canonically
identifies the string kernel with parameter $\lambda = 1$ with the
sequentialization of the Euclidean scalar product $\kernel(\cdot,\cdot)=\langle \cdot,\cdot \rangle$ on $\calX=\RR^{|\Sigma|}$, via identifying a string $\ba\in\Sigma^L$ with the sequence
$$(0,e_{\ba[1]},e_{\ba[1]}+ e_{\ba[2]},\dots,
\sum_{r=1}^{L}e_{\ba[r]})\in \calX^{L+1}.$$
We state the result for arbitrary $\lambda$, where one can also express the string kernel as a sequentialization:
\begin{Prop}
Consider the string kernel $K_\Sigma$ on an alphabet $\Sigma$. Then,
$$K_\Sigma(\bsigma,\btau) = \sum_{\substack{\bi\sqsubset [\ell(\bsigma)] \\ \bj\sqsubset [\ell(\btau)]}}\lambda^{d(\bi)+d(\bj)}\cdot \mathbb{1}(\bsigma[\bi]=\btau[\bj]) =: \sum_{\substack{\bi\sqsubset [\ell(\bsigma)] \\ \bj\sqsubset [\ell(\btau)]\\ \ell(\bi) = \ell(\bj)}} \kappa(\bi,\bj)\cdot \prod_{r=1}^{\ell(\bi)} \kernel (\bsigma[i_r], \btau[j_r]$$
with the positive definite kernels $\kernel(a,b) = \mathbb{1}(a=b)$ and $\kappa (\bi,\bj) = \lambda^{d(\bi)+d(\bj)}$.
\end{Prop}
\begin{proof}
The equality follows from substituting definitions. It remains to show that $\kernel$,$\kappa$ are positive definite.
For $\kappa$, note that we have a scalar product representation $\kappa(\bi,\bj) = \langle \lambda^{d(\bi)}, \lambda^{d(\bj)}\rangle_\RR$ (over the real numbers), therefore $\kappa$ is positive definite. Positive definiteness of $\kernel$ follows similarly from the scalar product representation
$\kernel(a,b) = \sum_{c\in \Sigma^+}\mathbb{1}(a=c)\cdot \mathbb{1}(b=c)$.
\end{proof}

\begin{Rem}
The above shows that string kernel arises a sequentialization as
introduced in Section \ref{sec:approx}.
However, it is interesting to note that it is not in
exact agreement with the continuous signature kernel $\langle
\Sig(x),\Sig(y) \rangle$ when we associate with a string $\ba$ a path
$$x\in\BV([0,L],\RR^{|\Sigma|})\;\quad x(t) = \{t\}\cdot e_{\ba\left[\lfloor t \rfloor\right]}+\sum_{r=1}^{\lfloor t \rfloor} e_{\ba[r]},$$
where as usual $\lfloor t \rfloor$ is the floor function of $t$, and
$\{t\}$ is the fractional part of $t$ (not the set containing $t$ as
one element). In fact, one can show that the string kernel cannot be
expressed as a signature kernel evaluated at a suitable continuous
path.
This difference is due to the fact that our sequentialization kernel/discretization is on the
level of the tensor algebra and not on the level of paths, see Remark \ref{rem:not_group_like}.
\end{Rem}

While $\kappa$ cannot be pulled into the product, the string kernel
nevertheless admits a sum-formula presentation similar to
Proposition~\ref{Prop:sigrecursion}~(iii), namely
$$K_\Sigma(\bsigma,\btau) = \sum_{\substack{i_1\ge 1 \\ j_1 \ge 1}}\lambda^{2-i_1-j_1} \kernel (\bsigma[i_1],\btau[j_1]) \sum_{\substack{i_2 \gneq i_1 \\ j_2 \gneq j_1}} \kernel(\bsigma[i_2],\btau[j_2])\dots \sum_{\substack{i_{L} \gneq i_{L-1} \\ j_L \gneq j_{L-1}}} \lambda^{i_L+j_L} \kernel(\bsigma[i_L],\btau[j_L]),$$
where $L=\min (\ell(\bsigma),\ell(\btau))$. Note that each sum runs over a double index, and $\lambda$-s occur only next to the first and last sum. The usual convention that a sum running over an empty index set evaluates to zero applies. The sum-formula is well-known and a main tool in efficiently evaluating string kernels, see for example the section ``efficient computation of SSK'', page 425 of \cite{lohdi02textclassification}.

Gappy and other string kernel variants such as in~\cite{leslie04faststringkernels} may be obtained from the truncated sequential kernel and other, suitable choices of $\kappa$.

\begin{Rem}
The interpretation as a sequential kernels directly highlights an alternative interpretation, and a natural generalization of the string kernel: suppose each character in the string was not an exact character, but a weighted sum $\alpha \mbox{A} + \beta \mbox{B}$; where for example the character $1\mbox{A} +0 \mbox{B}$ is the same as A, $0\mbox{A} +1 \mbox{B}$ is the same as B, and $\frac{1}{2}\mbox{A} +\frac{1}{2} \mbox{B}$ is the same as half-A-half-B. Such a scenario could practical sense if the exact meaning of a character is not entirely known, for example when the string was obtained through character recognition, for example where the letter l (lower-case-L) and the number 1 are often confounded. It can be observed that this situation can be coped with by mapping say $\frac{1}{2}\mbox{A} +\frac{1}{2} \mbox{B}$ to the vector $\frac{1}{2}e_{\mbox{A}}+\frac{1}{2}e_{\mbox{B}}$, with the sequential kernel left unchanged.
 \end{Rem}

\subsection{The global alignment kernel}

The global alignment kernel one of the most used kernels for sequences. We recapitulate its definition in modern terminology (Section 2.2 of \cite{cuturi2011fast}).

\begin{Def}
A 2-dimensional integer sequence $\bs \in \left(\NN^2\right)^+$ is called an alignment if $\seqdf \bs [i] \in \left\{(0,1),(1,0),(1,1) \right\}$ for all $i\in [\ell(\bs)-1]$. The set of such alignments will be denoted by $\calA$; that is, we write $\bs\in \calA$ if $\bs$ is an alignment.
\end{Def}

\begin{Def}
Fix an arbitrary set $\calX$, and a primary kernel $\kernel:\calX\times \calX\rightarrow \RR$. The global alignment kernel is defined as
$$K_{\mbox{GA}}:\calX^+\times \calX^+ \rightarrow \RR\;\quad (\bsigma,\btau) \mapsto  \sum_{\substack{\bi\sqsubseteq [\ell(\bsigma)] \\ \bj\sqsubseteq [\ell(\btau)]\\ \ell(\bi) = \ell(\bj) \le \ell(\bsigma)+\ell(\btau)\\(\bi,\bj)\in\calA}} \prod_{r=1}^{\ell(I)} \kernel (\bsigma[i_r], \btau[j_r]).$$
\end{Def}

In its native form, the global alignment kernel cannot be written as a sequential kernel. A simple proof for this is that the sequential kernels are all positive definite, while the global alignment kernel need not be (see~\cite{cuturi2011fast}). While one can write
$$K_{\mbox{GA}}:\calX^+\times \calX^+ \rightarrow \RR\;\quad (\bsigma,\btau) \mapsto  \sum_{\substack{\bi\sqsubseteq [\ell(\bsigma)] \\ \bj\sqsubseteq [\ell(\btau)]\\ \ell(\bi) = \ell(\bj)}} \kappa(\bi,\bj)\cdot \prod_{r=1}^{\ell(\bi)} \kernel (\bsigma[i_r], \btau[j_r])$$
with $\kappa(\bi,\bj) = \mathbb{1}\left( (\bi,\bj)\in \calA\right)$, this is not a sequential kernel since $\kappa$ is not positive definite.

However, a simple modification turns the global alignment kernel into a sequential (and thus positive definite) kernel:
\begin{Def}
A 1-dimensional integer sequence $\bsigma \in \left(\NN^1\right)^+$ is called a half-alignment if $\seqdf \bsigma [i] \in \left\{0,1\right\}$ for all $i\in [\ell(\bsigma)-1]$. The set of such half-alignments will be denoted by $\frac{1}{2}\calA$; that is, we write $\bsigma\in \frac{1}{2}\calA$ if $\bsigma$ is a half-alignment.
\end{Def}

With this, we can formulate a slightly modified global alignment kernel:
\begin{align*}
K_{\mbox{G}\frac{1}{2}{A}}:\calX^+\times \calX^+ \rightarrow \RR\;\quad (\bsigma,\btau) \mapsto & \sum_{\substack{\bi\sqsubseteq [\ell(\bsigma)] \\ \bj\sqsubseteq [\ell(\btau)]\\ \ell(\bi) = \ell(\bj) \le \ell(\bsigma)+\ell(\btau) \\ \bi,\bj\in\frac{1}{2}\calA}} \prod_{r=1}^{\ell(\bi)} \kernel (\bsigma[i_r], \btau[j_r])
\end{align*}
With a simple re-formulation, one obtains
\begin{align*}
K_{\mbox{G}\frac{1}{2}{A}} =\sum_{\substack{\bi\sqsubseteq [\ell(\bsigma)] \\ \bj\sqsubseteq [\ell(\btau)]\\ \ell(\bi) = \ell(\bj) \le \ell(\bsigma)+\ell(\btau)}} \kappa^\prime(\bi,\bj) \cdot \prod_{r=1}^{\ell(\bi)} \kernel (\bsigma[i_r], \btau[j_r]),
\end{align*}
where $\kappa^\prime(\bi,\bj) = \mathbb{1}\left(\bi\in \frac{1}{2} \calA\right)\cdot \mathbb{1}\left(\bj\in \frac{1}{2} \calA\right)$. Note that $\kappa'$ is positive definite, since it is explicitly given as a scalar product of features in $\RR$, thus $K_{\mbox{G}\frac{1}{2}{A}}$ is positive definite as a sequential kernel.

The terms missing in $K_{\mbox{GA}}$, when compared to $K_{\mbox{G}\frac{1}{2}{A}}$ are exactly those arising from sequence pairs $(\bi,\bj)\in \left(\NN^2\right)^+$ in which there is an increment $(0,0)$. In view of the discussion in Section~3.2 of~\cite{cuturi2011fast}, these missing terms are exactly the locus of non-transitivity in the sense of~\cite{shin2008generalization}.

One can now follow the authors~\cite{cuturi2011fast} and heuristically continue studying sufficient conditions under which the original global alignment kernel is positive definite, keeping the missing terms out. However, we would argue, especially in the view of the violated transitivity condition, that it may be more natural to add the missing terms back, unless there is a clear empirical reason in favour of leaving them out. Particularly, we would further argue that there is no first-principles reason to leave the terms out, since the definition of an alignment has been heuristic and somewhat arbitrary to start with.

\subsection{The relation-convolution kernel}

Finally, we would like to point out that the sequential kernels are closely related to the relation-convolution kernels in the sense of Haussler~\cite{haussler1999convolution}, Section 2.2. We cite it in a slightly less general form than originally defined:

\begin{Def}
\label{Def:RCk}
Fix arbitrary sets $\calY,\calZ$. Further, fix a relation $R \subseteq \calY\times \calZ$ and a kernel $\kernel_\calY:\calY\times \calY\rightarrow \RR$. A relation-convolution kernel is a kernel of the form
\begin{align*}
K_{\text{RC}}:\calY\times \calY \rightarrow \RR\;\quad (x,y) \mapsto  \sum_{\substack{r\in R^{-1}( x) \\ s\in R^{-1}(y)}} \kernel_\calY(r,s),
\end{align*}
where we have written $R^{-1}(x) :=\{y\;:\; (x,y) \in R\}$.
\end{Def}

In the presented form, the signature kernel
\begin{align*}
\KSeq: \calX^+ \times  \calX^+\rightarrow \RR; \quad (\bsigma,\btau)\mapsto \sum_{\substack{\bi\sqsubset [\ell(\bsigma)] \\ \bj \sqsubset [\ell(\btau)]\\ \ell(\bi)=\ell(\bj)}} \kappa(\bi,\bj)\cdot \prod_{r=1}^{\ell(\bi)} \kernel(\bsigma[i_r] ,\btau[j_r]),
\end{align*}
is a special case of Haussler's relation-convolution kernel, with the relation being ordered sub-sequence relation. Note though the main discrepancy which is that the native sequential kernel is evaluated at first differences, so a relation-convolution kernel is only obtained after the summation step described in Remark~\ref{Rem:seqztn-complicated}~\ref{item:variation1}.

It is also interesting to note that the sub-sequence relation is the same from which the ANOVA kernel is obtained, see Section 2.4 of~\cite{haussler1999convolution}. However, for the ANOVA kernel, the primary kernel is restricted to be zero for unequal sequence elements.

\section{Higher order sequentialization and noisy observations}\label{sec:noise}
We have seen in Section \ref{sec:approx} that the sequentialisation $\KSeqB$ of a kernel $\kernel$ converges to the inner
product of signature features, that is
\begin{align}\label{es:kseq_to_ksig}
\KSeqB(\sigma(\bs),\tau(\bt))\rightarrow \KSigB(\sigma,\tau)=\langle \Sig\left(  \phi(\sigma)\right),\Sig\left( \phi(\tau) \right) \rangle_{\tensalg} \text{ as }\mesh(\bs)\vee \mesh(\bt)\rightarrow0.
\end{align}
This requires (see Assumption \ref{Ass:primaryk}) that
$\phi(\sigma),\phi(\tau)\in\BV(\calH)$ so that the signature features
$\Sig\left(  \phi(\sigma)\right)$,$\Sig\left(  \phi(\tau)\right)$ are
well defined via Riemann--Stieljtes integration.
However, a common situation is that observations are perturbed by
noise in which case the bounded variation assumption is typically not
fulfilled.
An insight of rough path theory and stochastic analysis is that despite the breakdown of classic integration, it is for
large classes of paths possible to define a map that replaces the
signature $\sigma\mapsto\Sig\left( \phi\left( \sigma \right)\right)$ or in our learning context: becomes a
feature map for $\sigma$.

The price is that we need to replace $\SSig$ by a \emph{higher order
approximation} $\SSig_{\IndDegr}$ respectively $\langle \SSig\circ\phi,\SSig\circ\phi \rangle$ by
$\langle \SSig_{\IndDegr}\circ\phi,\SSig_{\IndDegr}\circ\phi \rangle$, here $\IndDegr\ge1$
will denote the order of approximation and has to be choosen higher the
more irregular the underlying paths are.
A thorough discussion for general noise (truly ``rough'' paths) requires some background in
stochastic analysis and rough path theory and is beyond the scope of
this article and we refer to
\cite{lyons2004stflour,friz2010multidimensional}.
The point we want to make in this section, is that with a few modifications, \emph{the methodology of the previous chapter
extends to sequences that are sampled from unbounded variation paths}.
\subsection{Brownian noise: order 2 approximations}
Below we demonstrate the needed adapations for multivariate white noise/Brownian motion; moreover,
we focus on sequentialization of the trivial kernel $\kernel=\langle \cdot,\cdot
 \rangle_{\RR^d}$, that is the primary feature map $\phi(x)=x$.
In this case, we deal with (semi)martingale paths and an approximation of degree $\IndDegr=2$ is needed.
\begin{Prop}\label{prop:no_first_order}
Let $\mu$ be the Wiener measure on
$C\left(\left[0,1\right]\right),\RR^d)$, that is
$x\sim\mu$ is a $d$-dimensional Brownian motion.
Let $\bt_n\in\simplex$ be dyadics, $\bt_n[i]:=i2^{-n}$, and define
$(x_n)\in\BV([0,1],\RR^d)$ as the piecewise linear interpolation on $\bt_n$,
\begin{align}
  \label{eq:pcw_linear_interpolation}
  x_n(t):=x(\bt_n)[i]+(t-\bt_n[i])\seqdf x(\bt_n)[i]\text{ for }t\in\left[ \bt_n[i],\bt_n[i+1] \right).
\end{align}
For any $p>2$ and any multindex
$(i_1,i_2)\in \left\{ 1,\ldots d \right\}^2$ it holds for $\mu$-a.e.~$x$ that
\begin{enumerate}
\item \label{item:unboundedVAR}$\Var(x)=\infty$,
\item\label{item:strat_convergence}
  $\lim_{n\rightarrow\infty}\Sig_{(i_1,i_2)}(x_n)$ exists in $\tensalg$ and
  equals the stochastic Ito--Stratonovich integral
  \begin{align}\label{eq:limit_exists}
    \int_{\simplex_2(\left[ 0,1 \right])}\circ\diff x_{i_1}\circ \diff x_{i_2},
  \end{align}
\item\label{it:order1_doesnt_conv} $\SSig_1(x(\bt_n))_{(i_1,i_2)}$ does not converge to (\ref{eq:limit_exists}) as
  $n\rightarrow\infty$ if $i_1=i_2$.
\end{enumerate}
However, if we denote
\begin{align}
  \label{eq:higherorder}
  \SSig_{(2)}: \calH^+ \rightarrow \tensalg,\quad
  \SSig_{(2)}(\bx)=\prod_{i=1}^{\ell(\seqdf\bx)} \left(  1+\seqdf
  \bx[i]+\frac{(\seqdf \bx[i])^{\otimes 2}}{2!}\right)
\end{align}
then the $(i_1,i_2)$ coordinate of $\SSig_{(2)}\left(  x(\bt_n)\right)$
converges to (\ref{eq:limit_exists}) as $n\rightarrow\infty$ for $\mu$-a.e.~$x$.
\end{Prop}
\begin{proof}
Point (\ref{item:unboundedVAR}) can be found in any textbook on
stochastic analysis, e.g.~\cite{revuz1999continuous}; similarly, point
(\ref{item:strat_convergence}) is classic, e.g.~\cite{friz2010multidimensional}[Chapter 13].
Also point (\ref{it:order1_doesnt_conv}) and the convergence of $\SSig_x(x_n)$ follows by a simple calculation:
\begin{align*}
\SSig_2(x_n)&=\prod_{i=1}^{\ell (\bt_n)-1} \left(1+\seqdf
              x(\bt_n)[i]+\frac{(\seqdf x(\bt_n)[i])^{\otimes
              2}}{2}\right)\\
&=1+\sum_{i\in\left[ \ell(\bt_n) -1\right]}\seqdf x(\bt_n)[i]+\sum_{i,j\in\left[
    \ell(\bt_n)-1\right],i<
    j}\seqdf x(\bt_n)[i]\otimes\seqdf x(\bt_n)[j]+\frac{1}{2}\sum_{i\in\left[\ell(\bt_n)-1\right]}\seqdf x(\bt_n)[i]^{\otimes 2}.
\end{align*}
Now the first sum equals, $x(\bt_n[\ell(\bt_n)])=\int_{\simplex_1(0,1)} \diff x$, the second
equals
\begin{align}
  \sum_{j\in\left[ \ell(\bt_n)-1 \right]}\sum_{i\in\left[ j-1
  \right]}\seqdf \bx[i]\otimes\seqdf \bx[j]=\sum_{j\in\left[
  \ell(\bt_n)-1 \right]}x(\bt_n[i])\otimes  \seqdf x(\bt_n[i])
\end{align}
and a classic convergence results in stochastic calculus,\cite{revuz1999continuous}, shows that
\begin{align}
    \sum_{j\in\left[ \ell(\bt_n)-1 \right]}x(\bt_n[i])\otimes \seqdf x(\bt_n[i])+\frac{1}{2}\sum_{i\in\left[ \ell(\bt_n) -1\right]}
  \seqdf x(\bt_n[i])^{\otimes 2}
\end{align}
converges to the Stratonovich integral $\int \circ \diff x \circ
\diff x$.
Now point (\ref{it:order1_doesnt_conv}) follows by recalling that Brownian motion has non-vanishing quadratic variation,
that is $\sum_{i\in\left[ \ell(\bt_n)-1\right]}\seqdf
x(\bt_n[i]))^{\otimes 2}$ does not converge to $0$ as $n\rightarrow\infty$.
\end{proof}
\begin{Cor}
  Let $\mu$ be the Wiener measure on
  $C\left(\left[0,1\right]\right),\RR^d)$.
  Denote the $\SSig_{(2),2}$ the projection of $\SSig_{(2)}$ to
  $\bigoplus_{m=0}^2(\RR^d)^{\otimes m}$.
  If $x,y\sim \mu$ independently, then with probability one
  \begin{align}
    \langle \SSig_{(2),2}(x(\bs_m)),\SSig_{(2),2}(y(\bt_n))
    \rangle_{\tensalg}\rightarrow \langle
    X,Y\rangle_{\tensalg}\text{ as }m,n\rightarrow\infty
  \end{align}
where $X:=1+ \int_{\simplex_1}\circ \diff x+\int_{\simplex_2}
(\circ\diff x)^{\otimes 2}$ and $Y:=1+\int_{\simplex_1}\circ \diff
y+\int_{\simplex_2}(\circ \diff y)^{\otimes 2}$ are given by Ito--Stratonovich integrals
and $\bs_m$, $\bt_n$ are dyadic partitions, $\bs[i]=i2^{-m}$, $\bt[i]=i2^{-n}$.
\end{Cor}
\begin{Rem}
Essentially the same result holds for other partitions than dyadics
that vanish quickly enough and for continuous semimartingales, that is
bounded variation paths with additive noise of the same path
regularity as Brownian motion (see \cite{friz2010multidimensional}).
Similarly, one can show that $\SSig(x(\bt_n))$ converges for higher
iterated integrals though the calculation gets a bit cumbersome so we
do not address this.
Useful tools to study such convergence questions are the notions
of \emph{multiplicative functionals} and \emph{extension theorem}; we refer to
\cite{lyons1998differential}.
\end{Rem}
\subsection{Higher order approximations}
In the Brownian (and semimartingale) discussed above, the
Riemann--Stieljtes integral was replaced by stochastic
Ito--Stratonovich integral, thus providing a map from (semimartingale)
paths to $\tensalg$.
A general strategy to construct a function that maps a path $x$ to
$\tensalg$ and behaves like iterated integrals is to find a sequence of bounded variation paths
$(x_n)\subset \BV([0,1],\calH)$ such that the $\tensalg$-valued paths
$t\mapsto \sum_{m} \int_{\simplex(\left[ 0,t \right])}(\diff x_n)$
converges in an apropiate sense to a $\tensalg$-valued path denoted
$t\mapsto X(t)$.
This is the notion of a \emph{geometric rough path} and allows to study classes of much ``rougher'' paths;
loosely speaking: \emph{if we deal with a geometric $p$-rough path, then
  we have to consider an approximation $\SSig_{(\IndDegr)}$ of order at least
$\IndDegr=\lfloor p \rfloor$}; thus the rougher the path, the higher
the order of signature approximation.
\begin{Ex}\label{Ex:rough_path_lifts}
As pointed out above, the details of how to map a trajectory to an
element in $\tensalg$ vary, exploit probabilistic structure and we refer \cite{friz2010multidimensional,lyons2004stflour} for details.
Here we just mention that such constructions are well-known for
\begin{itemize}
\item \emph{Brownian motion} (leading to $p$-rough paths for any $p>2$ )
\item more generally, continuous \emph{Semimartingale} (leading to
  $p$-rough paths for any $p>2$),
\item \emph{fractional Brownian motion} of Hurst parameter $H>\frac{1}{4}$,
\item more generally, \emph{Gaussian processes} (leading to $p$-rough paths where $p$
  depends on the regulary of the covariation process),
\item \emph{Markov processes} in continuous time (leading to $p$-rough
  paths with $p$ depending on the generator of
  the Markov process).
\end{itemize}
\end{Ex}
Again, a rigorous treatment requires knowledge of geometric $p$-rough paths
and is beyond the aim of this paper.
However, we still give the general definition of a an order $\IndDegr$
approximation and the associated signature and sequentialized kernel that are needed to treat the paths in Example \ref{Ex:rough_path_lifts}.
\begin{Def}
Let $\IndDegr\in\NN$. Define
\begin{align*}
  \SSigD: &\calH^+ \rightarrow \tensalg,\quad\SSigD(\bx)=
  \prod_{i=1}^{\ell(\seqdf\bx)} \sum_{m=0}^\IndDegr
  \frac{(\seqdf \bx[i])^{\otimes m}}{m!}
  \end{align*}
and denote with $\SSigDk$ the projection of $\SSigD$ to
$\bigoplus_{m=0}^{\IndLvl}\calH^{\otimes m}$.
Let $\kernel:\calX\times\calX\rightarrow\RR$. Define the sequentialization of $\kernel$ of order $\IndDegr$ and
truncated at $\IndLvl$ as
\begin{align*}
  \KSeqDk(\bsigma,\btau):\calX^+ \times \calX^+ \rightarrow\RR,\quad
  \KSeqDk\left( \bsigma,\btau \right)= \langle \SSigDk\left( \phi\left( \bsigma \right) \right) ,\SSigDk\left(  \phi\left( \btau \right)\right)\rangle_{\tensalg}.
\end{align*}
\end{Def}
\begin{Rem}
The sequnentialization $\KSeqBk$ of a kernel $\kernel$ from Section
\ref{sec:approx.disc} arises as special case of above, general
definition: $\KSeqBk=\KSeqBonek$ for $\IndDegr=1$, $\IndLvl\in\NN$.
\end{Rem}
In analogy to the order $\IndDegr=1$ approximations discussed in Section
\ref{sec:approx}, the central mathematical identity is now
\begin{align*}
\prod_{i=1}^{\ell(\seqdf\bx)} \sum_{m=0}^\IndDegr \frac{(\seqdf \bx[i])^{\otimes m}}{m!}=\sum_{\bi \sqsubseteq_D [ \ell(\seqdf\bx)]}
\frac{1}{\bi!} \prod_{r=1}^{\ell(\bi)} \seqdf \bx[i_r] \approx
 \sum_{m=0}^\infty \int_{\simplex^m ([0,1])}\diff x^{\otimes m}
\end{align*}
where $\bx= x(\bt)$ is a suitable discretization of the (unbounded
variation!) path $x\in C([0,1],\calH)$ with $\bt\in\simplex([0,1])$.
Note the appearance of the $\bi!$ term in the first identity together
with the restriction $\bi \sqsubseteq_D [ \ell(\seqdf\bx)]$ which
makes a recursion formula for the inner product more complex.
However, a variation of the recursive Horner type formula still holds for
$\KSeqBDk$ and we give an efficient algorithm for $\KSeqBDk$ for general\footnote{Without
  loss of generality, it is sufficient to
  consider $\IndDegr\le \IndLvl$ since for $\IndDegr>\IndLvl$ they
  only difference occurs at the level of more than
$\IndLvl$ times iterated integrals and we already cut off at degree
$\IndLvl$).} $\IndDegr,\IndLvl\in\NN$
in Section \ref{sec:comp} below.
Since it is a multi-way recursion which is better formulated in terms of a dynamic programming algorithm, we defer its formulation to the following Section~\ref{sec:comp}.
It relies on the formula below for the discretized higher
order signature kernel:
\begin{Prop}
\label{Prop:sighigh}
For $\bx,\by\in \calH^+$, $\IndDegr,\IndLvl\in \NN$,
\begin{enumerate}[label=(\alph*)]
\item $\SSigD(\bx)=\sum_{\bi \sqsubseteq_D [ \ell(\seqdf\bx)]}
  \frac{1}{\bi!} \prod_{r=1}^{\ell(\bi)} \seqdf \bx[i_r]$,
\item $\langle \SSigDk(\bx),\SSigDk(\by) \rangle_{\tensalg}
= \sum_{\substack{\bi,\bj \in \simplex (\NN)\\ \ell(\bi)=\ell(\bj)\le
  \IndLvl\\ \# \bi,\# \bj \le D}} \frac{1}{\bi!\bj!}
  \prod_{r=1}^{\ell(\bi)}\langle \seqdf \bx[i_r], \seqdf \by[j_r]
  \rangle_{\calH}.$
\end{enumerate}
\end{Prop}
\begin{proof}
Writing out the definition of $\SSigD(\bx)$ yields
$$\SSigD(\bx) = \prod_{r=1}^{\ell (\seqdf \bx)} \sum_{d=0}^{\IndDegr}\frac{1}{d!} (\seqdf \bx[r])^{\otimes d}.$$
Application of the (non-commutative) associative law yields
$$\prod_{r=1}^{\ell (\seqdf\bx)} \sum_{d=0}^D \frac{1}{d!} (\seqdf \bx[r])^{\otimes d} = \sum_{\bi \sqsubseteq_D [ \ell(\seqdf\bx)]} \frac{1}{\bi!} \prod_{r=1}^{\ell(\bi)} \seqdf \bx[i_r].$$
The analogous expression for $\SSigD(\by)$ and taking the scalar product while noting
$$\langle \prod_{r=1}^{\ell(\bi)} \seqdf \bx[i_r], \prod_{r^\prime=1}^{\ell(\bj)} \seqdf \by[j_{r'}]\rangle_{\tensalg} = \delta_{\ell(\bi),\ell(\bj)}\cdot \prod_{r=1}^{\ell(\bi)}\langle \seqdf \bx[i_r], \seqdf \by[j_r] \rangle_{\calH}.$$
Truncating at tensor degree $\IndLvl$ yields the claim.
\end{proof}
\section{Efficient computation of sequentialized kernels}
\label{sec:comp}
Naive evaluation of the signature kernels $\KSeqAk$ and sequentialized
kernels $\KSeqBk$ (or more generally $\KSeqBDk$) incurs a cost exponential in the length of the input or the degree. Inspired by dynamic programming, we present in this section a number of algorithms for signature and sequential kernels whose time and memory requirements are polynomial in the length of the sequence.

Table~\ref{complexity_overview} presents an overview on the different
algorithms presented: given $N$ sequences,
$\bsigma_1,\ldots,\bsigma_N\in \calX^+$, each of length less than $L$,
that is $\ell(\bsigma_i)\leq L$ for $i=1,\ldots,N$,
Table~\ref{complexity_overview} shows the computational complexity and
storage requirements for calculating the Gram-matrix
$\KSeqBk\left(\bsigma_i,\bsigma_j\right)_{i,j=1}^N$ of the
sequentialization $\KSeqBk$ of a kernel $\kernel$ on $\calX$.
The simplest variant involving dynamic programming, Algorithm~\ref{alg:sigpw}, allows to evaluate the sequential kernel in a number of elementary computations that is linear in the length of either sequence (thus quadratic for two sequences of equal length). Further combining the strategy with low-rank ideas allows to compute a full sequential kernel matrix that is both linear in the maximum length of sequence and the number of sequences, culminating in Algorithm~\ref{alg:sigpwLRdbl}.

\begin{table}[h]\label{complexity.overview}
\centering
\small
\begin{tabular}{|c|c|c|c|}
\hline
method & algorithm & complexity & storage\\
\hline
\hline
naive evaluation & Proposition~\ref{Prop:sigrecursion} & $O(N^2\cdot L^{2\IndLvl})$ & $O(1)$\\
dynamic programming & Algorithm~\ref{alg:sigpw} & $O(N^2 \cdot L^2\cdot \IndLvl)$ & $O(L^2)$\\
DP \& LR, per-element & Algorithm~\ref{alg:sigpwLR} & $O(N^2 \cdot L\cdot \rho\cdot \IndLvl)$ & $O(L\cdot \rho)$\\
DP \& LR, per-sequence & Algorithm~\ref{alg:sigpw} & $O((N+ \rho) \cdot L^2\cdot \rho^2\cdot \IndLvl)$ & $O(N\cdot L^2)$\\
DP \& LR, simultaneous & Algorithm~\ref{alg:sigpwLRdbl} & $O(N\cdot L\cdot \rho\cdot \IndLvl)$ & $O(N\cdot L\cdot \rho)$\\
\hline
\end{tabular}
\caption{Overview over presented algorithms to compute the sequential kernel $\KSeqBk$, their computational cost and storage cost when evaluating an $(N\times N)$ kernel matrix between sequences of length at most $L$. Methods: DP = dynamic programming, LR = low-rank. In the low-rank methods, $\rho$ is a meta-parameter which also controls accuracy of approximation and prediction, not necessarily in the same way for the different algorithms.\label{complexity_overview}}
\end{table}

The higher order sequential kernel $\KSeqBDk$ will not be discussed to the same degree of detail. It is demonstrated in Algorithm~\ref{alg:sigpwhigh} how the simple dynamic programming Algorithm~\ref{alg:sigpw} changes when the approximation is carried out to a higher order $D$. All dynamic programming algorithms listed in Table~\ref{complexity_overview} may be modified in the same way while incurring an additional factor $D$ in computation and storage requirements.

\begin{Rem}\label{rem:inner_product_pcw_linear}
  Let $\sigma,\tau\in\BV(\RR^d)$, $\bs,\bt\in\simplex$ with
  $\ell(\bs),\ell(\bt)\le L$.
  Denote with $\sigma^{\bs},\tau^{\bt} \in \BV(\RR^d)$ the bounded variation paths that
  are the piecewise linear interpolation of points
  $\sigma(\bs)$,$\tau(\bt)$.
  Note that $\Sigk(\sigma^{\bs})$,$\Sigk(\tau^{\bt})$ are composed of
  $O(d^\IndLvl)$ real numbers which makes a naive evaluation of
  $\langle \Sigk(\sigma^{\bs}),\Sigk(\tau^{\bt}) \rangle_{\tensalg}$
   infeasible for moderately high $d$ or $\IndLvl$.
  On the other hand, Table ~\ref{complexity_overview} applied with $\calX=\RR^d$, $\kernel=\langle \cdot,\cdot
  \rangle_{\RR^d}$ and $\IndDegr=\IndLvl$, provides efficient methods
  for calculating this inner product since $\KSeqBDk(\sigma(\bs),\tau(\bt))=\langle \Sigk(\sigma^{\bs}),\Sigk(\tau^{\bt}) \rangle_{\tensalg}$.
\end{Rem}
\subsection{Dynamic programming for tensors}

Before giving algorithms for computing the sequential kernels, we introduce a number of notations and fast algorithms for dynamic programming subroutines which will allow to state the latter algorithms concisely and which are at the basis of fast computations.

\begin{Not}
We will denote the $(i_1,\dots, i_k)$-th element of a $k$-fold array (= degree $k$ tensor) $A$ by $A[i_1,\dots, i_k]$. Occasionally, for ease of reading, we will use ``$|$'' instead of ``$,$'' as a separator, for example $A[i_1,i_2|i_3,\dots, i_k]$ of the indices on the left side of ``$|$'' are semantically distinct from those on the right side, in the example to separate the group of indices $i_1,i_2$ from the group $i_3,\dots, i_k$.
The arrays in the remainder will all contain elements in $\RR$, and the indices will always be positive integers, excluding zero.
\end{Not}

\begin{Not}
For a function $f: \RR\rightarrow \RR$, and an array $A$, we will denote by $f(A)$ the array where $f$ is applied element-wise. I.e., $f(A)[i_1,\dots, i_k] = f(A[i_1,\dots, i_k])$.
Similarly, for $f: \RR^m\rightarrow \RR$ and arrays $A_1,\dots, A_m$, we denote
$$f(A_1,\dots, A_m)[i_1,\dots, i_k] = f(A_1[i_1,\dots, i_k],\dots, A_m[i_1,\dots, i_k]).$$
For example, $\frac{1}{2}\cdot A^2$ is the array $A$ having all elements squared, then divided by two. The array $A+B$ contains, element-wise, sums of elements of $A$ and $B$.
\end{Not}

\begin{Not}
We introduce notation for sub-setting, shifting, and cumulative sum. Let $A$ be a $k$-fold array of size $(n_1\times\dots\times n_k)$.
\begin{description}
\item[(i)] For an index $i_j$ (at $j$-th position), we will write $A[:,\dots,:,i_j,:,\dots, :]$ for $(k-1)$-fold array of size $(n_1\times\dots\times n_{j-1}\times n_{j+1}\times\dots n_k)$ such that $A[:,\dots,:,i_j,:,\dots, :][i_1,\dots,i_{j-1},i_{j+1},\dots, i_k]=A[i_1,\dots, i_k].$ We define in analogy, iteratively, $A[:,\dots,:,i_j,:,\dots,:,i_{j'},:,\dots,:]$, and so on. Arrays of this type are called slices (of $A$).
\item[(ii)] For an integer $m$, we will write $A[:,\dots,:,+m,:,\dots, :]$ for the $k$-fold array of size $(n_1\times\dots\times n_{j-1}\times (n_j + k) \times n_{j+1}\times\dots n_k)$ such that $A[:,\dots,:,i_j,:,\dots, :][i_1,\dots,i_{j-1},i_j+k,i_{j+1},\dots, i_k]=A[i_1,\dots, i_k],$ and where non-existing indices of $A$ are treated as zero. Arrays of this type are called shifted (versions of $A$). For negative $m$, the shifts will be denoted with a ``minus''-sign instead of a ``plus''-sign.
\item[(iii)] We will write $A[:,\dots,:,\boxplus,:,\dots, :]$, where $\boxplus$ is at the $j$-th position, for the $k$-fold array of size $(n_1\times\dots\times\dots n_k)$ such that $A[:,\dots,:,\boxplus,:,\dots, :][i_1,\dots, i_k]=\sum_{\kappa =1}^{i_j} A[i_1,\dots, i_{j-1},\kappa, i_{j+1},\dots, i_k].$ Arrays of this type are called slice-wise cumulative sums (of $A$).
\item[(iv)] We will write $A[:,\dots,:,\Sigma,:,\dots, :]$, where $\Sigma$ is at the $j$-th position, for the $(k-1)$-fold array of size $(n_1\times\dots\times n_{j-1}\times n_{j+1}\times\dots n_k)$ such that $A[:,\dots,:,\Sigma,:,\dots, :][i_1,\dots, i_{k-1}]=\sum_{\kappa =1}^{n_j} A[i_1,\dots, i_{j-1},\kappa, i_{j},\dots, i_{k-1}].$ Arrays of this type are called slice-wise sums (of $A$).
\end{description}
We will further use iterations and  mixtures of the above notation, noting that the index-wise sub-setting, shifting, and cumulation commute with each other. Therefore expressions such as $A[+1,:|\Sigma,-3]$ or $A[j|:,+3,\boxplus]$ are well-defined, for example. We will also use the notation $A[\boxplus - m,\dots]$ and $A[\boxplus + m,\dots]$ to indicate the shifted variant of the cumulative sum array.
\end{Not}

Before continuing, we would like to note that cumulative sums can be computed efficiently, in the order of the size of an array, as opposed to squared complexity of more naive approaches. The algorithm is classical, we present it for the convenience of the reader in Algorithms~\ref{alg:cumsumvec} and ~\ref{alg:cumsum}.

\begin{algorithm}[ht]
\caption{Computing the cumulative sum of a vector.\newline
\textit{Input:} A $1$-fold array $A$ of size $(n)$ \newline
\textit{Output:} The cumulative sum array $A[\boxplus]$ \label{alg:cumsumvec}}
\begin{algorithmic}[1]
    \State Let $Q\leftarrow A$.
    \For{ $\kappa = 2$ to $n$}
    \State $Q[\kappa] \leftarrow Q[\kappa-1] + A[\kappa]$
    \EndFor
    \State Return $Q$
\end{algorithmic}
\end{algorithm}

\begin{algorithm}[ht]
\caption{Computing the cumulative sum of an array.\newline
\textit{Input:} A $k$-fold array $A$ of size $(n_1,\times,\dots,\times n_k)$ \newline
\textit{Output:} The cumulative sum array $A[\boxplus,\dots, \boxplus,:,\dots, :]$ (up to the $m$-th index) \label{alg:cumsum}}
\begin{algorithmic}[1]
    \State Let $Q\leftarrow A$
    \For{ $\kappa = 2$ to $m$}
    \State Let $Q\leftarrow Q[:,\dots,:,\boxplus,:,\dots, :]$ (at the $\kappa$-th index), where the right side is computed via applying algorithm~\ref{alg:cumsumvec} slice-wise.
    \EndFor
    \State Return $Q$
\end{algorithmic}
\end{algorithm}

\subsection{Computing the sequential kernel}

We give a fast algorithm to compute the sequential kernel $\KSeqBk$, by using the recursive presentation from Proposition~\ref{Prop:sigrecursion}.

\begin{algorithm}[ht]
\caption{Evaluation of the sequential kernel $\KSeqBk$\newline
\textit{Input:} Ordered sequences $\bsigma,\btau \in \calX^+$. A kernel $\kernel:\calX^+\times \calX^+\rightarrow \RR$ to sequentialize. A cut-off degree $\IndLvl$. \newline
\textit{Output:} $\KSeqBk(\bsigma,\btau)$, as the sequentialization of $\kernel$ \label{alg:sigpw}}
\begin{algorithmic}[1]
    \State Compute an $(L\times L')$ array $K$ such that
    $K[i,j] = \seqdf\kernel(\bsigma,\btau)[i,j]$.\\
    (or, alternatively, obtain it as additional input to the algorithm) \label{alg:sigpw.kernline}
    \State Initialize an $(\IndLvl\times L\times L')$-array $A$.
    \State Set $A[1|:,:] \leftarrow K$.
    \For{ $m = 2$ to $\IndLvl$}
    \State Compute $Q\leftarrow A[m-1|\boxplus,\boxplus]$.\label{alg:sigpw.sumline}
    \State Set $A[m|:,:]\leftarrow K\cdot (1+ Q[+1,+1])$\label{alg:sigpw.loopline}
    \EndFor
    \State Compute $R\leftarrow 1+ A[\IndLvl|\Sigma,\Sigma]$
    \State Return $R$
\end{algorithmic}
\end{algorithm}

Algorithm~\ref{alg:sigpw} includes a cut-off degree $\IndLvl$ which for exact computation can be set to $\IndLvl=\min (L,L')$ but can be set lower to reduce computational cost when an approximation is good enough.

Note that following our convention on arrays, all multiplications in Algorithm~\ref{alg:sigpw} are entry-wise, not matrix multiplications, even though some arrays have the format of compatible matrices.

Correctness of Algorithm~\ref{alg:sigpw} is proved in Proposition~\ref{Prop:sigrecursion}. At the end of the algorithm, the array $A$ contains as elements $A[m|i,j]$ the contributions from sub-sequences $\bi\sqsubseteq [i], \bj\sqsubseteq [j]$, beginning at $i$ and $j$, and of total length at most $m$.

Disregarding the cost of computing $K$ which can vary depending on the exact form of $\kernel$ (but which is, usually, $O(n\ell(\bsigma)\ell(\btau)))$ time and $O(\ell(\bsigma)\ell(\btau))$ time, with constants that may depend on $\kernel$), the computational cost of Algorithm~\ref{alg:sigpw} is $O(\IndLvl\ell(\bsigma)\ell(\btau))$ elementary arithmetic operations (= the number of loop elements) and $O(\IndLvl\ell(\bsigma)\ell(\btau)$ units of elementary storage. The storage requirement can be reduced to $O(\ell(\bsigma)\ell(\btau))$ by discarding $A[m-1|:,:]$ from memory after step~\ref{alg:sigpw.loopline} each time.

Note that in each loop over the index $L$, a matrix $Q$ is pre-computed, to avoid a five-fold loop that would be necessary with the more naive version of line~\ref{alg:sigpw.loopline},
$$A[m|i,j] \leftarrow A[m|i,j]\cdot\left( 1+ \sum_{i'\gneq i}\sum_{j'\gneq j} A[m-1|i',j']\right),$$
that leads to a blown up computational cost of $O(\IndLvl\ell(\bsigma)^2\ell(\btau)^2)$ at the asymptotically insignificant gain of storing one $(\ell(\bsigma)\times \ell(\btau))$ matrix less (the matrix $Q$).

Note that the whole code can be directly translated to the vector of matrix operations commonly available in programming languages such as R, MATLAB, or Python.

\subsection{Computing the higher order sequential kernel}
\begin{algorithm}[ht]
\caption{Evaluation of the higher order sequential kernel $\KSeqBDk$\newline
\textit{Input:} Ordered sequences $\bsigma,\btau \in \calX^+$. A kernel $\kernel:\calX^+\times\calX^+\rightarrow \RR$ to sequentialize. A cut-off degree $\IndLvl$, an approximation order $\IndDegr,\IndDegr\le \IndLvl$. \newline
\textit{Output:} $\KSeqBDk(\bsigma,\btau)$, as the sequentialization of $\kernel$ \label{alg:sigpwhigh}}
\begin{algorithmic}[1]
    \State Compute an $(L\times L')$ array $K$ such that $K[i,j] = \seqdf\kernel(\bsigma, \btau)[i,j]$.\\
    (or, alternatively, obtain it as additional input to the algorithm)\label{alg:sigpwhigh.kernline}
    \State Initialize an $(\IndLvl\times D\times D\times L\times L')$-array $A$, all entries zero.
    \For{ $m = 2$ to $\IndLvl$}
    \State $D'\leftarrow \min (D,m-1)$\label{alg:sigpwhigh.lineA}
    \State $A[m|1,1|:,:] \leftarrow K\cdot (1+ A[m-1|\Sigma,\Sigma|\boxplus + 1,\boxplus + 1])$
    \For{ $d = 2$ to $D'$}
    \State $A[m|d,1|:,:] \leftarrow A[m|d,1|:,:]+\frac{1}{d}\cdot K\cdot A[m-1|d-1,\Sigma|:,\boxplus + 1]$.
    \State $A[m|1,d|:,:] \leftarrow A[m|1,d|:,:]+\frac{1}{d}\cdot K\cdot A[m-1|\Sigma,d-1|\boxplus + 1,:]$.
    \For{ $d' = 2$ to $D'$}
    \State $A[m|d,d'|:,:] \leftarrow A[m|d,d'|:,:]+\frac{1}{dd'}\cdot K\cdot A[m-1|d-1,d'-1|:,:]$.
    \EndFor
    \EndFor\label{alg:sigpwhigh.lineB}
    \EndFor
    \State Compute $R\leftarrow 1+ A[\IndLvl|\Sigma,\Sigma|\Sigma,\Sigma]$
    \State Return $R$
\end{algorithmic}
\end{algorithm}

All multiplications in Algorithm~\ref{alg:sigpwhigh} are entry-wise,
not matrix multiplications. At the end of Algorithm~\ref{alg:sigpwhigh}, the array $A$ contains as elements $A[m|d,d'|i,j]$ the contributions from sub-sequences $\bi\subseteq [i], \bj\subseteq [j]$, beginning at $i$ and $j$, with end-sequences $iii\dots$ of length $d$ and $jjj\dots$ of length $d'$, and of total length at most $\IndLvl$. We prove the recursion used in Algorithm~\ref{alg:sigpwhigh} and thus the correctness of this statement after introducing some necessary notation:

\begin{Not}
Let $\bt=(t_1,\ldots,t_M)\in \intvl^M$ be a sequence, for any set $\intvl$, and some $M\in \NN$. We will denote by
$$d(\bt):=\max\{m\;:\; t_1 = t_2 = \dots = t_m\}$$
the number of repetitions of the initial symbol.
\end{Not}

\begin{Prop}
Keep the notations of Algorithm~\ref{alg:sigpw}, let $\phi:\calX\rightarrow \calH$ the feature map associated with the kernel $\kernel$. Let $\bx,\by\in \calH^+$ such that $\bx = \phi(\bsigma),\by = \phi(\btau)$. Denote by
\begin{align*}
A[m|d,d'|i,j] &:= \sum_{\substack{\bx^\prime\sqsubseteq_D \seqdf \bx
  \\ \by^\prime \sqsubseteq_D \seqdf \by}}
  \frac{1}{\bx^\prime!\by^\prime!}\langle \bx^\prime,
  \by^\prime\rangle_{\calH^+} = \sum_{\substack{\bi\sqsubseteq_D
  [\ell(\bx)] \\ \bj\sqsubseteq_D [\ell(\by)]\\ \ell(\bi) =
  \ell(\bj)\le m}} \frac{1}{\bi!\bj!} \prod_{\kappa=1}^{\ell(\bi)}
  \langle (\seqdf \bx)[i_\kappa], (\seqdf \by)[j_\kappa] \rangle_\calH
\end{align*}
where the sums are additionally restricted in the following way:
\begin{align*}
i = i_{1} = i_{2} = \dots = i_{d} \neq i_{d+1}\quad \mbox{and}\quad j = j_{1} = j_{2} = \dots = j_{d'} \neq j_{d'+1}. \quad \mbox{Or, equivalently,}\\
\bx[i] = \bx^\prime[1] = \bx^\prime[2] = \dots = \bx^{\prime}[d] \neq \bx^\prime[d+1]\quad \mbox{and}\quad \by[j] = \by^\prime[1] = \by^\prime[2] = \dots = \by^\prime[d] \neq \by^\prime[d+1]
\end{align*}
Then, the following recursion equalities hold:
\begin{align*}
A[m|1,1|i,j] & = \langle \bx[i], \by[j]\rangle_\calH \cdot \left(1+\sum_{i'\gneq i}\sum_{j'\gneq j} A[m-1|d-1,d'-1|i',j']\right),\\
A[m|d,1|i,j] & = \frac{1}{d}\cdot\langle \bx[i], \by[j]\rangle_\calH \cdot \sum_{j'\gneq j}\sum_{\kappa=1}^D A[m-1|d-1,\kappa|i,j']\quad\mbox{for}\;d\ge 2,\\
A[m|1,d'|i,j] & = \frac{1}{d'}\cdot\langle \bx[i], \by[j]\rangle_\calH \cdot \sum_{i'\gneq i}\sum_{\kappa=1}^D A[-1|\kappa,d'-1|i',j]\quad\mbox{for}\;d'\ge 2,\\
A[m|d,d'|i,j] & = \frac{1}{dd'}\cdot\langle \bx[i],\by[j]\rangle_\calH \cdot A[m-1|d-1,d'-1|i,j]\quad\mbox{for}\;d,d'\ge 2.
\end{align*}
\end{Prop}
\begin{proof}
Note that the sums on the right hand side that define $A[m|d,d'|i,j]$ as a (weighted) sum over elements parameterised by paired index sequences $I,J$ of the same length. Note that by definition, the sum goes over all index sequences such that $\ell(\bi)=\ell(\bj)\le \IndLvl$, $d(\bi) = d$, and $d(\bj) = d'$.

With this, the statement follows from comparing the summation in the loop between Line~\ref{alg:sigpwhigh.lineA} and~\ref{alg:sigpwhigh.lineB} of Algorithm~\ref{alg:sigpwhigh} with the explicit formula in Proposition~\ref{Prop:sighigh}.
\end{proof}

The computational cost of Algorithm~\ref{alg:sigpwhigh} is $O(D^2\IndLvl\ell(\bsigma)\ell(\btau))$ elementary arithmetic operations (= the number of loop elements) and $O(D^2\ell(\bsigma)\ell(\btau))$ units of elementary storage (when freeing up space for array entries directly after the last time they are read out).

\subsection{Large scale strategies}

Even though a computational cost of $O(\IndLvl\ell(\btau)\ell(\bsigma))$ for the sequential kernel via Algorithm~\ref{alg:sigpw} (for ease of reading, we will not discuss the higher order kernel, most considerations hold in analogy) can be considered efficient in a polynomial time, one has to note that this is the cost of evaluating $\KSeqBk (\bsigma,\btau)$ for a single pair of sequences $\bsigma,\btau\in\calX^+$. Thus, computation of a symmetric kernel matrix of $N$ sequences, of length at most $\IndLvl$, would cost $O(N^2\cdot L^2\cdot \IndLvl)$ elementary arithmetic operations when done by iterating over entries. While this is for moderate sizes of $N$ and $\IndLvl$ still achievable on contemporary desktop computers, it may become quickly prohibitive when combined with parameter tuning or cross-validation schemes (as later in our experiments). Furthermore, there exist regimes (low signal dimension $n$, large length $L$) in which an explicit computation of features plus subsequent inner product, with a complexity of $O(N^2 \cdot L \cdot n^\IndLvl)$, may be faster due to the linear dependence on $L$ at the cost of an exponential dependence on $\IndLvl$.

We present below a number of approaches by which the above-mentioned issues may be addressed. These are somewhat independent and address different parts of the total complexity in different ways, but can be in-principle combined.

\subsubsection{Low-rank methods for the sequence-vs-sequence kernel matrix}

The sequential kernel is a kernel on ordered data, therefore learning algorithms which use the kernel matrix as an input, such as support vector machines, kernel ridge regression, or kernel principal components, are in-principle directly amenable to large-scale variants of low-rank type. Strategies of this kind include the incomplete Cholesky decomposition, Nystr\"om approximation, or the inducing point formalism in a Gaussian process framework.

For $N$ sequences, all strategies mentioned above require evaluation of at most an $(N\times r)$ and an $(r\times r)$ matrix (where $r$ is a meta-parameter), which costs $O((r+N)\cdot r\cdot L^2\cdot \IndLvl)$ elementary operations and $O((r+N)\cdot r\cdot L^2)$ storage, followed by a slightly modified variant of the learning algorithm itself which usually costs $O((r+N)\cdot r^2)$ elementary operations and storage (at most) instead of the unmodified variant which usually costs $O(N^3)$ (at most).

This alleviates the dependency of the sequential kernel evaluation on $N$, but not on $L$; which is not unexpected, since the strategy is completely independent of how the sequential kernel was evaluated. For the same reason, any improvements on the cost of single evaluations will combine with the above improvement un the number of evaluations.

\subsubsection{Low-rank methods for the element-vs-element kernel matrix}

A second kernel matrix is crucial to obtaining a single evaluation of type $\KSeqBk (\bsigma,\btau)$, namely the cross-kernel matrix between the elements $\bsigma[i],\btau[j]$ of both sequences which is the object underlying the computations, see Line~\ref {alg:sigpw.kernline} of Algorithm~\ref {alg:sigpw}, and Line~\ref{alg:sigpwhigh.kernline} of Algorithm~\ref{alg:sigpwhigh}. Summations and multiplications are performed on this cross-kernel matrix until the final result, $\KSeqBk (\bsigma,\btau)$, is obtained and returned.

The low-rank methods of the previous paragraph cannot be applied naively to the cross-kernel matrix. A minor issue is the fact that the cross-kernel matrix is in general non-symmetric, but the above-mentioned strategies (incomplete Cholesky, Nystr\"om, inducing points) translate verbatim to the context of non-symmetric cross-kernel matrices, by replacing the respective symmetric decomposition with the analogue non-symmetric one. The major issue consists in the summation- and multiplication-type operations which are performed on the kernel matrix. In their naive form, these operations require access the full cross-kernel matrix, and without modification will therefore give rise to the same computational complexity irrespectively of whether a low-rank decomposition of the initial cross-kernel matrix is considered, or not.

We show how this can be circumvented by working exclusively on low-rank factorizations.

\begin{Def}
Let $A$ be an $(a\times b)$-array. For $U$ an $(a\times r)$-array, and $V$ a $(b\times r)$-array, we say that $(U,V)$ is a low-rank presentation of $A$, of rank $r$, if
$$A[i,j] = \left(U[i,:]\cdot V[j,:]\right)[\Sigma].$$
In matrix notation, this is equivalent to saying that $A=UV^\top$.
\end{Def}

Note that in matrix terms, the fact that $A$ has a low-rank presentation of rank $r$ does imply that $A$ is of rank $r$ or less (by equivalence of matrix rank with decomposition rank), but it does not imply that $A$ is of rank exactly $r$.

We state a number of straightforward but computationally useful Lemmas:

\begin{Lem}
\label{Lem:LR-sum}
Let $A$ be an $(a\times b)$-array with low-rank presentation $(U,V)$. Then:
\begin{description}
\item[(i)] For $m\in\NN$, a low-rank presentation of $A[+m , :]$ is $(U[+m, :], V)$. A low-rank presentation of $A[:, +m]$ is $(U, V[+m, :])$.
\item[(ii)] A low-rank presentation of $A[\boxplus , :]$ is $(U[\boxplus, :], V)$. A low-rank presentation of $A[:, \boxplus]$ is $(U, V[\boxplus, :])$.
\item[(iii)] A low-rank presentation of $A[\Sigma , :]$ is $(U[\Sigma, :], V)$. A low-rank presentation of $A[: , \Sigma]$ is $(U, V[\Sigma, :])$.
\end{description}
\end{Lem}
\begin{proof}
The statements follow directly from writing out the decompositions.
\end{proof}

\begin{Lem}
\label{Lem:LR-add}
Let $A_1,A_2$ be $(a\times b)$-arrays, let $U_i$ be arrays of size $(a\times r_i)$, let $V_i$ be arrays of size $(b\times r_i)$, for $i=1,2$, for some $r_i\in\NN$. Let $A:= A_1+A_2$, write $U$ for the $(a\times (r_1+r_2))$-array obtained by concatenating $U_1,U_2$, and $V$ or the $(b\times (r_1+r_2))$-array obtained by concatenating $V_1,V_2$ (such that the order of indices matches). Then, the following are equivalent:
\begin{description}
\item[(i)] $(U,V)$ is a low-rank presentation of $A$, of rank $r_1+r_2$.
\item[(ii)] $(U_1,V_1)$ is a low-rank presentation of $A_1$, of rank $r_1$, and $(U_2,V_2)$ is a low-rank presentation of $A_2$, of rank $r_2$.
\end{description}
\end{Lem}
\begin{proof}
The statement follows directly from writing out the decompositions of $A$ and $A_1+A_2$ in terms of $U_1,U_2,V_1,V_2$ and observing that they are formally equal.
\end{proof}

\begin{Not}
We introduce notation for index repetition. Let $A$ be a $k$-fold array of size $(n_1\times\dots\times n_k)$.
\begin{description}
\item[(i)] For an integer $m$, we will write $A[:,\dots,:,m\cdot :,:,\dots, :]$ for the $k$-fold array of size $(n_1\times\dots\times n_{j-1}\times (m\cdot n_j) \times n_{j+1}\times\dots n_k)$ such that $A[:,\dots,:,m\cdot :,:,\dots, :][i_1,\dots,i_{j-1},i_j,i_{j+1},\dots, i_k]=A[i_1,\dots,i_{j-1},r,i_{j+1},\dots, i_k],$ where $r$ is the remainder in integer division of $i_j$ by $m$. This is intuitively equivalent to concatenating $m$ copies of $A$ along the $j$-th direction.
\item[(ii)] For an integer $m$, we will write $A[:,\dots,:,:\cdot m,:,\dots, :]$ for the $k$-fold array of size $(n_1\times\dots\times n_{j-1}\times (m\cdot n_j) \times n_{j+1}\times\dots n_k)$ such that $A[:,\dots,:,:\cdot m,:,\dots, :][i_1,\dots,i_{j-1},i_j,i_{j+1},\dots, i_k]=A[i_1,\dots,i_{j-1},q,i_{j+1},\dots, i_k],$ where $q$ is the quotient in integer division of $i_j$ by $m$. This is intuitively equivalent to repeating each slice of $A$ along the $j$-th direction $m$ times.
\end{description}
\end{Not}

\begin{Lem}
\label{Lem:LR-mult}
Let $A_1,A_2$ be $(a\times b)$-arrays, let $U_i$ be arrays of size $(a\times r_i)$, let $V_i$ be arrays of size $(b\times r_i)$, for $i=1,2$, for some $r_i\in\NN$. Let $A:= A_1\cdot A_2$ (i.e., by our convention, component-wise multiplication), write $U$ for the $(a\times (r_1\cdot r_2))$-array $U_1[:,:\cdot r_2]\cdot U_2[:,r_1\cdot :]$, and $V$ or the $(b\times (r_1+r_2))$-array $V_1[:,:\cdot r_2]\cdot V_2[:,r_1\cdot :]$ (multiplication is per convention component-wise). Then, the following are equivalent:
\begin{description}
\item[(i)] $(U,V)$ is a low-rank presentation of $A$, of rank $r_1\cdot r_2$.
\item[(ii)] $(U_1,V_1)$ is a low-rank presentation of $A_1$, of rank $r_1$, and $(U_2,V_2)$ is a low-rank presentation of $A_2$, of rank $r_2$.
\end{description}
\end{Lem}
\begin{proof}
The statement follows directly from writing out the decompositions of $A$ and $A_1\cdot A_2$ in terms of $U_1,U_2,V_1,V_2$ and observing that they are formally equal.
\end{proof}

Lemmas~\ref{Lem:LR-sum}, \ref{Lem:LR-add} and~\ref{Lem:LR-mult} cover all operations performed on the kernel matrix in Algorithms~\ref{alg:sigpw} and~\ref{alg:sigpwhigh}, namely, shifting, summation and cumulative summation, component-wise addition and multiplication. Therefore the respective manipulations on low-rank factors may be used to replace all operations within the algorithm. For convenience of the reader, we explicitly present Algorithm~\ref{alg:sigpwLR} which is a low-rank version of Algorithm~\ref{alg:sigpw}.

\begin{algorithm}[ht]
\caption{Evaluation of the sequential kernel $\KSeqBk$, with low-rank speed-up\newline
\textit{Input:} Ordered sequences $\bsigma,\btau \in \calX^+$. A kernel $\kernel:\calX^+\times \calX^+\rightarrow \RR$ to sequentialize. A cut-off degree $\IndLvl$. \newline
\textit{Output:} $\KSeqBDk(\bsigma,\btau)$, as the sequentialization of $\kernel$ \label{alg:sigpwLR}}
\begin{algorithmic}[1]
    \State Let $L\leftarrow \ell(\bsigma), L'\leftarrow \ell(\btau)$.
    \State Compute arrays $U,V$ such that $(U,V)$ is a low-rank presentation of rank $\upsilon$, approximating the kernel matrix $K$ with $K[i,j] = k(\seqdf \bsigma[i],\seqdf \btau[j])$\\
    (or, alternatively, obtain it as additional input to the algorithm).
    \State Initialize an $(\IndLvl\times L \times *)$-array $B$ and an $(\IndLvl\times L'\times *)$-array $C$ (where * means that the size may change dynamically).
    \State Set $B[1|:,:] \leftarrow U$ and $C[1|:,:] \leftarrow V$.
    \For{ $m = 2$ to $\IndLvl$}
    \State Compute $P\leftarrow B[m-1|\boxplus +1,:]$ and $Q\leftarrow C[m-1|\boxplus +1,:]$. \label{alg:sigpwLR.line5}
    \State Append an $(\IndLvl\times 1)$-array of ones to $P$, append an $(L'\times 1)$-array of ones to $Q$ \label{alg:sigpwLR.line6}.
    \State Set $\rho$ such that $(\IndLvl\times \rho)$ is the size of $P$, and $(L'\times \rho)$ is the size of $Q$.
    \State Set $B[m|:,:] \leftarrow U[:,:\cdot \rho]\cdot P[:,\upsilon\cdot :]$ \label{alg:sigpwLR.line8}
    \State Set $C[m|:,:] \leftarrow V[:,\rho\cdot :]\cdot Q[:,:\cdot \upsilon]$ \label{alg:sigpwLR.line9}
    \State optional: ``simplify'' the low-rank presentation $(B,C)$, reducing its rank \label{alg:sigpwLR.line10}
    \EndFor
    \State Compute $R\leftarrow B[\IndLvl|\Sigma,:]$ \label{alg:sigpwLR.line12}
    \State Compute $S\leftarrow C[\IndLvl|\Sigma,:]$ \label{alg:sigpwLR.line13}
    \State Return $1 + (R\cdot S)[\Sigma]$ \label{alg:sigpwLR.line14}
\end{algorithmic}
\end{algorithm}

Algorithm~\ref{alg:sigpwLR} is obtained from Algorithm~\ref{alg:sigpw} as follows: Line~\ref{alg:sigpwLR.line5} is via Lemma~\ref{Lem:LR-sum}~(i) and~(ii). Line~\ref{alg:sigpwLR.line6}, is via~ \ref{Lem:LR-add} and observing that a low-rank presentation of the all-ones matrix is a pair of all-ones vectors. Lines~\ref{alg:sigpwLR.line8} and~\ref{alg:sigpwLR.line9} are via~\ref{Lem:LR-mult}. Lines~\ref{alg:sigpwLR.line12} to~\ref{alg:sigpwLR.line13} are via Lemma~\ref{Lem:LR-sum}~(iii). Line~\ref{alg:sigpwLR.line14} is evaluation, following the definition of low-rank presentation. Line~\ref{alg:sigpwLR.line10} is optional, aiming at keeping size of the low-rank presentations low (and thus the computational cost); it can be achieved for example via singular value decomposition type techniques, sub-sampling techniques, or random projection type techniques.

The higher order Algorithm~\ref{alg:sigpwhigh} can be treated in a similar way, by applying the low-rank representation to the matrices/2D-arrays $A[m|i,j|:,:]$. All assignments and manipulations can be re-phrased in those matrices, therefore the same strategy applies.

The computational cost of one run of Algorithm~\ref{alg:sigpwLR} is of the same order as the maximum size of $B$ and $C$. That is, if $\rho$ is the smallest integer such that at any time $B$ requires $\ell(\bsigma)\cdot \IndLvl\cdot \rho$ space, and $C$ requires $\ell(\btau)\cdot \IndLvl\cdot \rho$ space, then the computational complexity of Algorithm~\ref{alg:sigpwLR} is $O((L+L')\cdot \rho\cdot \IndLvl)$. Noting that the one can always choose a low-rank representation of $B[m|:,:]$ and $C[m|:,:]$ of rank $\min(L,L')$ or less, the computational complexity of the low-rank algorithm is always bounded by $O(L\cdot L'\cdot \IndLvl)$, which is the complexity of Algorithm~\ref{alg:sigpw}.

For the linear kernel, one can infer that the rank will be bounded by $\rho\le n^{\IndLvl}$, by using that $K$ admits a low-rank presentation of rank $n$, then keeping track of matrix sizes: each of the $(\IndLvl-1)$ repetitions of Lines~\ref{alg:sigpwLR.line8} resp.~\ref{alg:sigpwLR.line9} enlarges the size of $B,C$ by a multiplicative factor of $n$.

\subsubsection{Simultaneous low-rank methods}

Algorithm~\ref{alg:sigpwLR} yields an efficient low-rank speed-up for computing a single element of the kernel matrix. When employing this speed-up for each entry of the final kernel matrix of size $N$, the computational cost is $O(N^2\cdot L\cdot \rho\cdot k)$. As stated before, a cost quadratic in $N$ may be prohibitive on large scale data.

This can be addressed by combining both low-rank strategies mentioned before, on sequence-vs-sequence and element-vs-element basis. For this, note that the computation of $R$ depends only on $U$, and the computation of $S$ depends only on $V$. If $U$ is chosen to depend only on $r$, and $V$ only on $s$, one notes that Algorithm~\ref{alg:sigpwLR} can be split into computation of $R$ an $S$ for each $r$ and $s$, thus allowing a further reduction of computational cost from $O(N^2\cdot (L+L')\cdot \rho\cdot \IndLvl)$ to $O(N\cdot (L+L')\cdot \rho\cdot \IndLvl)$.

Pseudo-code is given in Algorithm~\ref{alg:sigpwLRdbl}. It is obtained from Algorithm~\ref{alg:sigpwLR} by starting with a joint low-rank decomposition of element-vs-element kernel matrices, then separating the otherwise redundant computations of the factors $R,S$.

\begin{algorithm}[ht]
\caption{Computation of the sequential kernel matrix $\KSeqBk$, with (double) low-rank speed-up\newline
\textit{Input:} Ordered sequences $\bsigma_1,\dots, \bsigma_N \in \calX^+$. A kernel $\kernel:\calX^+\times \calX^+\rightarrow \RR$ to sequentialize. A cut-off degree $\IndLvl$. \newline
\textit{Output:} A matrix $U$ such that $(U,U)$ is a low-rank presentation of the kernel matrix $K\in\RR^{N\times N}$ where $K_{ij}=\KSeqBk(\bsigma_i,\bsigma_j)$.
 \label{alg:sigpwLRdbl}}
\begin{algorithmic}[1]
    \State Compute arrays $U^{(i)}$ of such that each $(U^{(i)},U^{(j)})$ forms a (joint) low-rank presentation of rank $\upsilon$, approximating the element-vs-element kernel matrices $K^{(ij)}$ with $K^{(ij)}[a,b] = \kernel(\seqdf \bsigma_i[a],\seqdf \bsigma_j[b])$\\
    (or, alternatively, obtain it as additional input to the algorithm)\label{alg:sigpwLR.line1}.
    \State Initialize an $(N\times \IndLvl\times *\times *)$-array $B$ (where * means that the sizes may change dynamically).
    \State Set $B[i|1|:,:] \leftarrow U^{(i)}$ for all $i\in [N]$.
    \For{ $m = 2$ to $\IndLvl$}
    \State Compute $P\leftarrow B[:|m-1|\boxplus +1,:]$. \label{alg:sigpwLR.line5}
    \State Set $\kappa,\rho$ such that $(N\times \kappa\times \rho)$ is the size of $P$.
    \State Append an $(N\times \kappa\times 1)$-array of ones to $P$ \label{alg:sigpwLR.line6}.
    \State Set $B[:|m|:,:] \leftarrow B[:|1|:,:\cdot \rho]\cdot P[:,:,\upsilon\cdot :]$ \label{alg:sigpwLR.line8}
    \State optional: ``simplify'' the low-rank presentation encoded in $B$, reducing its rank \label{alg:sigpwLR.line10}
    \EndFor
    \State Compute $U\leftarrow B[:|\IndLvl|\Sigma,:]$ \label{alg:sigpwLR.line12}
    \State Return $U$ \label{alg:sigpwLR.line14}
\end{algorithmic}
\end{algorithm}

In line~\ref{alg:sigpwLR.line1}, the algorithm starts with a joint low-rank presentation of the element-wise kernel matrices - that is, the matrices $U^{(i)}$, when row-concatenated, should have low rank. For example, if $\kernel$ is the Euclidean scalar product, $U^{(i)}$ can be taken as the raw data matrix for $s_i$, where rows are different time points and columns are features. More generally, jointly low-rank $U^{(i)}$ can be obtained by running a suitable joint diagonalization or singular value decomposition scheme on the element-vs-element kernel matrices $K^{(ij)}$.

Note that such a joint low-rank decomposition may require choice of a higher rank $\rho$ for some kernels $\kernel$ than when only a single entry of the kernel matrix is computed as in Algorithm~\ref{alg:sigpwLRdbl}.

\subsubsection{Fast sequential kernel methods}

Following the analogy of sequential kernels to string kernels established in Section~\ref{sec:discr.string}, fast string kernel methods such as the gappy, substitution, or mismatch kernels presented in~\cite{leslie04faststringkernels} may be transferred to general sequential kernels. In general, this amounts to small modification of Algorithm~\ref{alg:sigpw}; for example, to obtain a gappy variant of the sequential kernel, summation in line~\ref{alg:sigpw.sumline} of Algorithm~\ref{alg:sigpw} over the whole matrix, of quadratic size, is replaced by summation over a linear part of it. The scaling factor $\lambda$ also may or may not be added in.

It should be noted that not all fast string kernel ideas combine straightforwardly with the low-rank methods introduced above, though they can be adapted. For example, for the gappy kernel, one may consider a joint low-rank decomposition of element-vs-element cross-kernel matrices where suitable entries have been set to zero.

\section{Experimental validation}

We perform two experiments to validate the practical usefulness of the sequential kernel:
\begin{description}
\item[(1)] On a real world dataset of hand movement classification (eponymous UCI dataset~\cite{sapsanis2013emg}), we show the sequential kernel outperforms the best previously reported predictive performance~\cite{sapsanis2013emg}, as well as non-sequential kernel and aggregate baselines.
\item[(2)] On a real world dataset on hand written digit recognition (pendigits), we show that the sequentialization of the Euclidean kernel (= linear use of signature features) achieves only sub-baseline performance, similarly to previously reported results~\cite{Diehlinvariants}. Using the sequentialized Gaussian kernel improves prediction accuracy to the baseline region.
\end{description}

We would like to stress that our experiments do not constitute a systematic benchmark comparison to prior work, only validation that the sequential kernel is a practically meaningful concept: result (1) validates the first kernelization step in the sense that the order information captured by the sequential kernel can be useful, when compared to alternatives which ignore it. Result (2) validates the second kernelization step in the sense that using a non-linear kernel in sequentialization may outperform the linear kernel.

A benchmark comparison is likely to require a larger amount of work, since it would have to include a number of previous methods (multiple variants of the string and general alignment kernels, dynamic time warping, naive use of signatures), for most of which there is no freely available code with interface to a machine learning toolbox, and benchmark methods (order-agnostic baselines such as summary aggregation and chunking; distributional regression; naive baselines) which have not been compared to in literature previously.

For the benefit of the scientific community, we have decided to share our more theoretical results early and provide the opportunity to others to work with a toolbox-compatible implementation of the sequential kernels (code link will be provided here shortly), acknowledging that further experimentation is desirable. We will supply benchmark comparisons at a later time point.

\subsection{Validation and prediction set-up}
\label{Exp:setup}

\subsubsection{Prediction tasks}
In all datasets, samples are multi-variate (time) series. All learning tasks are supervised classification tasks of predicting class labels attached to series of equal length.

\subsubsection{Prediction methods}
For prediction, we use eps-support vector classification (as available in the python/scikit-learn package) on the kernel matrices obtained from the following kernels:
\begin{description}
\item[(1.a)] the Euclidean kernel $\kernel(x,y) = \langle x,y\rangle$. This kernel has no parameters.
\item[(1.b)] the Gaussian kernel $\kernel(x,y) = \exp \left(\frac{1}{2}\gamma^2 \|x-y\|^2\right)$. This kernel has one parameter, a scaling constant $\gamma$.
\item[(2.a)] the (truncated) sequentialization $\kernel^+_{\le M}$ of the linear/Euclidean kernel $\kernel(x,y) = \gamma \langle x,y\rangle$. This sequential kernel has two parameters, a scaling constant $\gamma$, and the truncation level $\IndLvl$.
\item[(2.b)] the (truncated) sequentialization $\kernel^+_{\le M}$ of the Gaussian kernel $\kernel(x,y) = \theta \exp \left(\frac{1}{2}\gamma^2 \|x-y\|^2\right)$. This sequential kernel has three parameters: scaling constants $\gamma$ and $\theta$, and truncation level $\IndLvl$.
\end{description}
(1.a) and (1.b) are considered standard kernels, (2.a) and (2.b) are sequential kernels.
Note that the non-sequential kernels (1.a) and (1.b) can only be applied to sequential data samples of equal length which is the case for the datasets considered. Note that even though (1.a), (1.b) may be applied to sequences of same length, they do not use any information about their ordering: both the Euclidean and the Gaussian kernel are invariant under (joint) permutation of the order of the indexing in the arguments.

We would also like to note a further subtlety: the sequential kernels (2.a), (2.b) do use information about the ordering of the sequences, but only for a truncation $\IndLvl\ge 2$. For $\IndLvl = 1$, the kernel corresponds to choosing the increment/mean aggregate feature (Euclidean) or a type of distributional classification (Gaussian). We will therefore explicitly compare truncation levels $1$ versus $2$ and higher, to enable us to make a statement about whether using the order information was beneficial (or not).

There will be no further baseline, naive, or state-of-art predictors in the set-up, comparison will be conducted between the kernel classifiers and performances reported in literature.

We would note that this is a limitation in our set-up which we will rectify in future (more time consuming) instances of a larger benchmarking experiment. The current experiments merely aim to validate whether the sequential kernel is a practically meaningful concept, in particular whether each of the two kernelization steps is practically useful, and whether making use of order information is beneficial.

\subsubsection{Tuning and error estimation}
In all experiments, we use nested (double) cross-validation for parameter tuning (inner loop) and estimation of error metrics (outer loop). In both instances of cross-validation, we perform uniform 5-fold cross-validation.

Unless stated otherwise, parameters are tuned on the tuning grid given in Table~\ref{Exp:parameters} (when applicable). Kernel parameters are the same as in the above section ``prediction mehods''. The best parameter is selected by 5-fold cross-validation, as the parameter yielding the minimum test-f1-score, averaged over the five folds.

\begin{table}[h!]
\centering
\begin{tabular}{c|c}
parameter & range\\
\hline
kernel param.~ $\gamma$ & 0.01, 0.1, 1\\
kernel param.~$\theta$ & 0.01, 0.1, 1\\
truncation level $\IndLvl$ & 1,2,3\\
SVC regularizer& 0.1, 1, 10, 100, 1000
\end{tabular}
\caption{Tuning grid}\label{Exp:parameters}
\end{table}

\subsubsection{Error metrics}

The out-of-sample classification error is reported as precision, recall, and f1-score of out-of-sample prediction on the test fold. Errors measures are aggregated with equal weights on classes and folds. These aggregates are reported in the result tables.

\subsection{Experiment: Classifying hand movements}

We performed classification with the eps-support vector machine (SVC) on the hand movements dataset from UCI~\cite{sapsanis2013emg}. The first database in the dataset which we considered for this experiment contains, for each of five subjects (two male, three female) 180 samples of hand movement sEMG recordings. Each sample is a time series in two variables (channels) at 3.000 time points. The time series fall into six classes of distinct types of hand movement (spherical, tip, palmar, lateral, cylindrical, hook). For each subject, 30 samples of each class were recorded. Hence, for each subject, there is a total of 180 sequences in $\calX^{3000}$ with $\calX=\RR^2$.

For each of the five subjects, we conducted the classification experiment as described in Section~\ref{Exp:setup}, comparing prediction via SVC using one of the following three kernels: (1.a) the Euclidean kernel, (1.b) the Gaussian kernel, (2.a) the sequentialized Euclidean kernel. For the non-sequential kernels (1.a), (1.b), prediction was performed with and without prior standardization of the data. For the sequential kernel, the tuning grid was considered in two parts: a cut-off level of $\IndLvl = 1$, corresponding to mean aggregation, and cut-off levels of $\IndLvl = 2,3$, corresponding to the case where genuine sequence information is used.

The results are reported in Tables~\ref{Tab:handmove1} to~\ref{Tab:handmove5}. Jackknife standard errors (pooling the five folds) are all 0.04 or smaller. Baseline performance of an uninformed estimator is $1/6 \approx 0.17$.

\begin{table}[h!]
    \scriptsize
        \begin{minipage}{0.5\textwidth}

            \centering
            \begin{tabular}{l|c|c|c}
              method&precision&recall&f1-score\\
              \hline
              (1.a) linear &  0.37   &   0.38   &   0.36\\
              (1.a) linear, standardized  &  0.33   &   0.32   &   0.29 \\
              (1.b) Gaussian  &   0.57   &   0.59  &    0.56   \\
              (1.b) Gaussian, standardized &    0.54  &    0.50 &     0.50  \\
              (2.a) mean aggregation &   0.19   &  0.20 &     0.18\\
              \hline
              (2.a) sequential, level $\ge 2$ &    0.87&      0.86&      0.86
            \end{tabular}
            \caption{female1.mat}
            \label{Tab:handmove1}
        \end{minipage}
        \begin{minipage}{0.5\textwidth}
            \centering
            \begin{tabular}{l|c|c|c}
  method&precision&recall&f1-score\\
\hline
(1.a) linear &  0.47   &   0.39   &   0.37  \\
(1.a) linear, standardized  &  0.31  &    0.28   &   0.27   \\
(1.b) Gaussian  &   0.71  &    0.71   &   0.70   \\
(1.b) Gaussian, standardized &    0.59   &   0.58   &   0.56   \\
(2.a) mean aggregation &   0.18  &  0.20&     0.18\\
\hline
(2.a) sequential, level $\ge 2$ &    0.94  &    0.97 &     0.95

\end{tabular}
\caption{female2.mat}
\label{Tab:handmove2}
        \end{minipage}
\end{table}

\begin{table}[h!]
    \scriptsize
        \begin{minipage}{0.5\textwidth}

            \centering
\begin{tabular}{l|c|c|c}
  method&precision&recall&f1-score\\
\hline
(1.a) linear &   0.48  &    0.46   &   0.46 \\
(1.a) linear, standardized  &  0.47  &    0.42   &   0.43   \\
(1.b) Gaussian  &  0.66  &    0.64  &    0.63  \\
(1.b) Gaussian, standardized &    0.54  &    0.51   &   0.50    \\
(2.a) mean aggregation &   0.26&      0.23&      0.20\\
\hline
(2.a) sequential, level $\ge 2$ &    0.96&      0.96&      0.96

\end{tabular}
\caption{female3.mat}
\label{Tab:handmove3}
        \end{minipage}
        \begin{minipage}{0.5\textwidth}
            \centering
\begin{tabular}{l|c|c|c}
  method&precision&recall&f1-score\\
\hline
(1.a) linear &   0.37  &    0.33    &  0.33 \\
(1.a) linear, standardized  &  0.38   &   0.36   &   0.36   \\
(1.b) Gaussian  &  0.59   &   0.57   &   0.57  \\
(1.b) Gaussian, standardized &    0.53   &   0.54   &   0.53    \\
(2.a) mean aggregation &     0.20&      0.18&      0.17\\
\hline
(2.a) sequential, level $\ge 2$ &    0.96&      0.96&      0.96

\end{tabular}
\caption{male1.mat}
\label{Tab:handmove4}
        \end{minipage}
\end{table}

\begin{table}[h!]
\centering
\scriptsize
\begin{tabular}{l|c|c|c}
  method&precision&recall&f1-score\\
\hline
(1.a) linear &   0.36   &   0.33   &   0.32 \\
(1.a) linear, standardized  &  0.37  &    0.29   &   0.27   \\
(1.b) Gaussian  &  0.72  &    0.71  &    0.70  \\
(1.b) Gaussian, standardized &    0.34  &    0.39  &    0.35    \\
(2.a) mean aggregation &     0.22&      0.23&      0.20\\
\hline
(2.a) sequential, level $\ge 2$ &     0.93&      0.93&      0.93
\end{tabular}
\caption{male2.mat}
\label{Tab:handmove5}
\end{table}

One can observe that for all five subjects, SVC with sequential kernel of level $2$ or higher outperforms SVC using any of the other kernels not using any sequence information.

The sequence kernel appears to outperform the reported methods from the original paper~\cite{sapsanis2013emg} as well (Figures 11 and 12), though this probably may not be entirely clarified due to three issues:
\begin{description}
\item[(i)] The authors provide no code;
\item[(ii)] it is not described how the ``subject index'' in Figures 11 and 12 relates to the subject file names;
\item[(iii)] Figures 11 and 12, reporting the results and supposedly pertaining to two different classification methods, are exactly identical, thus likely one of the two is and erroneous copy of the other.
\end{description}

\subsection{Experiment: Pendigits}
We use performed classification on the pendgits dataset from the UCI repository \footnote{\url{https://archive.ics.uci.edu/ml/datasets/Pen-Based+Recognition+of+Handwritten+Digits}}.
It contains 10992 samples of digits between 0 and 9 written by 44
different writers with a digital pen on a tablet.
One sample consists of a pair of horizontal and vertical coordinates
of sampled at 8 different time points, hence we deal with a sequence in $\calX^8$ with $\calX=\RR^2$.

The data set comes with a pre-specified training fold of 7494 samples, and a test fold of 3498 samples. Estimation of the prediction error is performed in this validation split, while tuning is done as described via nested 5-fold cross-validation, inside the pre-specified training fold.

We compared prediction via SVC using one of the following three kernels: (2.a) the sequentialized Euclidean kernel, and (2.b) the sequentialized Gaussian kernel. For both, the truncation level was set to $\IndLvl = 4$.

The results are reported in Table~\ref{pendigits_results}. Jackknife standard errors (pooling the five folds) are all 0.01 or smaller. Baseline performance of an uninformed estimator is $1/10 \approx 0.10$.

\begin{table}[h!]
\centering
\begin{tabular}{c|c|c|c|c}
  method\textbackslash method &precision&recall&f1-score \\
\hline
sequential, linear &0.91&0.90&0.89\\
sequential, Gaussian & 0.97&0.97&0.97\\
\end{tabular}
\caption{Pendigits
}
\label{pendigits_results}
\end{table}

The quality of SVC prediction with the sequentialized linear kernel roughly correspond to those of Diehl~\cite{Diehlinvariants}. It is outperformed by SVC prediction with the sequentialized Gaussian kernel which is similar to the baseline performance of $k$-nearest neighbors reported in the documentation of the pendigits dataset.

\appendix

\section{Second kernelization of the signature kernel}
\label{Apx:Kern2}
We need give meaning to $\kernel(\diff \sigma,\diff \tau)$ when $\sigma$ takes values in an arbitrary set $\mathcal{X}$ and hence
the differentials do not make sense.
To motivate the definition below, consider first the case of two paths
$\sigma,\tau$ such that $x:=\phi\left(\sigma\right)$ and $y:=\phi\left(\tau\right)$
are piecewise linear between time points $\bs=\left(
  \bs[1],\ldots,\bs[m] \right)$ resp.~$\bt=\left( \bt[1],\ldots,\bt[n]
\right)$.
In this case,
\begin{align*}
\int_{(s,t)\in(\bs[1],\bs[m])\times(\bt[1],\bt[n])}\langle \diff x(s),\diff y(t)
  \rangle_{\calH}&=\sum_{\substack{i\in[m-1]\\j\in[n-1]}}\int_{(s,t)\in(\bs[i],\bs[i+1])\times
  (\bt[j],\bt[j+1])}\langle \diff x(s),\diff y(t)
  \rangle_{\calH} \\
&=\sum_{\substack{i\in[m-1]\\j\in[n-1]}}\langle \seqdf
  x(\bs)[i],\seqdf y(\bt)[j]\rangle_{\calH}\\
&=\sum_{\substack{i\in[m-1]\\j\in[n-1]}}\kernel\left[
\begin{array}{cc}
\sigma\left(\bs[i]\right) & \sigma\left(\bs[i+1]\right)\\
\tau\left(\bt[j]\right) & \tau\left(\bt[j+1]\right)
\end{array}
\right]
\end{align*}
where we use the notation
\begin{align*}
\kernel\left[
\begin{array}{cc}
a &b\\
c &d
\end{array}
\right] :=\kernel\left(b,d\right)+\kernel\left(a,c\right)-\kernel\left(b,c\right)-\kernel\left(a,d\right).
\end{align*}
If $\bs[1]=\bt[1]=0$ and $\bs[m]=\bt[n]=1$, then Proposition \ref{Prop:Recsig} reads as
\begin{align}\label{eq:pcw_linear_recursion}
\KSigAk(\phi(\sigma),\phi(\tau))=1+\sum_{\substack{i_1\in[m-1]\\j_1\in[n-1]}}\left(
  1+\ldots
  \sum_{\substack{i_{\IndLvl}\in[i_{\IndLvl-1}-1]\\j_{\IndLvl}\in[j_{\IndLvl-1}-1]}}
\kernel
\left[
\begin{array}{cc}
\sigma\left(\bs[i_{\IndLvl}]\right) & \sigma\left(\bs[i_{\IndLvl}+1]\right)\\
\tau\left(\bt[j_{\IndLvl}]\right) & \tau\left(\bt[j_{\IndLvl}+1]\right)
\end{array}
 \right]
\ldots\right)
\kernel
\left[
\begin{array}{cc}
\sigma\left(\bs[i_1]\right) & \sigma\left(\bs[i_1+1]\right)\\
\tau\left(\bt[j_1]\right) & \tau\left(\bt[j_1+1]\right)
\end{array}
 \right].
\end{align}
Now define a signed measure on $\left[0,1\right]^{2}$
via the rule
\[
\Kmeasure_{\sigma,\tau}\left(\left[r,s\right]\times\left[u,v\right]\right):=\kernel\left[\begin{array}{cc}
\sigma\left(r\right) & \sigma\left(s\right)\\
\tau\left(u\right) & \tau\left(v\right)
\end{array}\right]
\]
and note that (\ref{eq:pcw_linear_recursion}) reads as
\[
1+\int_{\left[0,1\right]\times\left[0,1\right]}\left(1+\ldots\int_{\left[0,s_{\IndLvl-1}\right]\times\left[0,t_{\IndLvl-1}\right]}\diff
\Kmeasure_{\sigma,\tau}(s_{\IndLvl},t_{\IndLvl})\ldots\right)\diff\Kmeasure_{\sigma,\tau}\left(s_{1},t_{1}\right).
\]
The content of the definition and theorem below is that this formula
makes sense for arbitrary paths $\sigma,\tau$ such that $\phi\left(\sigma\right)$,$\phi\left(\tau\right)\in \BV(\calH)$.
\begin{defn}
Let $\kernel:\calX\times\calX\rightarrow\RR$ and define
the signed-Borel-measure valued map
\[
\Kmeasure:\Paths\left( \calX \right)\times\Paths\left( \calX \right)\rightarrow\mathcal{M}\left(\left[0,1\right]\times\left[0,1\right]\right),\left(\sigma,\tau\right)\mapsto\Kmeasure_{\sigma,\tau}
\]
via $\Kmeasure_{\sigma,\tau}\left(\left[a,b\right]\times\left[c,d\right]\right):=\kernel\left(\sigma\left(b\right),\tau\left(d\right)\right)+\kernel\left(\sigma\left(a\right),\tau\left(c\right)\right)-\kernel\left(\sigma\left(b\right),\tau\left(d\right)\right)-\kernel\left(\sigma\left(a\right),\tau\left(d\right)\right).$\end{defn}

\begin{thm}
Under the Assumptions~\ref{Ass:primaryk}, it holds that
\begin{align*}
\KSigBk\left( \sigma,\tau \right)
&=1+\sum_{m=1}^{\IndLvl}\int_{\left(\boldsymbol{s},\boldsymbol{t}\right)\in\Delta_{m}\times\Delta_{m}}\diff\Kmeasure_{\sigma,\tau}\left(\boldsymbol{s}\left[1\right],\boldsymbol{t}\left[1\right]\right)\cdots
  \diff\Kmeasure_{\sigma,\tau}\left(\boldsymbol{s}\left[m\right],\boldsymbol{t}\left[m\right]\right)\\
\KSigBk(\sigma,\tau) &= 1+\int_{\left( s_1,t_1
              \right)\in\left( 0,1 \right)\times \left( 0,1 \right)} \left( 1+ \dots \int_{\left(
              s_{\IndLvl},t_{\IndLvl} \right)\in\left( 0,s_{\IndLvl-1} \right)\times \left(
              0,t_{\IndLvl}-1 \right)} \diff\Kmeasure_{\sigma,\tau}\left( s_1,t_2 \right) \dots
        \right) \diff\Kmeasure_{\sigma,\tau}\left( s_1,t_2 \right)
\end{align*}
If $\calX$ is an $\RR$-vector space and $\sigma,\tau$ are differentiable, then
\begin{align*}
\diff\Kmeasure_{\sigma,\tau}\left(s,t\right) =
\kernel \left(\dot{\sigma}(s), \dot{\tau}(t)\right)
\diff s \diff t.
\end{align*}
\end{thm}
\begin{proof}
Let $x:=\phi\left(\sigma\right)$, $y:=\phi\left(\tau\right)$ and note
that
\begin{align*}
  \KSigBk(\sigma,\tau)=\KSigAk(\phi(x),\phi(y)).
\end{align*}
Fix a sequence $\left(\boldsymbol{t}_{n}\right)\subset\Delta\left(\left[0,1\right]\right)$
with $\mesh\left(\boldsymbol{t}_{n}\right)\rightarrow 0$ as
$n\rightarrow\infty$ denote with $x^{n}$,$y^{n}$
the paths given by piecewise linear interpolation
of points $x(\bt_n)=\left( x\left(\bt_{n}\left[1\right]\right),\ldots,x\left(\bt_n\left[n\right]\right)\right)$,
$y(\bt_n)=\left( y\left(\bt_{n}\left[1\right]\right),\ldots,y\left(\bt_n\left[n\right]\right) \right)$.
By the above discussion, the statment holds for $x^{n}$,$y^{n}$ if we
replace the measure $\Kmeasure_{\sigma,\tau}$ by
a measure $\Kmeasure_{\sigma,\tau,n}$ on $\left[0,1\right]^{2}$ as
\begin{align*}
\Kmeasure_{\sigma,\tau,n}\left(\left[a,b\right]\times\left[c,d\right]\right):=\left\langle x^{n}\left(b\right),y^{n}\left(d\right)\right\rangle _{\mathcal{H}}+\left\langle x^{n}\left(a\right),y^{n}\left(a\right)\right\rangle _{\mathcal{H}}-\left\langle x^{n}\left(b\right),y^{n}\left(c\right)\right\rangle_{\calH} -\left\langle x^{n}\left(a\right),y^{n}\left(d\right)\right\rangle_{\calH} .
\end{align*}
As $\mesh\left(\boldsymbol{t}_{n}\right)\rightarrow0$, the right hand
side converges to $\kernel\left[\begin{array}{cc}
\sigma\left(a\right) & \sigma\left(b\right)\\
\tau\left(c\right) & \tau\left(d\right)
\end{array}\right]$. Hence the measure $\Kmeasure_{\sigma,\tau,n}$ converges weakly to the measure
$\Kmeasure_{\sigma,\tau}$. On the other hand, $\left\langle \Sig\left(x^{n}\right),\Sig\left(y^{n}\right)\right\rangle_{\tensalg} $
converges to $\left\langle \Sig\left(x\right),\Sig\left(y\right)\right\rangle_{\tensalg} $
which finishes the proof by sending $n\rightarrow\infty$ in (\ref{eq:pcw_linear_recursion}).
\end{proof}

\section{Integral approximation bounds and proof of Theorem~\ref{Thm:discretmesh}}
\label{Apx:Euler}
\begin{Def}
Let $[a,b]\subset[0,1]$ and $x\in \BV([a,b],\calH)$.
For $V=[a',b']\subseteq \intvl$, we write $x[a',b']$ for the element of $\BV(V,\calH)$ which is obtained by restriction of $x$ to $V$, i.e.,
$$x[a',b'] := [x: V\rightarrow \calH]\subseteq \BV(V,\calH).$$
\end{Def}
The following sum identity becomes very useful.
\begin{Prop}
\label{Prop:sigsum}
Let $\bx\in \calH^+$. Then
$ \SSig(\bx) = \sum_{\substack{\bx^\prime \sqsubset \seqdf \bx}} \prod_{i=1}^{\ell(\bx^\prime)} \bx'[i]$.
\end{Prop}
\begin{proof}
This follows from an explicit algebraic computation.
\end{proof}

\begin{Lem}
\label{Lem:intquad}
Let $x\in \BV (\intvl)$ for $\intvl = [a,b]\subseteq \RR$.
Let $m\in \NN$, let $V_i:= [a_i,b_i]\subseteq \intvl$ for $1\le i\le m$, and let $V=V_1\times \dots \times V_m$.
It holds that
$$\int_V \diff x^{\otimes m} = \prod_{i=1}^m \left[ x(b_i)-x(a_i)\right].$$
\end{Lem}
\begin{proof}
Observe that the integral on the left hand side can be split as
$$\int_V \diff x^{\otimes m} = \int_{V_1}\dots \int_{V_m} \diff x_1 \otimes \dots \otimes  \diff x_m.$$
Separating differential operators, one obtains
$$\int_{V_1}\dots \int_{V_m} \diff x_1\otimes \dots \otimes  \diff x_m = \int_{V_1}\diff x_1\otimes  \dots \otimes  \int_{V_m}\diff x_m.$$
The claim follows from observing that $\inf_{V_i}\diff x_i = x(b_i)-x(a_i)$.
\end{proof}

\begin{Lem}
\label{Lem:intsimpl}
Let $x\in \BV (\intvl,\calH)$ for $\intvl = [a,b]\subseteq [0,1]$.
It holds that
$$\left\|\int_{\simplex^m(\intvl)} \diff x^{\otimes m}\right\|_\calH \le \frac{1}{m!} \Var(x)^m.$$
\end{Lem}
\begin{proof}
By the (continuous/integral) triangle inequality, it holds that
$$\left\|\int_{\simplex^m(\intvl)} \diff x^{\otimes m}\right\|_\calH \le \int_{\simplex^m(\intvl)} \|\diff x\|^m_\calH.$$
Further observe that
$$\int_{\simplex^m(\intvl)} \|\diff x\|^m = \frac{1}{m!}\int_{\intvl^m} \|\diff x\|^m_\calH.$$
The integral on the right hand side can be split, i.e.,
$$\frac{1}{m!}\int_{\intvl^m} \|\diff x\|^m = \frac{1}{m!}\left(\int_{\intvl} \|\diff x\|_\calH\right)^m.$$
Observing that $\int_{\intvl} \|\diff x\|_\calH = \Var(x)$ yields the claim.
\end{proof}

\begin{Lem}
\label{Lem:intbound}
Fix $m,M\in \NN$. Let $x\in \BV ([a,b],\calH)$ and $\bt =(t_1,\ldots,t_{M+1})\in \simplex^{M+1}([a,b])$.
Write
$$\mathcal{C}:= \int_{\simplex^m(\intvl)} \diff x^{\otimes m} \quad\mbox{and}\quad
\mathcal{D}:= \sum_{\substack{\bi\sqsubset [M] \\ \ell(\bi) = m}} (\seqdf x(\bt))[i_1]\otimes\dots \otimes (\seqdf x(\bt))[i_m].$$
Further write, $\intvl:=[a,b]$ and $\intvl_i := [t_i,t_{i+1}]$ for
$1\le i\le M$, and for $\bi\sqsubseteq [M]$, write $\intvl_{\bi}:=
\bigtimes_{i\in \bi}[t_i,t_{i+1}]$. Then, the following equalities
hold:
\begin{enumerate}[label=(\roman*)]
\item $\mathcal{C}-\mathcal{D} = \sum_{\substack{\bi\sqsubseteq [M] \\
      \ell(\bi) = m\\ \# \bi \gneq 1}} \int_{\simplex^m(\intvl)\cap
    \intvl_{\bi}} \diff x^{\otimes m}$,
\item $\left\|\int_{\simplex^m(\intvl)\cap \intvl_{\bi}} \diff x^{\otimes m}\right \|_\calH\le \frac{1}{\bi!}\prod_{i\in \bi} \Var (x[t_i,t_{i+1}]).$
\end{enumerate}
\end{Lem}
\begin{proof}
(i) By splitting the integral, we can write
$$\mathcal{C} = \sum_{\substack{\bi\sqsubseteq [M] \\ \ell(\bi) = m}} \int_{\simplex^m(\intvl)\cap \intvl_{\bi}} \diff x^{\otimes m}.$$
Note that this is not the same summing over $I$ as in $\mathcal{D}$. In $\mathcal{C}$, indices in the index sequence $I$ may repeat, and for $\mathcal{D}$; they may not as they have to increase strictly monotonously.
We will thus write
\begin{align*}
\mathcal{C}_1:= & \sum_{\substack{I\sqsubset [M] \\ \ell(\bi) = m}} \int_{\simplex^m(\intvl)\cap \intvl_{\bi}} \diff x^{\otimes m}\quad\mbox{and}\\
\mathcal{C}_2:= & \sum_{\substack{I\sqsubseteq [M] \\ \ell(\bi) = m\\ \bi!\gneq 1}} \int_{\simplex^m(\intvl)\cap \intvl_{\bi}} \diff x^{\otimes m},
\end{align*}
that is, $\mathcal{C}_1$ collects index sequences without repeated indices, and $\mathcal{C}_2$ collects those with repeat.

Now note that for a non-repeating index sequence $I$ we have $\simplex^m(\intvl)\cap \intvl_{\bi} = \intvl_{\bi}$, therefore
$$\mathcal{C}_1 = \sum_{\substack{\bi\sqsubset [M] \\ \ell(\bi) = m}} \int_{\intvl_{\bi}} \diff x^{\otimes m}$$
Subtraction and collecting terms with the same index yields
$$\mathcal{C}_1 - \mathcal{D} = \sum_{\substack{\bi\sqsubset [M] \\ \ell(\bi) = m}} \left[\int_{\intvl_{\bi}} \diff x^{\otimes m} - (\seqdf x(\bt))[i_1]\otimes\dots \otimes (\seqdf x(\bt))[i_m]\right].$$
This is zero by Lemma~\ref{Lem:intquad}, therefore $\mathcal{C} - \mathcal{D} = \mathcal{C}_1-\mathcal{C}_2-\mathcal{D} = \mathcal{C}_2$ which was the claimed statement.\\

(ii) Fix an index sequence $\bi$. Let $i_1,\dots, i_k$ the distinct occurring indices in $\bi$, and $n_1,\dots, n_k$ the total counts of their respective occurrences. Note that therefore
$$\simplex^m(\intvl)\cap \intvl_{\bi} :=\bigtimes_{j=1}^k \simplex^{n_j} ([\bt[i_k],\bt[i_k+1]]).$$
Write $S_j := \simplex^{n_j} ([\bt[i_k],\bt[i_k+1]])$. Then,
$$\int_{\simplex^m(\intvl)\cap \intvl_{\bi}} \diff x^{\otimes m} = \bigotimes_{j=1}^k \int_{S_j} \diff x^{\otimes n_j}.$$
Therefore, we obtain as a norm bound
$$\left\|\bigotimes_{j=1}^k \int_{S_j} \diff x^{\otimes n_j}\right\|_\calH = \prod_{j=1}^k \left\|\int_{S_j}\diff x^{\otimes n_j}\right\|_\calH \le \frac{1}{\bi!}\prod_{i\in \bi} \Var (x[t_i,t_{i+1}]),$$
where the rightmost inequality follows from applying Lemma~\ref{Lem:intsimpl} to every multiplicand in the product.
\end{proof}

\begin{Thm}
\label{Thm:Eulersimplex}
Let $x\in \BV (\intvl,\calH)$ for $\intvl = [a,b]\subseteq \RR$, and let $\bt\in \simplex^{M+1}(\intvl)$. Write
$$\mathcal{C}:= \Sig(x) \quad\mbox{and}\quad
\mathcal{D}:= \sum_{ \bi\sqsubset [M] } \prod_{r=1}^{\ell(\bi)}(\seqdf x(\bt))[i_r].$$
Write $\mathcal{C}_m, \mathcal{D}_m$ for the respective homogenous parts of $\mathcal{C},\mathcal{D}$. Further define
$$G(z) := \exp ( z \cdot \Var (x) ) - \prod_{i=1}^M \left( 1+ z\cdot \Var (x[t_i, t_{i+1}]) \right) =: \sum_{m=1}^\infty g_m\cdot z^m.$$
That is, the first equality defines $G$ as a function of $z$, the second defines $g_m$ as the Taylor coefficients of $G$ of its expansion in $z$ around $0$.

Then, $G$ is the generating function for an upper bound of approximating $\mathcal{C}$ with $\mathcal{D}$; namely, more precisely:

\begin{itemize}
\item[(i)] $\|\mathcal{C}_m - \mathcal{D}_m\|_{\calH^{\otimes m}}\le g_m$.
\item[(ii)] $\|\mathcal{C} - \mathcal{D}\|_{\tensalg}\le G(1) = \exp( \Var( x)) - \prod_{i=1}^M \left( 1+ \Var (x[t_i, t_{i+1}]) \right)$.
\item[(iii)] If $\bt$ is chosen such that $\Var (x[t_i, t_{i+1}]) = \frac{1}{M}\Var (x)$ for all $i$, then
$$\|\mathcal{C} - \mathcal{D}\|_{\tensalg} \le \frac{\exp(\Var (x))}{M}\left(1+\frac{(\Var (x)^M}{(M-2)!}\right).$$
\end{itemize}
\end{Thm}
\begin{proof}
(i) The bounds given by $G$ are those from Lemma~\ref{Lem:intbound}~(i) and (ii) where $\mathcal{C}_m$ here is $\mathcal{C}$ in Lemma~\ref{Lem:intbound}~(i), and $\mathcal{D}_m$ here is $\mathcal{D}$ in the lemma. The statement follows from explicitly writing out the coefficient $g_m$ and Lemma~\ref{Lem:intbound}~(i) and (ii).\\

(ii) Applying the triangle inequality and then part (i), one obtains
$$\|\mathcal{C} - \mathcal{D}\|_{\tensalg} \le \sum_{m=1}^\infty \|C_m
- \mathcal{D}_m\|_{\calH^{\otimes m}} \le \sum_{m=1}^\infty g_m = G(1)$$
and therefore the statement when writing out $G(1)$.\\

(iii) This follows from observing that for such a choice of $\bt$, it holds that
$$G(1) = \exp ( \Var (x)) - \left(1+\frac{1}{M}\Var (x)\right),$$
to which Euler's Theorem~\ref{Thm:Euler} may be applied.
\end{proof}

\begin{Prop}
\label{Prop:Eulerv2}
Let $x\in\RR, n\in \NN$, $x\geq0$. Let $x_1,\dots, x_M\in \RR$,
$x_i\geq0$ for $i=1,\ldots,M$ such that $\sum_{i=1}^M x_i = x$. Then,
$$\exp(x) = \prod_{i=1}^m\left(1+x_i\right) + g(x,x_1,\dots, x_M),\quad\mbox{where}\; 0\le g(x,x_1,\dots, x_M) \le x\exp(x)\cdot \max_i x_i.$$
In particular, it holds that
\begin{align*}
\lim_{\max_i x_i\rightarrow 0}\prod_{i=1}^m\left(1+x_i\right) = \exp(x),
\end{align*}
where convergence is uniform of order $O(\max_i x_i)$ on any compact subset of $[0,\infty)$.
\end{Prop}
\begin{proof}
All statements follow from the first, which we proceed to prove. Writing out the product, we obtain
$$\prod_{i=1}^m\left(1+x_i\right) = \sum_{\bi\sqsubset [M]} x^{\bi},$$
where abbreviatingly we have written $x^{\bi}:=\prod_{i\in \bi} x_i$.
The Taylor expansion of the exponential on the other hand yields
$$\exp(x)= \exp \left(\sum_{i=1}^M x_i\right) = \sum_{\bi\sqsubseteq [M]} \frac{1}{\bi!} x^{\bi}.$$
Note the major different between both sums above being the repeating indices which may occur in the expansions of $\exp(x)$. More precisely, we obtain
$$\exp(x) - \prod_{i=1}^m\left(1+x_i\right) = \sum_{\substack{\bi\sqsubseteq [M] \\ \bi! \gneq 1}} \frac{1}{\bi!} x^{\bi}.$$
We further split up the sum by length of $\bi$:
$$\exp(x) - \prod_{i=1}^m\left(1+x_i\right) = \sum_{m=2}^\infty \sum_{\substack{\bi\sqsubseteq [M] \\ \ell( \bi)= m\\ \bi! \gneq 1}} \frac{1}{\bi!} x^{\bi}.$$
Positivity of $g$ follows from this equation and positivity of $x$.
Now consider the map $\phi$ which removes the first duplicated index
in an ordered index sequence $\bi$ yielding a sequence of length $\ell(\bi)-1$. On sequences of length $m$, the map $\phi$ is at most $m$-to-one, and surjective onto sequences of length $m-1$. Therefore,
$$\sum_{\substack{\bi\sqsubseteq [M] \\ \ell( \bi) = m\\ \bi! \gneq 1}} \frac{1}{\bi!} x^{\bi} \le X\cdot \sum_{m=2}^\infty \frac{m}{2} \sum_{\substack{\bi\sqsubseteq [M] \\ \ell( \bi) = m-1}} \frac{1}{\bi!} x^{\bi},$$
where $X = \max_i x_i$. Thus,
$$\sum_{m=2}^\infty \sum_{\substack{\bi\sqsubseteq [M] \\ \ell( \bi) = m\\ \bi! \gneq 1}} \frac{1}{\bi!} x^{\bi} \le X\cdot \sum_{m=1}^\infty m\cdot \sum_{\substack{\bi\sqsubseteq [M] \\ \ell( \bi) = m}} \frac{1}{\bi!} x^{\bi}.$$
Comparing to the expansion of $\exp(x)$ above, one observes that the right hand side is equal to
$$X\cdot \sum_{m=1}^\infty m\cdot \frac{x^m}{m!} = X\cdot x \sum_{m=0}^\infty \frac{x^m}{m!} = X\cdot x \cdot \exp(x).$$
\end{proof}

Note that the bounds from Proposition~\ref{Prop:Eulerv2} are worse than those from Euler's Theorem~\ref{Thm:Euler} in the case of equal $x_i$, by a factor of $x$. This is due to the fact that the bound also needs to be valid for heavily imbalanced partitions of $x$ into $x_i$.

\pagestyle{empty}
\bibliographystyle{plain}
\bibliography{sequences_arXiv}

\begin{thebibliography}{10}

\bibitem{bahlmann2002online}
Claus Bahlmann, Bernard Haasdonk, and Hans Burkhardt.
\newblock Online handwriting recognition with support vector machines-a kernel
  approach.
\newblock In {\em Frontiers in handwriting recognition, 2002. proceedings.
  eighth international workshop on}, pages 49--54. IEEE, 2002.

\bibitem{2014arXiv1406.7871B}
H.~{Boedihardjo}, X.~{Geng}, T.~{Lyons}, and D.~{Yang}.
\newblock {The Signature of a Rough Path: Uniqueness}.
\newblock {\em ArXiv e-prints}, June 2014.

\bibitem{CoRoPa}
Djèlil Chafaï, Terry Lyons, Christoph Ladroue, and Anastasia Papavasiliou.
\newblock {\em Computational Rough Paths Project on SourceForge}, 2006
  (accessed January 8, 2016).

\bibitem{cressie2015statistics}
Noel Cressie.
\newblock {\em Statistics for spatial data}.
\newblock John Wiley \& Sons, 2015.

\bibitem{cuturi2011fast}
Marco Cuturi.
\newblock Fast global alignment kernels.
\newblock In {\em Proceedings of the 28th International Conference on Machine
  Learning (ICML-11)}, pages 929--936, 2011.

\bibitem{cuturi2007kernel}
Marco Cuturi, J-P Vert, Oystein Birkenes, and Takashi Matsui.
\newblock A kernel for time series based on global alignments.
\newblock In {\em Acoustics, Speech and Signal Processing, 2007. ICASSP 2007.
  IEEE International Conference on}, volume~2, pages II--413. IEEE, 2007.

\bibitem{Diehlinvariants}
Joscha Diehl.
\newblock Rotation invariants of two dimensional curves based on iterated
  integrals.
\newblock {\em CoRR}, abs/1305.6883, 2013.

\bibitem{friz2010multidimensional}
Peter~K Friz and Nicolas~B Victoir.
\newblock {\em Multidimensional stochastic processes as rough paths: theory and
  applications}, volume 120.
\newblock Cambridge University Press, 2010.

\bibitem{giorgino2009computing}
Toni Giorgino.
\newblock Computing and visualizing dynamic time warping alignments in r: the
  dtw package.
\newblock {\em Journal of statistical Software}, 31(7):1--24, 2009.

\bibitem{DBLP:journals/corr/Graham13}
Benjamin Graham.
\newblock Sparse arrays of signatures for online character recognition.
\newblock abs/1308.0371, 2013.

\bibitem{gyurko2013extracting}
Lajos~Gergely Gyurk{\'o}, Terry Lyons, Mark Kontkowski, and Jonathan Field.
\newblock Extracting information from the signature of a financial data stream.
\newblock {\em arXiv preprint arXiv:1307.7244}, 2013.

\bibitem{hairer2014theory}
Martin Hairer.
\newblock A theory of regularity structures.
\newblock {\em Inventiones mathematicae}, 198(2):269--504, 2014.

\bibitem{hambly2010uniqueness}
Ben Hambly and Terry Lyons.
\newblock Uniqueness for the signature of a path of bounded variation and the
  reduced path group.
\newblock {\em Annals of Mathematics}, 171(1):109--167, 2010.

\bibitem{haussler1999convolution}
David Haussler.
\newblock Convolution kernels on discrete structures.
\newblock Technical report, Citeseer, 1999.

\bibitem{hinton2012deep}
Geoffrey Hinton, Li~Deng, Dong Yu, George~E Dahl, Abdel-rahman Mohamed, Navdeep
  Jaitly, Andrew Senior, Vincent Vanhoucke, Patrick Nguyen, Tara~N Sainath,
  et~al.
\newblock Deep neural networks for acoustic modeling in speech recognition: The
  shared views of four research groups.
\newblock {\em Signal Processing Magazine, IEEE}, 29(6):82--97, 2012.

\bibitem{kruskal1983overview}
Joseph~B Kruskal.
\newblock An overview of sequence comparison: Time warps, string edits, and
  macromolecules.
\newblock {\em SIAM review}, 25(2):201--237, 1983.

\bibitem{leslie04faststringkernels}
Christina Leslie and Rui Kuang.
\newblock Fast string kernels using inexact matching for protein sequences.
\newblock {\em Journal of Machine Learning Research}, 2004.

\bibitem{levin2013learning}
Daniel Levin, Terry Lyons, Hao Ni, et~al.
\newblock Learning from the past, predicting the statistics for the future,
  learning an evolving system.
\newblock {\em ArXiv e-prints (Sept. 2013)}, 2013.

\bibitem{lohdi02textclassification}
Huma Lohdi, Craig Saunders, John Shawe-Taylor, Nello Cristianini, and Chris
  Watkins.
\newblock Text classification using string kernels.
\newblock {\em Journal of Machine Learning Research}, 2002.

\bibitem{lyons2004stflour}
Terry Lyons.
\newblock {\em Differential equations driven by rough paths}.
\newblock Springer Berlin Heidelberg New York, 2004.

\bibitem{lyons2014feature}
Terry Lyons, Hao Ni, and Harald Oberhauser.
\newblock A feature set for streams and an application to high-frequency
  financial tick data.
\newblock In {\em Proceedings of the 2014 International Conference on Big Data
  Science and Computing}, page~5. ACM, 2014.

\bibitem{lyons1998differential}
Terry~J Lyons.
\newblock Differential equations driven by rough signals.
\newblock {\em Revista Matem{\'a}tica Iberoamericana}, 14(2):215--310, 1998.

\bibitem{matheron1963principles}
Georges Matheron.
\newblock Principles of geostatistics.
\newblock {\em Economic geology}, 58(8):1246--1266, 1963.

\bibitem{noma2002dynamic}
Hiroshi Shimodaira Ken-ichi Noma.
\newblock Dynamic time-alignment kernel in support vector machine.
\newblock {\em Advances in neural information processing systems}, 14:921,
  2002.

\bibitem{papavasiliou2011parameter}
Anastasia Papavasiliou, Christophe Ladroue, et~al.
\newblock Parameter estimation for rough differential equations.
\newblock {\em The Annals of Statistics}, 39(4):2047--2073, 2011.

\bibitem{rabiner1989tutorial}
Lawrence~R Rabiner.
\newblock A tutorial on hidden markov models and selected applications in
  speech recognition.
\newblock {\em Proceedings of the IEEE}, 77(2):257--286, 1989.

\bibitem{rasmussen2006gaussian}
Carl~Edward Rasmussen.
\newblock Gaussian processes for machine learning.
\newblock 2006.

\bibitem{revuz1999continuous}
Daniel Revuz and Marc Yor.
\newblock {\em Continuous martingales and Brownian motion}, volume 293.
\newblock Springer Science \& Business Media, 1999.

\bibitem{sakoe1979two}
Hiroaki Sakoe.
\newblock Two-level dp-matching--a dynamic programming-based pattern matching
  algorithm for connected word recognition.
\newblock {\em Acoustics, Speech and Signal Processing, IEEE Transactions on},
  27(6):588--595, 1979.

\bibitem{sakoe1970similarity}
Hiroaki Sakoe and Seibi Chiba.
\newblock A similarity evaluation of speech patterns by dynamic programming.
\newblock In {\em Nat. Meeting of Institute of Electronic Communications
  Engineers of Japan}, page 136, 1970.

\bibitem{sapsanis2013emg}
Christos Sapsanis, George Georgoulas, and Anthony Tzes.
\newblock Emg based classification of basic hand movements based on
  time-frequency features.
\newblock In {\em Control \& Automation (MED), 2013 21st Mediterranean
  Conference on}, pages 716--722. IEEE, 2013.

\bibitem{scholkopf2002learning}
Bernhard Sch{\"o}lkopf and Alexander~J Smola.
\newblock {\em Learning with kernels: Support vector machines, regularization,
  optimization, and beyond}.
\newblock MIT press, 2002.

\bibitem{shawe2004kernel}
John Shawe-Taylor and Nello Cristianini.
\newblock {\em Kernel methods for pattern analysis}.
\newblock Cambridge university press, 2004.

\bibitem{shin2008generalization}
Kilho Shin and Tetsuji Kuboyama.
\newblock A generalization of haussler's convolution kernel: mapping kernel.
\newblock In {\em Proceedings of the 25th international conference on Machine
  learning}, pages 944--951. ACM, 2008.

\bibitem{williams1996gaussian}
Christopher~KI Williams and Carl~Edward Rasmussen.
\newblock Gaussian processes for regression.
\newblock 1996.

\bibitem{DBLP:journals/corr/YangJL15}
Weixin Yang, Lianwen Jin, and Manfei Liu.
\newblock Character-level chinese writer identification using path signature
  feature, dropstroke and deep {CNN}.
\newblock abs/1505.04922, 2015.

\end{thebibliography}
\end{document}